\documentclass[11pt]{article}
\pdfoutput=1
\usepackage{authblk}

\def\pb{\pagebreak}   

\usepackage[dvipsnames]{xcolor}

\usepackage[framemethod=default]{mdframed}
\usepackage{caption}

\newcommand{\ff}{f}

\newcommand{\R}{{\mathbb{R}}}

\newcommand{\z}[1]{{z^{({#1})}}}
\newcommand{\tz}[1]{{\tilde{z}^{({#1})}}}

\newcommand{\cZ}[1]{{\mathcal{Z}^{(#1)}}}

\newcommand{\Q}[1]{{\tau^{(#1)}}}

\newcommand{\tQ}[1]{{\tilde{\tau}^{(#1)}}}

\newcommand{\tu}{{\tilde{\u}}}

\newcommand{\bbu}{{\bar{u}}}
\newcommand{\bbw}{{\bar{w}}}
\usepackage{bbm}
\usepackage{indentfirst}
\usepackage{bm, mathrsfs, graphics,float,amssymb,amsmath,subeqnarray,setspace,graphicx,amsthm,epstopdf,subfigure, enumerate, color}
\usepackage[utf8]{inputenc}
\usepackage[colorlinks,
            linkcolor=red,
            anchorcolor=blue,
            citecolor=blue
            ]{hyperref}
\usepackage{natbib}

\usepackage{fullpage}
\numberwithin{equation}{section}

\newtheorem{lemma}{Lemma}
\newtheorem{theorem}{Theorem}

\newtheorem{definition}{Definition}
\newtheorem{corollary}[theorem]{Corollary}
\newtheorem{remark}{Remark}

\ifx\assumption\undefined
\newtheorem{assumption}{Assumption}
\fi

\newcommand{\cO}{\mathcal{O}}

\newcommand{\EE}{\mathbb{E}}
\newcommand{\RR}{\mathbb{R}}

\newcommand{\PP}{\mathbb{P}}

\usepackage{underscore}

\newcommand{\cI}{\mathcal{I}}
\newcommand{\bone}{\mathbf{1}}

\newcommand{\argmin}{\mathop{\mathrm{argmin}}}

\usepackage{multirow}
\usepackage{tablefootnote}
\usepackage{colortbl}
\usepackage{hhline}

\usepackage{algorithm}
\usepackage{algorithmic}

\usepackage{makecell}
\usepackage{tikz}

\def\ep{\varepsilon}

\newcommand{\x}{{x}}

\newcommand{\cA}{\mathcal{A}}

\newcommand{\E}{\mathbb{E}}

\newcommand{\hf}{\hat{f}}

\renewcommand{\u}{{u}}

\newcommand{\<}{\left\langle}
\renewcommand{\>}{\right\rangle}

\usepackage{mathtools}

\newcommand{\trho}{\tilde{\rho}}

\newcommand{\rhoo}{\rho_{0}}

\newcommand{\zl}{z^{(\ell)}}
\newcommand{\zmi}{z^{(\ell-1)}}

\newcommand{\zL}{z^{(L)}}

\newcommand{\hl}{h^{(\ell)}}
\newcommand{\gl}{g^{(\ell)}}

\newcommand{\ts}{\tilde{s}}
\newcommand{\hs}{\hat{s}}

\newcommand{\tPhi}{\tilde{\Phi}}

\newcommand{\tPsi}{\tilde{\Psi}}

\newcommand{\bR}{\breve{R}}
\newcommand{\dR}{\dot{R}}

\usepackage{amsmath}
\newcommand{\hcZ}{\hat{\mathcal{Z}}}
\allowdisplaybreaks  
\begin{document}
\title{Convex Formulation of Overparameterized Deep Neural Networks}

\author[$*$]{Cong Fang}

\affil[$*$]{
Shenzhen Research Institute of Big Data\thanks{This work was done when Cong Fang was an intern at Shenzhen Research Institute of Big Data. The author is now with Princeton University.}}

\author[$*$]{
Yihong Gu
}

\author[$\dag$]{
Weizhong Zhang
}

\affil[$\dag$]{Department of  Mathematics\\ HKUST}

\author[$\ddag$]{
Tong Zhang
}
\affil[$\ddag$]{Department of Computer Science and Mathematics\\ HKUST}


\date{}
\maketitle
\begin{abstract}
Analysis of over-parameterized neural networks has drawn significant attention in recent years. It was shown that such systems behave like convex systems under various restricted settings, such as for two-level neural networks, and when learning is only restricted locally in the so-called neural tangent kernel space around specialized initializations.
However, there are no theoretical techniques that can analyze fully trained deep neural networks encountered in practice. 
This paper solves this fundamental problem by investigating such overparameterized deep neural networks when fully trained. 
We generalize a new technique called \emph{neural feature repopulation}, originally introduced in \citep{fang2019over} for two-level neural networks, to analyze deep neural networks. It is shown that under suitable representations, overparameterized deep neural networks are inherently {\em convex}, and when optimized, the system can 
learn effective features suitable for the underlying learning task under mild conditions. This new analysis is consistent with empirical observations that deep neural networks are capable of learning efficient feature representations. 
Therefore, the highly unexpected result of this paper can satisfactorily explain the practical success of deep neural networks. Empirical studies confirm that predictions of our theory are consistent with results observed in practice.
\end{abstract}

\section{Introduction}
Deep Neural Networks (DNNs) have achieved great successes in  numerous real applications, such as image classification \citep{krizhevsky2012imagenet,simonyan2014very,he2016deep}, face recognition \citep{sun2014deep}, video understanding \citep{yue2015beyond}, neural language processing \citep{bahdanau2014neural, luong2015effective}, etc. However, compared to the empirical successes, the theoretical understanding of DNNs falls far behind. Part of the reasons might be the general perception that  DNNs are highly non-convex learning models.

In recent years, there has been significant breakthroughs \citep{MeiE7665, chizat2018global, du2018gradient,allen2019convergence} in  analyzing  over-parameterized  Neural Networks (NNs),  which are  NNs with massive neurons in  hidden layer(s).  It is observed from empirical studies that such NNs are easy to train \citep{zhang2016understanding}.  And it was noted that under some restrictive settings, such as two-level NNs \citep{MeiE7665, chizat2018global} and when learning is only restricted locally in the  neural tangent kernel space around certain initializations \citep{du2018gradient,allen2019convergence}, NNs behave like convex systems when  the number of the hidden neurons goes to infinity.  Unfortunately, to the best of our knowledge, existing studies failed to analyze fully trained DNNs encountered in practice. In particular, the existing analysis cannot  explain how DNNs can learn discriminative features specifically for the underlying learning task, as observed in real applications \citep{zeiler2014visualizing}.

 To remedy the gap between the existing theories and practical observations,  this paper  develops a new theory   that can be applied to fully trained DNNs. Following a similar argument in the analysis of two-level NNs in \cite{fang2019over},  we introduce a new theoretical framework called \emph{neural feature
repopulation} (NFR), to reformulate over-parameterized DNNs. Our results show that under suitable conditions,  in the limit of infinite number of hidden neurons,  DNNs are    infinity-dimensional \emph{convex} learning models with appropriate re-parameterization.  In our framework, given a DNN,   the hidden layers are regraded as features and the model output is  given by a  simple linear model using features of the top hidden layer.  The output of a DNN, in the limit of infinite number of hidden neurons, depends on the distributions of the features and the final linear model.  We show that using the NFR technique, it is possible to decouple the  distributions of the features from the loss function and their impact can be integrated into the regularizer. This largely simplifies the objective function.  Under our framework, the  feature learning process is characterized by the regularizer. When suitable regularizers are imposed, the overall objective function under special re-parameterization is convex, and it guarantees that the DNN learns useful feature representations under mild conditions.  Unlike the Neural Tangent Kernel approach, our theoretical framework for DNN does not require the variables to be confined in an infinitesimal region. Therefore it can explain the ability of fully trained DNNs to learn target feature representations. This matches the empirical observations.

More concretely,  the paper is organized as follows. 
Section~\ref{sec:relatedwork} discusses the relationship of this paper to earlier studies, especially recent works on the analysis of over-parameterized NNs. 
Section~\ref{sec:discrete-DNNs} introduces the definition of discrete DNNs, and we introduce an importance weighting formulation which eventually motivates our NFR formulation. 
Section~\ref{sec:continuous-DNNs} describes the continuous DNN when the number of hidden nodes goes to infinity in the discrete DNN. In this formulation, each hidden layer is represented by a distribution over its hidden nodes that represent functions of the input. We can further interpret a discrete DNN as a random sample of hidden nodes from a continuous DNN at each layer, and then study the variance of such random discretization. The variance formula motivates the study of a class of regularization conditions for DNNs. Using the connection between  discrete and continuous (over-parameterized) DNNs, we introduce the  process of NFR in Section \ref{sec:nfr}. In this process, an over-parameterized DNN can be reformulated as a convex system that learns effective feature representations for the underlying task.  Experiments are presented in Section~\ref{sec:experiments} to demonstrate that our theory is consistent with empirical observations. In Section~\ref{sec:algo}, we consider  a  new  optimization  procedure  inspired  by  the  NFR view  to  verify  its  effectiveness. Final remarks are given in Section~\ref{sec:conclusion}.  

The main contributions of this work can be described as follows.
\begin{itemize}
\item We propose a new framework for analyzing overparameterized deep NN called  \emph{neural feature repopulation} (NFR).  It can be used to remove the effect of learned feature distributions over hidden nodes from the loss function, and confine the effect only to the regularizer. This significantly simplifies the objective function.

\item   We  study  a class of regularizers. With these regularizers, the  over-parameterized  DNN  can be reformulated as  a \emph{convex} system  using NFR under certain conditions. The global solution of  such a convex  model guarantees useful feature representations for the underlying learning task.

\item Our theory matches empirical findings, and hence this theory can satisfactorily explain the success of fully trained overparameterized \emph{deep} NNs.  
\end{itemize}

We shall also mention that the paper focuses on presenting the intuitions and consequences of the new framework. In order to make the underlying ideas easier to understand, some of the analyses are not stated with complete mathematical rigor. A more formal treatment of the results will be left to future works. \\

\noindent\textbf{Notations.}  For a vector $x\in \R^d$, we denote $\|\cdot\|_1$ and $\|\cdot\|$ to be its $\ell_1$ and $\ell_2$ norms, respectively. We let $x^\top$ be the transpose of $x$ and let $x_k$ be the value of $k$-th dimension of $x$ with $k=1,\ldots, d$. Let  $[m]=\{1,2,...,m\}$ for a positive integer $m$. For a function $f(x): \RR^{d} \rightarrow \RR$, we denote $\nabla_{x}f$ to be the gradient of $f$ with respect to $x$. For two real valued numbers $a$ and $b$, we denote $a\vee b $ to be $\max(a,b)$. If $\mu$ and $\nu$ are two measures on the same measurable space, we denote $\mu \ll \nu$ if $\mu$ is  absolutely continuous with respect to $\nu$, and $\mu \sim \nu$ if $\mu \ll \nu$ and $\nu\ll \mu$.


\section{Related Work}
\label{sec:relatedwork}

In recent years, there have been a number of significant developments to obtain better theoretical understandings of NNs.  We review works that are most related to this paper.  The main challenge for developing such theoretical analysis is the non-convexity of the NN model, which implies that first-order algorithms such as (Stochastic) Gradient Descent may converge to local stationary points. 

In some earlier works, a number of researchers \citep{hardt2016identity, freeman2016topology, BrutzkusG17,digvijay2017,Rong2017, ainesh2018, Rong2019} studied NNs   under special  conditions either for input data or for NN architectures. By carefully  characterizing the geometric landscapes  of the objective function,  these early works showed that some special NNs  satisfy the so-called strict saddle property \citep{ge2015escaping}. One can then use some recent results in nonconvex optimization \citep{jin2017escape,fang2018spider,fang2019sharp} to show that  first-order algorithms for such NNs can efficiently escape saddle points and converge to some local mimima.
 \begin{table}[t]
	\centering
	\caption{Some Representative Works on the Three Views for Over-parameterize NNs }
	\label{table:summary-over-parameterized-NNs}
	{
		\begin{tabular}{p{3.5cm}<{\centering}|p{4.6cm}<{\centering}p{4.6cm}<{\centering}}
			\hline\hline
			Views	&Two-level & Multi-level\\ \hline\hline
					\makecell*[c]{Mean Field} & \makecell*[c]{\citealt{MeiE7665}\\ \citealt{chizat2018global}}& \makecell*[c]{\citealt{nguyen2019mean}}\\\hline
	\makecell*[c]{	Neural\\ Tangent Kernel}	&\makecell*[c]{\citealt{ du2019gradient}\\ \citealt{li2018learning}}&\makecell*[c]{\citealt{ du2018gradient}\\ \citealt{allen2018learning,allen2019convergence}}\\\hline
					\makecell*[c]{Neural Feature\\Repopulation}&\makecell*[c]{ \citealt{fang2019over}}& \makecell*[c]{this work}\\ \hline\hline
			\end{tabular}
	}
\end{table}	

A number of more recent theoretical analysis
 focused on over-parameterized NNs,  which are NNs that contain a large number of  hidden nodes. The motivation of overparameterization comes from the empirical observation that over-parameterized DNNs  are much easier to train and  often achieve better performances \citep{zhang2016understanding}.   When the number of  hidden units  goes to infinity,  the network naturally becomes a continuous NN.   In the continuous limit, the resulting networks are found to  behave like convex systems under appropriate conditions.  Our work follows  this line of the research. 
 In the following,  we review the three existing  points of views, i.e.,  the mean field view, the neural tangent kernel view, and the  NFR view. Due to the space limitation,  Table \ref{table:summary-over-parameterized-NNs} only summarizes some of the representative  studies on  these views, and additional studies are discussed in the text below.

\subsection{Mean Field View for Over-parameterized Two-level NNs}
The Mean Field approach has  been recently introduced to   analyze two-level NNs.  This point of view  models the continuous NN as a distribution over the NN's parameters, and it studies the evolution of the distribution  as a Wasserstein gradient flow during the training process \citep{MeiE7665,chizat2018global,sirignano2019mean,rotskoff2018neural,mei2019mean, dou2019training, wei2018margin}. The process can be represented by a  partial differential equation, which can be further studied  mathematically. For two-level continuous NNs, it is known that the objective function with respect to the distribution of parameters is convex in the continuous limit. And it can be shown that (noisy)  Gradient Decent can find global optimal solutions under certain conditions.

The benefit of the Mean Field View is that it mathematically characterizes the entire training process of NN. However, it relies on the special observation that a two-level continuous NN is naturally a linear model with respect to  the distribution of NN parameters, and this observation is not applicable to multi-level architectures. Consequently, it is difficult to generalize the Mean Field view to analyze DNNs without losing convexity.   In fact, in a recent attempt \cite{nguyen2019mean}, the technique of Mean Field view is applied to DNNs,  but the author could only  obtain the evolution dynamic equation for Gradient Descent.   Since DNN is no longer a linear model with respect to the distribution of the parameters, \cite{nguyen2019mean} failed to show that the system is convex. Therefore similar to the situation of Stochastic Gradient Descent for discrete DNNs, Gradient Descent for continuous DNNs can still lead to suboptimal solutions.

\subsection{Neural Tangent Kernel View for  Over-parameterized DNNs}
One remarkable direction to extend the Mean Field view to analyze DNNs is to restrict the DNN parameters in an infinitesimal region around the initial values.   With proper scaling and random initialization,  DNN can be regarded as a  linear model in this  infinitesimal region, which makes the training dynamics solvable. This point of view is referred to as the Neural Tangent Kernel (NTK) view,  since  the evolution of the trainable parameters, when restricted to the infinitesimal region, can  be characterized by a kernel in the  tangent space.  There have been a series of studies based on this point of view, which give various theoretical guarantees including  sharper convergence rates and  tighter generalization bounds with early stopping for both two-layer NNs \citep{du2019gradient,li2018learning,arora2019fine, su2019learning} and  DNNs with more complex structures \citep{du2018gradient,lee2018deep,jacot2018neural,allen2018learning,allen2019convergence,zou2018stochastic}.

However, as point out by \cite{fang2019over},  NTK is not a satisfactory mathematical model for NNs. This is because NTK essentially approximates a nonlinear NN by a linear model of the infinite dimensional random features associated with the NTK, and these features  are not learned from the underlying task. In contrast, it is well known empirically that the success of NNs largely relies on their ability to learn discriminative feature representations \citep{lecun2015deep}.  Therefore the  NTK view does not match the  empirical observations.  

A recent attempt to justify NTK  is given by \cite{arora2019exact}, who proposed an efficient exact algorithm to compute Convolutional Neural Tangent Kernel. However, the classification accuracy of $77\%$ on the CIFAR-10 dataset obtained by  kernel regression using NTK is $5\%$ less than that of the corresponding fully trained Convolutional NNs , and is at least $15\%$ less than the accuracies achieved by modern NNs such as ResNet \citep{he2016deep}.

\subsection{Neural Feature Repopulation  for Over-parameterized NNs}
More recently, \cite{fang2019over} proposed the NFR view for analyzing two-level NNs. While the Mean Filed view does not have the concept of ``layer'' in its analysis of two-level NNs, the NFR view treats the top-layer linear model and the bottom-layer feature learning separately. The dynamics of feature learning in NN is modeled by a ``repopulation'' process in NFR.  It was shown in \citep{fang2019over} that  under certain conditions, two-level  NNs  can learn a near-optimal distribution over the features in terms of efficient representation with respect to the underlying task.  

Our current work  can be regarded as a \emph{non-trivial} generalization of the NFR view from two-level NNs to \emph{deep} NNs. Specifically, we employ the NFR technique to simplify the objective function of DNN training by showing that it is possible to reparameterize a DNN as a linear model of learned features. Moreover, the feature learning process is decoupled from the loss function, and is determined by the regularizer. This reparameterization significantly  simplifies the overall objective function. We introduce a class of regularizers that can guarantee the convexity of the overall objective function, and  we show that efficient distributions over features can be obtained as  the  result of training.   Compared to NTK \citep{du2018gradient, allen2018learning,allen2019convergence}, NFR is more consistent  with empirical observations, because NN parameters are no longer restricted in  an infinitesimal region. This implies that meaningful features can be learned from training.

\section{Discrete Deep Neural Networks}
\label{sec:discrete-DNNs}

In this section, we introduce the discrete fully connected deep neural networks. 
We also introduce an importance sampling scheme for discrete DNN, which can be used to motivate the NFR technique later in the paper. 

\subsection{Standard DNN}

In machine learning, we are interested in prediction problems, where we are given an input vector $x=[x_1,\ldots,x_d] \in \R^d$, and we want to predict its corresponding output $y$. 

In general, we want to learn a function $\hat{f}(x) \in \R^K$ that can be used for prediction. The quality of prediction is measured by a loss function $\phi(\cdot,\cdot)$, which is typically convex in the first argument.
For regression, where $y \in \R^K$, we often use the least squares loss function \[
\phi(\hat{\ff}(x),y)=\|\hat{\ff}(x)-y\|^2 . 
\]
For classification, where $y \in [K]$, we often use the logistic loss, that is
\[
\phi(\hat{\ff}(x),y)= -\hat{\ff}_y(x) + \ln \sum_{j=1}^K \exp(\hat{\ff}_j(x)) .
\]
 
In this paper, we consider a deep neural network $\hat{\ff}(x)$ with $L$ hidden layers, which can be defined recursively as a function of $x$.
First we let $\hat{f}^{(0)}_j(x)=x_j$ be the input for $j\in [d]$. Let $m^{(0)}=d$.
For $\ell \in [L]$, we define the nodes in the $\ell$-th layer as
\begin{align}
\hat{f}^{(\ell)}_j(x) = h^{(\ell)}\left(\hat{g}^{(\ell)}_j(x)
\right) \mbox{ with }\hat{g}^{(\ell)}_j(x)=
\frac{1}{m^{(\ell-1)}} \sum_{k=1}^{m^{(\ell-1)}} w^{(\ell)}_{j,k} \hat{f}^{(\ell-1)}_k(x),
\quad j\in [m^{(\ell)}] ,\label{eqn:hatf-l-j}
\end{align}
where $h^{(\ell)}$ is the activation function and $w^{(\ell)} \in \RR^{m^{(\ell)}\times m^{(\ell-1)}}$ is the weight matrix of the $\ell$-th layer comprised of $m^{(\ell)}$ rows, i.e., $w^{(\ell)}_j=[w^{(\ell)}_{j,1},\ldots,w^{(\ell)}_{j,m^{(\ell-1)}}]\in \mathbb{R}^{m^{(\ell-1)}}$ with $j \in [m^{(\ell)}]$.
At the top layer, we define the output $\hat{f}(x)$ as
\begin{align}
\hat{f}(x) = \frac{1}{m^{(L)}} \sum_{j=1}^{m^{(L)}} \u_j
\hat{f}^{(L)}_j(x),\label{eqn:hatf}
\end{align}
where $\u_j \in R^K$ is a vector. This defines an $(L+1)$-level fully connected deep neural network with $m^{(\ell)}$ nodes in each layer.

The formula for training a discrete deep neural network $\hat{\ff}$ is to minimize the the following objective function:
\begin{align}
&\hat{Q}(w, \u) = J(\hat{f}) + \hat{R}(w, \u), \nonumber 
\end{align}
where $w=\{w^{(1)}, \ldots, w^{(L)}\}$ and $J(\hat{f})= \mathbb{E}_{x,y} \phi(\hat{f}(x),y)$ with $\phi(\cdot,\cdot) $ being the loss function and $\hat{R}(w, \u)$ being the regularizer that takes the form of 
\begin{align}
\hat{R}(w, \u) = \sum_{\ell = 1}^{L} 
\lambda^{(\ell)} \hat{R}^{(\ell)}(w^{(\ell)})
 + \lambda^{(u)}\hat{R}^{(u)}(\u) . \label{eq:reg}
\end{align}
The parameters $\lambda^{(1)}, ...,\lambda^{(L)}$ and $\lambda^{(u)}$ are the non-negative hyper-parameters for the regularizer.
In this paper, we are particularly interested in the following form of regularizer:
\begin{align}
&\hat{R}^{(\ell)}(w^{(\ell)})=
\frac{1}{m^{(\ell-1)}}
\sum_{k=1}^{m^{(\ell-1)}}
r_2\left(
\frac{1}{m^{(\ell)}}\sum_{j=1}^{m^{(\ell)}} r_1(w_{j,k}^{(\ell)})
\right),\label{eq:reg-w}\\
&\hat{R}^{(u)}(\u)= \frac{1}{m^{(L)}}\sum_{j=1}^{m^{(L)}} r^{(u)}(\u_j).\label{eq:reg-u}
\end{align}
This class of regularizers will be analyzed  
in Sections \ref{subsec:variance-of-discrete-approximization} and \ref{sec:global-optimum}.

\subsection{Importance Weighted DNN}\label{importance weight}

In our framework,  the hidden units of a discrete NN can be regarded as samples from a continuous distribution (refer to Section \ref{sec:continuous-DNNs}).  Instead of uniform sampling, as in \eqref{eqn:hatf-l-j} and \eqref{eqn:hatf}, we may also consider importance sampling to construct the hidden nodes $\hat{f}^{(\ell)}_j(x)$ with $j\in [m^{(\ell)}]$ and the final function $\hf(\x)$.  Specifically,  we assign the $k$-th hidden node at layer $\ell$  an importance weighting $\hat{p}^{(\ell)}_k$, with $\sum_{k=1}^{m^{(\ell)}} \hat{p}^{(\ell)}_k=m^{(\ell)}$,  and we let the  hidden nodes follow a non-uniform distribution  whose  probability mass function with index $k$ at layer $\ell$ is  $\hat{p}^{(\ell)}_k/m^{(\ell)}$.  Then  we  can  rewrite  
 \eqref{eqn:hatf-l-j} and \eqref{eqn:hatf} as  
$$
\hat{f}^{(\ell)}_j(x) = h^{(\ell)}\left(\frac{1}{m^{(\ell-1)}} 
\sum_{k=1}^{m^{(\ell-1)}} \frac{w^{(\ell)}_{j,k}}{\hat{p}^{(\ell-1)}_k} \Big(\hat{p}^{(\ell-1)}_k\hat{f}^{(\ell-1)}_k(x)\Big)
\right), 
\quad j\in [m^{(\ell)}],$$
and
\begin{align}
\hat{f}(x) = \frac{1}{m^{(L)}}\sum_{j=1}^{m^{(L)}} \frac{\u_j}{\hat{p}^{(L)}_j}
\Big(\hat{p}^{(L)}_j\hat{f}^{(L)}_j(x)\Big), \label{eq:disc-nn-output}
\end{align}
where  we have also replaced the weight $w^{(\ell)}_{j,k}$ and  $\u_j$ with  $w^{(\ell)}_{j,k}/\hat{p}^{(\ell-1)}_k$ and $\u_j/\hat{p}^{(L)}_j$, respectively. We can find that under such transformation, the functions $\hat{f}^{(\ell)}_j(x)$ and $\hf(x)$ remain unchanged.  Similarly,  for the regularizers,  by   replacing the  weight $w^{(\ell)}_{j,k}$  and $\u_j$  with  $w^{(\ell)}_{j,k}/\hat{p}^{(\ell-1)}_k$  and $\u_j/\hat{p}^{(L)}_j$, respectively,   and  by replacing the uniform  distribution over the hidden units with the corresponding non-uniform distribution, we have
\begin{align}
    \hat{R}(\hat{p};w, \u)   
    &=\sum_{\ell = 1}^{L}  \lambda^{(\ell)} \hat{R}^{(\ell)}(\hat{p};w^{(\ell)})+ \lambda^{(\u)}\hat{R}^{(u)}(\hat{p};u) ,\label{eq:reg-im}
\end{align}
where
\[
\hat{R}^{(\ell)}(\hat{p};w^{(\ell)})=
\frac{1}{m^{(\ell-1)}}
\sum_{k=1}^{m^{(\ell-1)}}
r_2\left(
\frac{1}{m^{(\ell)}}\sum_{j=1}^{m^{(\ell)}} r_1(w_{j,k}^{(\ell)}/\hat{p}^{(\ell-1)}_k) \hat{p}^{(\ell)}_j
\right) \hat{p}^{(\ell-1)}_k,
\]
and
\[
\hat{R}^{(u)}(\hat{p};u)=\frac{1}{m^{(L)}}\sum_{j=1}^{m^{(L)}}r^{(u)}(\u_j/\hat{p}^{(L)}_j)\hat{p}_j^{(L)}.
\] 

An important observation is that  under the importance weighting transformation,   function values on  all hidden nodes and the final output value  remain unchanged, while the regularization values change. 
This means that the importance weighting parameters $\hat{p}$ only appear in the regularization term.
This observation eventually leads to the NFR technique to reparameterize continuous DNNs.

 The discrete importance weighting formula presented here provides intuitions to our NFR method for continuous DNNs, and the detailed analysis will be provided in Section \ref{sec:importance-weighting}.


\section{Continuous Deep Neural Networks}\label{sec:continuous-DNNs}

When $m^{(\ell)} \to \infty$ for all $\ell\in  [L]$, we can define a continuous DNN according to the definition of discrete DNN in Section~\ref{sec:discrete-DNNs}.
In the case of continuous DNN, each hidden node  in the $\ell$-th layer is associated with a real valued function  of the input $x$. It can be characterized by  the weights connecting it to the hidden nodes in the $(\ell-1)$-th layer and therefore can be represented as a real
valued function defined on these  nodes.  The space of the hidden nodes at layer $\ell$,  i.e. all such real valued functions, is denoted as $\mathcal{Z}^{(\ell)}$ in this paper  and can be regarded as an  infinite-dimensional feature (representation) of the input data $x$.    A continuous DNN can be obtained by defining probability measures $\rho^{(\ell)}$ on hidden nodes for each hidden layer $\ell$, which is equivalent to the probability measures on real valued functions or features  of the input $x$. A discrete DNN can be obtained by sampling $m^{(\ell)}$  hidden nodes, i.e.  elements  belonging to $\mathcal{Z}^{(\ell)}$,  from $\rho^{(\ell)}$ at each layer $\ell$.  We  denote 
$\rho=\{\rho^{(0)},\ldots,\rho^{(L)}\}$ and present the details below.

\subsection{Continuous DNN Formulation}
By convention, we let  the $0$-th layer be the input layer and denote  $\cZ{0}=[d]$ to be its node space corresponding to the $d$ components of the input $x$.  And we let $\rho^{(0)}$ be a probability measure over  $\cZ{0}$. And  for each node $z^{(0)} \in \cZ{0}$, we  let
\begin{align*}
f^{(0)}(\rho,z^{(0)};x)= x_{z^{(0)}} .
\end{align*}

Now consider the $\ell$-th  layer with 
$\ell \in [L]$.  
For conceptual simplicity, 
in this paper, we let $\cZ{\ell}$ be the measurable real-valued function class on $\cZ{\ell-1}$. Given $\z{\ell} \in \cZ{\ell}$ and $\z{\ell-1} \in \cZ{\ell-1}$, we define 
\begin{align*}
w(\z{\ell},\z{\ell-1}) = \z{\ell}(\z{\ell-1}).
\end{align*}
Because $\z{\ell}$ can be regarded as a hidden node in layer $\ell$, and $\z{\ell-1}$ as a hidden node in the $(\ell-1)$-th layer, thus
$w(\z{\ell},\z{\ell-1})$ is the analogy of $w^{(\ell)}_{ij}$ in discrete DNN, which is the weight connecting node $i$ and $j$ in layer $\ell$ and $\ell-1$ respectively. 
Using this notation, we define the function associated with the node $z^{(\ell)}$ in $\ell$-th layer of the continuous NN as follows:
\begin{align}
\ff^{(\ell)}(\rho,\z{\ell};x) =& h^{(\ell)}\left( 
g^{(\ell)}(\rho,\z{\ell};x)\right), \label{eq:nn-activation}
\end{align}
where $h^{(\ell)}(\cdot)$ is the activation function of the $\ell$-th layer, and
\begin{align}
g^{(\ell)}(\rho,\z{\ell};x)=&
\int
  w(\z{\ell},\z{\ell-1}) \ff^{(\ell-1)}(\rho,\z{\ell-1};x) \; d \rho^{(\ell-1)}(\z{\ell-1}). \label{eq:nn-layer}
\end{align}
Moreover, we let $\rho^{(\ell)}$ be a probability measure over $\cZ{\ell}$.

Finally, for the output layer, let $\u(\cdot): \cZ{L} \rightarrow \R^{K}$ be a $K$-dimensional vector valued function on $\cZ{L}$, then we can define the final output of continuous DNN as 
\begin{equation}
\ff(\rho,\u;x) = \int \u(z^{(L)}) \ff^{(L)}(\rho,z^{(L)};x) \; d\rho^{(L)}(z^{(L)}). \label{eq:nn-output}
\end{equation}
The objective function in continuous NN takes the form of 
\begin{align}
Q(\rho, \u) =& J(f(\rho,\u;\cdot)) + R(\rho, \u),\label{eqn:Q-continuous}
\end{align}
where 
\begin{align}
 R(\rho, \u) =& \sum_{\ell=1}^L \lambda^{(\ell)}
R^{(\ell)}(\rho) + \lambda^{(\u)} R^{(\u)}(\rho,\u),\label{eqn:R-continuous}
\end{align}
with
\begin{align}
&R^{(\ell)}(\rho)=\int r_2\left (\int r_1\big(w (z^{(\ell)}, z^{(\ell-1)})\big)d\rho^{(\ell)}(z^{(\ell)})\right ) d\rho^{(\ell-1)}(z^{(\ell-1)}),\label{eqn:R-l-continuous}\\
&R^{(\u)}(\rho,\u) =  \int r^{(u)}(\u(z^{(L)}))d\rho^{(L)}(z^{(L)}).\label{eqn:R-u-continuous}
\end{align}

\begin{remark}
In this paper, unless otherwise specified, for any probability measure sequence 
$\rho=\{\rho^{(0)}, \ldots, \rho^{(L)}\}$,
we always let $\rho^{(0)}$ be the uniform distribution on $\cZ{0}$. 
\end{remark}
The above process defines a continuous DNN. In the following, we establish the relationship between the continuous and discrete DNNs.
\subsection{Assumptions}
Before presenting our analysis,  we specify the necessary assumptions first. We note that these assumptions are rather mild, and easy to be satisfied.

\begin{assumption}[Bounded Gradient]\label{ass:2} We assume the activation function is differentiable,   and its derivative is bounded.   That is, there exists a constant $c_0>0$, such that
\begin{align}
    |  \nabla \hl (x) |\leq c_0, \quad \ell \in [L].\nonumber
\end{align}
\end{assumption}

\begin{assumption}[Continuous Gradient]\label{ass:1}
We further assume that there exist two constants $\alpha>0$ and $c_1>0$ such that
\begin{align}
| \nabla \hl (x) -  \nabla \hl(y) |\leq c_1 | x-y |^{\alpha}, \quad \ell\in [L].\nonumber
\end{align}
\end{assumption}
Assumption \ref{ass:1} is a special type of modulus of continuity for $\nabla \hl (x)$.
When $\alpha = 1$, Assumption \ref{ass:1} is the standard L-smooth condition for the activation function $\hl (x)$. When $\alpha<1$, it  holds more generally in the local region.  When proving Theorem \ref{var}, our moment condition in Assumption \ref{ass:3} depends on $\alpha$.
We also note that the  commonly-used activation functions, e.g. sigmoid, tanh, and smooth relu, satisfy this assumption for all $0<\alpha\leq1$. 
\begin{assumption}[$(q_0,  q_1)$-Bounded Moment Condition]\label{ass:3}
 We assume for all $\ell\in \{2, \dots, L \}$, we have
\begin{align}
\E_{z^{(\ell-1)}, \dots, z^{(L)}} \Bigg| \left(\|u^{(L)}(z^{(L)})\|\vee 1]\right)&\prod_{i=\ell+1}^{L} \left( |w(z^{(i)}, z^{(i-1)})|\vee 1\right)\notag\\
&\left[ w(z^{(\ell)}, z^{(\ell-1)})f^{(\ell-1)}(\rho,z^{(\ell-1)};x) -g^{(\ell)}(\rho, z^{(\ell)};x) \right]\Bigg|^{q_0}\leq c_M.\nonumber
\end{align}
Moreover, we assume
\begin{align}
\E_{z^{(L)}} \left\| u^{(L)}f^{(L)}(\rho,z^{(L)};x) -f(\rho, u;x)\right\|^{q_1}\leq c_{M_1}.\nonumber
\end{align}
\end{assumption}
The constants $q_0$ and $q_1$ in Assumption \ref{ass:3} will be specified later based on our theorem statements.

\subsection{Relationship between Discrete and Continuous DNNs}
We can now investigate the relationship between the discrete and continuous DNNs.  Similar to the method used in \citep{fang2019over} for two-level NNs,
the discrete DNN can be constructed from a continuous one by sampling hidden nodes from the probability measure sequence $\rho$. The detailed procedure is as follows:
\begin{enumerate}
    \item Keep the input layer of the discrete DNN identical to that of the continuous DNN.
    \item For each hidden layer $\ell \in [L]$, draw $m^{(\ell)}$ i.i.d. samples $\{z^{(\ell)}_i: i\in [ m^{(\ell)}], z^{(\ell)}_i\in \mathcal{Z}^{(\ell)} \}$, which is denoted as $\hat{\mathcal{Z}}^{(\ell)}$, from $\rho^{(\ell)}$ of continuous DNN, and set the weights
\begin{align*}
w^{(\ell)}_{i,j} = w(z^{(\ell)}_i,z^{(\ell-1)}_j) \mbox{ with } i \in [m^{(\ell)}] \mbox{ and } j\in [m^{(\ell-1)}].
\end{align*}
\item For the top layer, set
\begin{align*}
\u_j= \u(z^{(L)}_j),
\end{align*}
for each sampled $z^{(L)}_j$ with $j=[m^{(L)}]$.
\end{enumerate}

The following result shows when $m^{(\ell)} \to \infty$ for all $\ell\in  [L]$,  the final output
converges to that of the  continuous DNN in $L^1$.  All the proofs in this paper are left to the Appendices.

\begin{theorem}[Consistence of Discretization]\label{theorem:convergence-discrete-NN}
Given a continuous NN. Under Assumptions \ref{ass:2} and \ref{ass:3} with  $q_0 = q_1= 1+c_{\ep} $ for any $c_{\ep} >0$, and suppose there is a discrete NN constructed from the continuous DNN using the procedure above,  then for any input $x$,  we have
\makeatletter
\newcommand{\leqnos}{\tagsleft@true\let\veqno\@@leqno}
\newcommand{\reqnos}{\tagsleft@false\let\veqno\@@eqno}
\reqnos
\makeatother
\begingroup\leqnos
\begin{align}
	&\lim\limits_{m^{(\ell)}\rightarrow \infty, \ell = 1, ...,k-1} \E_{ \hat{\mathcal{Z}}^{(1)},\ldots,\hat{\mathcal{Z}}^{(k-1)}, z^{(k)}_j}\;\left|\hat{f}^{(k)}_j(x)- f^{(k)}(\rho, z^{(k)}_j; x)\right| =0,  \quad  j\in [m^{(k)}] \mbox{ and } k\geq 2.\tag{$i$} \label{the:con1}\\
	&\lim\limits_{m^{(\ell)}\rightarrow \infty, \ell = 1, ...,L} \E\;\left\|\hat{f}(x)- f(\rho, u; x)\right\| =0.\tag{$ii$}\label{the:con2}
	\end{align}
\endgroup
\end{theorem}

 The part $\textup{(i)}$ of the theorem above does not cover the case of $k=1$, since it is trivial to show that $f^{(1)}(\rho, z^{(1)}_j; x) \equiv \hat{f}^{(1)}_j(x)$ holds for all $j\in [ m^{(0)}]$.

\subsection{Variance of Discrete Approximation}\label{subsec:variance-of-discrete-approximization}
While Theorem~\ref{theorem:convergence-discrete-NN} shows the convergence of discrete NN to continuous NN using random sampling, it is possible to estimate the variance of such approximation with a slightly strong condition, as shown below.
\begin{theorem}[Variance of Discrete Approximation]\label{var}
We denote  $\frac{\partial f(\rho,u;x)}{\partial z^{(L)}} = u(z^{(L)})$ and
$$
\frac{\partial f(\rho,u;x)}{\partial \zl} = \E_{z^{(\ell+1)}} \left[w( z^{(l+1)}, \zl) \nabla h^{(\ell+1)}\left(g(\rho,z^{(\ell+1)};x)\right)\frac{\partial f(\rho,u;x)}{\partial z^{(\ell+1)}}\right] \mbox{ with } \ell\in [L-1].
$$
Then under Assumptions \ref{ass:2}, \ref{ass:1}, and \ref{ass:3} with  $q_0 =  2(1+\alpha)^{L} $, $q_1 =2$ and treating $c_0$, $c_1$, $\alpha$, $c_M$, $c_{M_1}$ and $L$ as constants, we have
\begin{align}
&\E\left\| \hf(x) - f(\rho,u;x) \right\|^2\notag\\
=&\sum_{\ell =1}^{L-1}  \frac{1}{m^{(\ell)}}  \E_{z^{(\ell)}}\left|\E_{z^{(\ell+1)}}\frac{\partial f(\rho,u;x)}{\partial z^{(\ell+1)}}  \left[  f^{(\ell)}(\rho,z^{(\ell)};x)w(z^{(\ell+1)}, z^{(\ell)})- g^{(\ell+1)}(\rho,z^{(\ell+1)};x) \right] \right|^2 \notag\\
& + \frac{1}{m^{(L)}} \E_{z^{(L)}} \left\| f^{(L)}(\rho,z^{(L)};x) u(z^{(L)})-f(\rho, u;x)\right\|^2+ \mathcal{O}\left( \sum_{\ell=1}^{L} (m^{(\ell)})^{-2} \right)+ \mathcal{O}\left( \sum_{\ell=1}^{L} (m^{(\ell)})^{-(1+\alpha/2)} \right).\notag
\end{align}
\end{theorem}
In Theorem \ref{var}, for  $\ell\in [L-1]$,  we can choose $a = \mathcal{O}\left(\frac{1}{L^{1+\nu}}\right)$ with $\nu\geq0$.  Thus, the assumption only requires the bounded $2+\mathcal{O}\left(\frac{1}{L^{\nu}}\right)$-th moment. 

It was argued by \cite{fang2019over} that for two-level NNs, discretization variance is small when the underlying feature representation learned by the continuous NN is good. Similarly,
from Theorem \ref{var}, we can argue that if a continuous DNN learns good feature representations, then the variance of the corresponding discrete approximation is small. We can impose appropriate regularization condition to achieve this effect. 
We can derive  a regularization condition from the theorem above, which can induce good feature representations. Specifically, if we assume that both $|f^{(\ell)}(\rho,z^{(\ell)};x) \frac{\partial f(\rho,u;x)}{\partial z^{(\ell+1)}} |$  and $|f^{(L)}(\rho, z^{(L)};x)|$ are bounded, then in order to minimize the variance, Theorem \ref{var} implies that we can minimize the following regularization:
\begin{align}
R(\rho, u) = \sum_{\ell=1}^L\lambda^{(\ell)} R^{(\ell)}(\rho) + \lambda^{(u)} R^{(u)}(\rho, u), \label{eqn:R-specific}   
\end{align}
with 
\begin{align}
    &R^{(\ell)}(\rho) = \int \left(\int |w(z^{(\ell)},z^{(\ell-1)})|d\rho^{(\ell)}(z^{(\ell)})\right)^2d\rho^{(\ell-1)}(z^{(\ell-1)}),\label{eqn:R-l-specific}\\
    &R^{(u)}(\rho, u)=\int \left\| u(z^{(L)})\right\|^2d\rho^{(L)}(z^{(L)}), \label{eqn:R-u-specific}
\end{align}
which corresponds to the choices of $r_1(w)=|w|$, $r_2(w)=w^2$, and $r^{(u)}(u)=\|u\|^2$ in \eqref{eq:reg} (see the proofs of \eqref{eqn:R-l-specific} and \eqref{eqn:R-u-specific} in Appendix \ref{app:regg}).  We make further discussions below regarding  the  obtained regularizer.
\begin{itemize}
    \item From the modeling perspective,  the   regularizer derived in the paper controls the efficacy of the learned feature distributions in  terms of efficient representation under random sampling.  If the regularization value is small,  then the variance is small,  and  $f$ can be efficiently  represented by a discrete DNN with a small number of  hidden neurons randomly sampled from the feature distributions. It is well-known that two-level NNs can achieve universal approximation \citep{cybenko1989approximation}, but DNNs have stronger representation power than shallow NNs,  especially for targets with high-frequency components \citep{andoni2014learning}. That is, a much smaller number of hidden units are needed to represent such target functions using DNNs.  We will validate empirically in  Section \ref{sec:experiments} that for such targets,  the variance of deeper NNs  becomes smaller after training. 
    \item  From the computational perspective,   our unexpected result in Section \ref{sec:global-optimum} shows that with this regularizer, the objective function  is convex under suitable  re-parameterization. We also note that in the discrete formulation, the regularizer is the simple $\ell_{1,2}$ norm regularizer if we write $w(z^{(\ell)}, z^{(\ell-1)})$ as a matrix  with the $(j,i)$-th entry being $w(z^{(i)}, z^{(j)})$.  For such a regularizer, Proximal Gradient Descent \citep{parikh2014proximal} can be applied to efficiently solve the optimization problem.  
\end{itemize}


\section{Neural Feature Repopulation}
\label{sec:nfr}
Form \eqref{eqn:R-continuous}, we know that the continuous DNN can be fully characterized by $(\rho,\u)$, where $\rho$ denotes  the sequence
\[
\rho = \left\{\rho^{(\ell)}\right\}_{\ell\in [L]},
\]
and $\rho^{(\ell)}$ is the probability measure on the node space $\cZ{\ell}$.  Recall that  Section \ref{importance weight} introduces the importance weighting technique in the discrete NN.  In  Section \ref{sec:importance-weighting}, we will  adapt it to  continuous DNN. It will motivate the  NFR  technique to reformulate continuous DNN.  The details will be formally presented in Section \ref{sec:formulation-of-nfr}.  Finally,  Section \ref{sec:global-optimum} discusses some consequences of  NFR when we specify the regularizers.  In particular, we will show that for the class of $\ell_{1,2}$ norm regularization obtained in Section \ref{subsec:variance-of-discrete-approximization}, the entire objective function is convex under our re-parameterization. This generalizes a similar analysis for two-level NN in \citep{fang2019over}.

\subsection{Importance Weighting}\label{sec:importance-weighting}
Consider a sequence of probability measures
$P=\left\{P^{(\ell)}\right\}_{\ell\in [L]}$ such that
$P^{(\ell)} \sim \rho^{(\ell)}$ for $\ell \in [L]$. We define importance weighting functions $\{p^{(\ell)}(z^{(\ell)})\}_{\ell \in [L]}$ as follows: 
$$p^{(\ell)}(z^{(\ell)})=\frac{dP^{(\ell)}}{d\rho^{(\ell)}}(z^{(\ell)}), \quad \ell \in [L]. $$
Mathematically, $p^{(\ell)}(\zl)$ is the   Radon–Nikodym derivative.  We have $0<p^{(\ell)}(\zl)<\infty$ for all $\ell \in [L]$ and $\zl\in \mathcal{Z}^{(\ell)}$.   We  begin to re-parameterize the NN layer by layer.

(1) When $\ell=1$, we  define $\tilde{\tau}^{(0)}$ as an  identity mapping from $\cZ{0}$ onto itself and $\tilde{z}^{(0)}=\tilde{\tau}^{(0)}(z^{(0)})$ for $z^{(0)}\in \cZ{0}$. Let $\tilde{P}^{(0)}$ be the pushforward of $P^{(0)}$ by $\tilde{\tau}^{(0)}$, i.e.  $P^{(0)}\circ (\tilde{\tau}^{(0)})^{-1}$.  
We have $f^{(0)}(\rho,z^{(0)};x)=f^{(0)}(\tilde{P},\tilde{z}^{(0)};x)$, and thus can rewrite $f^{(1)}(\rho, z^{(1)};x)$ as 
\begin{align}
f^{(1)}(\rho, z^{(1)}; x) &= \int \frac{w(z^{(1)}, z^{(0)})}{p^{(0)}(z^{(0)})}f^{(0)}(\rho,z^{(0)};x) p^{(0)}(z^{(0)})d\rho^{(0)}(z^{(0)}) \nonumber \\
&= \int \frac{w(z^{(1)}, z^{(0)})}{p^{(0)}(z^{(0)})}f^{(0)}(\rho, z^{(0)};x) 
dP^{(0)}(z^{(0)})\nonumber \\
&\overset{a}{=} \int \frac{w(z^{(1)}, z^{(0)})}{p^{(0)}(z^{(0)})}f^{(0)}(\tilde{P},\tilde{z}^{(0)};x) 
d\tilde{P}^{(0)}(\tilde{z}^{(0)})\nonumber\\
&\overset{b}{=} \int w(\tilde{z}^{(1)}, \tilde{z}^{(0)})f^{(0)}(\tilde{P}, \tilde{z}^{(0)};x) d\tilde{P}^{(0)}(\tilde{z}^{(0)}) \nonumber\\
& = f^{(1)}(\tilde{P}, \tilde{z}^{(1)}; x),\nonumber
\end{align}
where $\overset{a}=$ holds since $\tilde{\tau}^{(0)}$ is the identity mapping; in $\overset{b}=$ we let $\tilde{z}^{(1)}$ be the value of the mapping  $\tilde{\tau}^{(1)}(\rho, P; \cdot):\cZ{1} \rightarrow \cZ{1}$ at $z^{(1)}$ and  define  $\tilde{\tau}^{(1)}$ as:
\begin{align}
\tilde{\tau}^{(1)}(\rho, P; z)(\tilde{\tau}^{0}(\zeta)) = z(\zeta)/p^{(0)}(\zeta)= z(\zeta)/\frac{dP^{(0)}}{d\rho^{(0)}}(\zeta), \mbox{ with } z\in \cZ{1}, \zeta \in \cZ{0}.\nonumber
\end{align} 
We  proceed to define $\tilde{P}^{(1)}$  as the pushforward of $P^{(1)}$ by $\tilde{\tau}^{(1)}$, i.e. $P^{(1)} \circ (\tilde{\tau}^{(1)})^{-1}$.

(2) For $\ell \geq 2$, we assume that $f^{(\ell-1)}(\rho, z^{(\ell-1)}; x) = f^{(\ell-1)}(\tilde{P}, \tilde{z}^{(\ell-1)};x)$ holds with $\tilde{z}^{(\ell-1)}=\tilde{\tau}^{(\ell-1)}(\rho, P;z^{(\ell-1)})$ and let $\tilde{P}^{(\ell-1)}$  be the pushforward of $P^{(\ell-1)}$ by $\tilde{\tau}^{(\ell-1)}$, i.e. $P^{(\ell-1)} \circ (\tilde{\tau}^{(\ell-1)})^{-1}$. We can then rewrite $f^{(\ell)}$  as follows: 
\begin{align}
f^{(\ell)}(\rho, z^{(\ell)}; x) &= \int \frac{w(z^{(\ell)}, z^{(\ell-1)})}{p^{(\ell-1)}(z^{(\ell-1)})}f^{(\ell-1)}(\rho, z^{(\ell-1)}; x) p^{(\ell-1)}(z^{(\ell-1)}) d\rho^{(\ell-1)}(z^{(\ell-1)})\nonumber \\
&=\int\frac{w(z^{(\ell)}, z^{(\ell-1)})}{p^{(\ell-1)}(z^{(\ell-1)})}f^{(\ell-1)}(\tilde{P}, \tilde{z}^{(\ell-1)};x)  dP^{(\ell-1)}(z^{(\ell-1)})\nonumber\\
&  =\int\frac{w(z^{(\ell)}, z^{(\ell-1)})}{p^{(\ell-1)}(z^{(\ell-1)})}f^{(\ell-1)}(\tilde{P}, \tilde{z}^{(\ell-1)};x) d\tilde{P}^{(\ell-1)}(\tilde{z}^{(\ell-1)})\nonumber\\
&=\int w(\tilde{z}^{(\ell)}, \tilde{z}^{(\ell-1)})f^{(\ell-1)}(\tilde{P}, \tilde{z}^{(\ell-1)};x) d\tilde{P}^{(\ell-1)}(\tilde{z}^{(\ell-1)})\nonumber\\
& = f^{(\ell)}(\tilde{P}, \tilde{z}^{(\ell)}; x), \nonumber
\end{align}
where $\tilde{z}^{(\ell)} = \tilde{\tau}^{(\ell)}(\rho, P; z^{(\ell)})$ with the mapping $\tilde{\tau}^{(\ell)}(\rho, P; \cdot): \cZ{\ell} \rightarrow \cZ{\ell}$ defined as follows:
\begin{align}
   \tilde{\tau}^{(\ell)}(\rho, P; z)(\tilde{\tau}^{(\ell-1)}(\rho, P; \zeta)) =z(\zeta)/p^{(\ell)}(\zeta)= z(\zeta)/\frac{dP^{(\ell-1)}}{d\rho^{(\ell-1)}}(\zeta), z \in \cZ{\ell}, \zeta \in \cZ{\ell-1}.\label{eqn:inv-tau-l}
\end{align}
The mapping $\tilde{\tau}^{(\ell)}$ induces a new  probability measure denoted as $\tilde{P}^{(\ell)}$, which is the pushforward of $P^{(\ell)}$ by $\tilde{\tau}^{(\ell)}$.

The results above indicate that the importance weight $p^{(\ell)}$ induces a series of  transformations, i.e. $\tilde{\tau}^{(\ell)}(\rho, P; \cdot): \cZ{\ell}\rightarrow \cZ{\ell}$ for $\ell\in [L]$ defined in (\ref{eqn:inv-tau-l}), 
which further induces a new probability measure sequence $\tilde{P} = \{\tilde{P}^{(0)},\ldots, \tilde{P}^{(L)} \}$ from $\rho$. Under the  aforementioned process,   $f^{(\ell)}(\rho, z^{(\ell)};x)=f^{(\ell)}(\tilde{P}, \tilde{z}^{(\ell)};x)$ holds for all $\ell \in [L]$.

(4) Finally, we can reformulate the top layer as
\begin{align}
f(\rho, u; x) &= \int \frac{u(z^{(L)})}{p^{(L)}(z^{(L)})} f^{(L)}(\rho, z^{(L)};x)  p^{(L)}(z^{(L)})d\rho^{(L)}(z^{(L)})\nonumber\\
&= \int \frac{u(z^{(L)})}{p^{(L)}(z^{(L)})} f^{(L)}(\rho, z^{(L)};x)  dP^{(L)}(z^{(L)})\nonumber\\
&= \int \frac{u(z^{(L)})}{p^{(L)}(z^{(L)})}  f^{(L)}(\tilde{P}, \tilde{z}^{(L)}; x)  dP^{(L)}(z^{(L)})\nonumber \\
&= \int \frac{u(z^{(L)})}{p^{(L)}(z^{(L)})}  f^{(L)}(\tilde{P}, \tilde{z}^{(L)}; x)  d\tilde{P}^{(L)}(\tilde{z}^{(L)})\nonumber \\
& = \int \tilde{u}(\tilde{z}^{(L)})f^{(L)}(\tilde{P}, \tilde{z}^{(L)}; x)  d\tilde{P}^{(L)}(\tilde{z}^{(L)})\nonumber\\
&= f(\tilde{P},\tilde{u};x),\nonumber
\end{align}
where $\tilde{u}= \tilde{\tau}^{(u)}(\rho, P; u)$ and $\tilde{\tau}^{(u)}(\rho, P; \cdot)$ is a mapping from $\mathcal{U}$ to itself defined as  
\begin{align}
   \tilde{\tau}^{(u)}(\rho, P; u)(\tilde{\tau}^{(L)}(\rho, P;\zeta))=u(\zeta)/p^{(L)}(\zeta)= u(\zeta)/\frac{dP^{(L)}}{d\rho^{(L)}}(\zeta) \mbox{ with } u \in \mathcal{U}, \zeta \in \cZ{L}.\label{eqn:inv-tau-u}
\end{align}
with $\mathcal{U}: \cZ{L}\rightarrow \R^{K}$ being the class of vector valued functions on $\cZ{L}$. 

Therefore, by importance weighting, for a fixed basic continuous DNN $f(\rho, u;x)$, a given probability sequence $P$ induces a different but equivalent continuous DNN $f(\tilde{P}, \tilde{u};x)$, which keeps the function  values  on  all  hidden  nodes  and  the  final  loss value unchanged. The discretization of this process is the same as the discrete importance weighting in Section~\ref{importance weight}.

The reverse of the above equivalence relationship also holds. That is, given a continuous DNN $f(\tilde{P}, \tilde{u};x)$, we can transform it into a equivalent basic continuous DNN $f(\rho,u;x)$. Such a process needs to define a specific importance weighting characterized by $P$ according to $\tilde{P}$,  and the  inverse mappings of $\tilde{\tau}^{(\ell)}(\rho, P; \cdot)$ and $\tilde{\tau}^{(u)}(\rho, P; \cdot)$. We refer the process as NFR, which can fundamentally simplify the objective and will be discussed  in the next subsection.  In the following, we  give the formal definitions of $\tilde{\tau}^{(\ell)}(\rho, P; \cdot)$, $\tilde{\tau}^{(u)}(\rho, P; \cdot)$, and their inverses below.

\begin{definition}[Feature Repopulation Transformations] \label{definition:tau-l-u}
	Given two probability measure sequences $\rho=\left\{\rho^{(\ell)}\right\}_{\ell=0,...,L}$ and  $P=\left\{P^{(\ell)}\right\}_{\ell= 0,\ldots,L}$ with $\rho^{(\ell)}\sim P^{(\ell)}, ~\ell \in [L]$, let $\tau^{(0)}(P,\rho; \cdot)$  be the identity mapping on $\cZ{0}$. We define   a  mapping sequence $\tau^{(\ell)}(P, \rho; \cdot): \cZ{\ell} \rightarrow \cZ{\ell}$ for $\ell\in [L]$ and a mapping $\tau^{(u)}(P, \rho; \cdot): \mathcal{U} \rightarrow \mathcal{U}$  recursively as follows:
\begin{align}
    &\tau^{(\ell)}(P,\rho; z)(\tau^{(\ell-1)}(P, \rho; \zeta)) = z(\zeta)\frac{dP^{(\ell-1)}}{d\rho^{(\ell-1)}}\Big(\tau^{(\ell-1)}(P,\rho; \zeta)\Big), \mbox{ with } \zeta \in \cZ{\ell-1}, \nonumber\\
 &\tau^{(u)}(P,\rho; u)(\tau^{(L)}(P,\rho; \zeta))= u(\zeta)\frac{dP^{(L)}}{d\rho^{(L)}}\Big(\tau^{(L)}(P,\rho; \zeta)\Big) \mbox{ with } \zeta \in \cZ{L}.\nonumber
\end{align}
 Finally, let $\tilde{\tau}^{(\ell)}(\rho, P;\cdot)$ with $\ell\in [L]$ and $\tilde{\tau}^{(u)}(\rho, P;\cdot)$  be the inverse mappings of $\tau^{(\ell)}(P,\rho; \cdot)$ and $\tau^{(u)}(P,\rho;\cdot)$, respectively.
\end{definition}

\subsection{The Formulation of Neural Feature Repopulation}\label{sec:formulation-of-nfr}
This section proposes NFR, which  is inspired by our reformulation approach for importance weighted continuous DNN in Section \ref{sec:importance-weighting}. Given a continuous NN characterized by $(\rho, u)$, we show it is possible to transform it to a standard  NN characterized by $(\rho_0,\tu)$  and under such transformation
 the objective function depends on the probability measure sequence $\rho$ only through  the regularizer.  In this paper, we assume $\rho_0^{(\ell)}$  with $\ell\in[L]$ are  ``good'' distributions, which satisfy $\frac{d\rho^{(\ell)}_0(z^{(\ell)}_1)}{d\rho^{(\ell)}_{0}(z^{(\ell)}_2)}<\infty$ for any $z^{(\ell)}_1\in \mathcal{Z}^{(\ell)}$ and   $z^{(\ell)}_2\in \mathcal{Z}^{(\ell)}$.  In the finite dimensional case,  Gaussian distributions satisfy this condition.    

\subsubsection{Example on Three-Level NN}
We first  give an example on three-level continuous DNN to illustrate how to perform NFR.  We use  $( \rho^{(1)}, \rho^{(2)}, u)$ to represent the continuous NN for the sake of simplicity. And our destination is to transform a NN denoted by $(\rho^{(1)}, \rho^{(2)}, u)$ to the standard one denoted by $( \rho_0^{(1)}, \rho_0^{(2)}, \tilde{u})$.  We consider performing NFR layer-wisely. That is  we first transform    $(\rho^{(1)}, \rho^{(2)}, u)$ to $( \rho_0^{(1)}, \tilde{\rho}^{(2)}, \bar{u})$,  then  from  $( \rho_0^{(1)}, \tilde{\rho}^{(2)}, \bar{u})$ to $( \rho_0^{(1)}, \rho_0^{(2)}, \tilde{u})$. 
 We  note that performing NFR layer-wisely is not fundamentally necessary.  In fact,  we can directly transform $( \rho^{(1)}, \rho^{(2)}, u)$ to the final standard one (refer to Appendix \ref{proof of the3}). However, the  procedure presented is more intuitive and thus easier to understand.  

\begin{enumerate}[(1)]
    \item  \label{ex:one}  $( \rho^{(1)}, \rho^{(2)}, u) \rightarrow  ( \rho_0^{(1)}, \tilde{\rho}^{(2)}, \bar{u})$:
    
  Consider the mapping $\kappa(\cdot): \mathcal{Z}^{(2)} \to \mathcal{Z}^{(2)}$ as
\begin{align}
    \kappa(z^{(2)})(z^{(1)}) = z^{(2)}(z^{(1)})\frac{d\rho^{(1)}}{d\rho_0^{(1)}}(z^{(1)}) , \mbox{ for all }z^{(1)}\in \cZ{1}.\label{ex:1}
\end{align}
We can find that $\kappa(\cdot)$ is equivalent to  $\tau^{(2)}(\rho, \rho_0;\cdot)$ in Definition \ref{definition:tau-l-u}. Because $\rho^{(0)}$ is always the uniform distribution on $\mathcal{Z}^{(0)}$, the $1$-st layer $f^{(1)}(\rho^{(0)}, z^{(1)};x)$ will never change.  For the $2$-nd layer,   we can rewrite it  as 
\begin{align}
f^{(2)}(\rho^{(1)}, z^{(2)};x) =& h^{(2)}\left( \int w(z^{(2)}, z^{(1)}) f^{(1)}(\rho_0,z^{(1)};x) d\rho^{(1)}(z^{(1)})\right)\label{ex2}\\
=&h^{(2)}\left( \int w(z^{(2)}, z^{(1)}) \frac{d\rho^{(1)}}{d\rho_0^{(1)}}(z^{(1)}) f^{(1)}(\rho_0,z^{(1)};x) d\rho^{(1)}_0(z^{(1)})\right)\nonumber\\
\overset{\eqref{ex:1}}{=}&h^{(2)}\left( \int w(\tilde{z}^{(2)}, z^{(1)})  f^{(1)}(\rho_0,z^{(1)};x) d\rho^{(1)}_0(z^{(1)})\right)\nonumber\\
=& f^{(2)}(\rho_0^{(1)},\tilde{z}^{(2)};x), \nonumber
\end{align}
where $\tilde{z}^{(2)} = \kappa(z^{(2)})$.
Let $\tilde{\rho}^{(2)}$ be the pushforward of $\rho^{(2)}$ by $\kappa$, i.e. $\rho^{(2)} \circ \kappa^{-1}$.   Similarly, define $\bar{u}(\tilde{z}^{(2)}) = u(z^{(2)})$. For the output, we have
\begin{align}
    f(\rho^{(1)}, \rho^{(2)},u;x) =& \int u(z^{(2)}) f^{(2)}(\rho^{(1)},z^{(2)};x) d\rho^{(2)}(z^{(2)})\nonumber\\
    \overset{\eqref{ex2}}{=}&\int u(z^{(2)}) f^{(2)}(\rho_0^{(1)},\tilde{z}^{(2)};x) d\rho^{(2)}(z^{(2)})\nonumber\\
    =&\int\bar{u}(\tilde{z}^{(2)}) f^{(2)}(\rho_0^{(1)},\tilde{z}^{(2)};x) d\tilde{\rho}^{(2)}(\tilde{z}^{(2)})\nonumber\\
    =&f(\rho^{(1)}_0, \tilde{\rho}^{(2)}, \bar{u};x).\notag
\end{align}
In this way, we have  transformed the NN from  $( \rho^{(1)}, \rho^{(2)}, u)$ to  $( \rho_0^{(1)}, \tilde{\rho}^{(2)}, \bar{u})$,   while the output, i.e. $f(\rho^{(1)}, \rho^{(2)},u;x)$, remains unchanged.
\item $ ( \rho_0^{(1)}, \tilde{\rho}^{(2)}, \bar{u}) \rightarrow  ( \rho_0^{(1)}, \rho_0^{(2)}, \tilde{u})$:

We can define:
$$ \tilde{u}(\tilde{z}^{(2)}) =\bar{u}(\tilde{z}^{(2)})\frac{d\tilde{\rho}^{(2)}}{d\rho_0^{(2)}}(\tilde{z}^{(2)}). $$
Then we have 
\begin{align}
   &f(\rho^{(1)}_0, \tilde{\rho}^{(2)}, \bar{u};x)\notag\\
   =&\int\bar{u}(\tilde{z}^{(2)}) f^{(2)}(\rho_0^{(1)},\tilde{z}^{(2)};x) d\tilde{\rho}^{(2)}(\tilde{z}^{(2)})\nonumber\\
   =&\int\bar{u}(\tilde{z}^{(2)})\frac{d\tilde{\rho}^{(2)}}{d\rho_0^{(2)}}(\tilde{z}^{(2)}) f^{(2)}(\rho_0^{(1)},\tilde{z}^{(2)};x) d\rho_0^{(2)}(\tilde{z}^{(2)})\nonumber\\
      =&\int\tilde{u}(\tilde{z}^{(2)}) f^{(2)}(\rho_0^{(1)},\tilde{z}^{(2)};x) d\rho_0^{(2)}(\tilde{z}^{(2)})\nonumber\\
      =&f(\rho^{(1)}_0, \rho_0^{(2)}, \tilde{u};x).\notag
\end{align}
\end{enumerate}
The above procedure illustrates how to transform an arbitrary three level NN with feature distributions $\rho$ to a standardized NN with feature distributions $\rhoo$.   For multiple level  NNs ($L>3$), one can perform Step \ref{ex:one}  recursively. The details are left in  Appendix \ref{proof of the3}.

\subsubsection{Formal  Results of Neural Feature Repopulation}
This subsection presents the formal results for NFR. We   consider general (deep) continuous NNs  and   further take  the transformations of  regularizers into account.  The theorems are  presented  below.
\begin{theorem}[Neural Feature Repopulation]\label{NFR} Suppose we are given a fixed probability measure sequence $\rhoo = \{\rhoo^{(\ell)}\}_{\ell\in [L]}$. For any continuous DNN characterized by $\rho=\left\{\rho^{(1)},\ldots, \rho^{(L)}\right\}$ and $\u$. Assume  $\rho^{(\ell)}\sim \rhoo^{(\ell)}$ for all $\ell\in [L]$,  we can define a new sequence $\tilde{\rho}=\{\tilde{\rho}^{(\ell)}\}_{\ell\in [L]}$ recursively by letting $\tilde{\rho}^{(\ell)}$ be the pushforward of $\rho^{(\ell)}$ by $\tau^{(\ell)}(\trho, \rho_0;
\cdot)$. We let $\tilde{u}=\tau^{(u)}(\tilde{\rho},\rhoo;u)$, then the followings hold.
\begin{itemize}
    \item [\textup{(i)}] We can rewrite $\ff^{(\ell)}(\rho,\z{\ell};x)$ and $f(\rho, u;x)$ as follows
\begin{align}
    &\ff^{(\ell)}(\rho,\z{\ell};x) =
\ff^{(\ell)}(\rhoo,\tilde{z}^{(\ell)};x),
\label{eq:nfr-f-layer}\\
&\ff(\rho,\u;x)=
\ff(\rhoo,\tu;x).
\label{eq:nfr-f-top}
\end{align}
\item [\textup{(ii)}] Moreover, we have 
\begin{align}
&J(f(\rho,u; \cdot))=J(f(\rho_0, \tilde{u};\cdot)), \label{eq:nfr-j}\\
&R^{(\ell)}(\rho)= \tilde{R}^{(\ell)}(\tilde{\rho}):=\int r_2\left(\int r_1\Big(w (\tilde{z}^{(\ell)}, \tilde{z}^{(\ell-1)})\frac{d\rhoo^{(\ell-1)}}
{d\tilde{\rho}^{(\ell-1)}}(\tz{\ell-1})\Big)d\tilde{\rho}^{(\ell)}(\tilde{z}^{(\ell)})\right ) d\tilde{\rho}^{(\ell-1)}(\tilde{z}^{(\ell-1)}),
 \label{eq:nfr-r-layer}\\
& R^{(u)}(\rho,\u)= \tilde{R}^{(u)}(\tilde{\rho},\tu) :=\int r^{(u)}\Big(\tu(\tz{L})\frac{d\rhoo^{(L)}}{d\tilde{\rho}^{(L)}}(\tz{L})\Big)d\tilde{\rho}^{(L)}(\tilde{z}^{(L)}).\label{eq:nfr-r-top}
 \end{align}
\end{itemize}
 \label{thm:nfr}
\end{theorem}

\begin{remark}
Since $\tau^{(\ell)}(\tilde{\rho},\rhoo;\cdot)$ depends only on $\rhoo^{(k)}$ and $\tilde{\rho}^{(k)}$ with $k \in \{ 0,\ldots, \ell-1\}$, but not on $\tilde{\rho}^{(\ell)}$,
the probability measure $\tilde{\rho}^{(\ell)}$ in Theorem~\ref{thm:nfr}, which relies on  $\tau^{(\ell)}(\tilde{\rho},\rhoo;\cdot)$, can be properly defined recursively. The detailed definition can be found in the proof of Theorem~\ref{thm:nfr} in the appendix. 
\end{remark}

\begin{remark}
We note that in Theorem~\ref{thm:nfr}, 
the probability measure $\tilde{\rho}(\cdot)$ corresponds to $P(\cdot)$
in Section~\ref{sec:importance-weighting},
and the density
$d \tilde{\rho}^{(\ell)}/d \rhoo^{(\ell)}$
can be regarded as the importance weighting $p^{(\ell)}(\cdot)$.
\end{remark}

 Theorem \ref{NFR} shows that given $\rho_0$, an NN represented by the pair $(\rho, u)$ induces another representation $(\tilde{\rho}, \tilde{u})$ that reparameterizes the neural network. Moreover, given $\rho_0$ and $\tu$, which determines a function
$\ff(\rho_0,\tu,x)$, 
there can be many equivalent
representations $\ff(\rho,\u;x)$. The following result shows an
equivalence relationship between $(\rho, u)$ and $(\tilde{\rho},\tilde{u})$, which is the inverse of Theorem~\ref{thm:nfr}.
\begin{theorem}
Consider a fixed sequence $\rho_0=\{\rho_0^{(\ell)}\}$, and any $(\tilde{\rho},\tu)$. Assume $\tilde{\rho}^{(\ell)}\sim \rhoo^{(\ell)}$, $\ell\in [L]$.  There exists $({\rho},u)$ such that
\eqref{eq:nfr-f-layer}, \eqref{eq:nfr-f-top}, \eqref{eq:nfr-r-layer}, \eqref{eq:nfr-r-top} hold. 
\label{thm:nfr-inv}
\end{theorem}

Based on the feature repopulated formulation of continuous DNN in Theorem \ref{thm:nfr} and the  equivalence  between $(\rho, u)$ and $(\tilde{\rho}, \tilde{u})$ shown in Theorem \ref{thm:nfr-inv}, we know that learning a continuous DNN by optimizing over $(\rho, u)$ is equivalent to minimizing the following feature repopulated objective function $\tilde{Q}(\tilde{p}, \tilde{u})$ over $(\tilde{\rho},\tilde{u})$:
\begin{align}
    \min_{\tilde{\rho}, \tilde{u}} \tilde{Q}(\tilde{\rho}, \tilde{u}) = J(f(\rho_0, \tilde{u};\cdot)) + \tilde{R}(\tilde{\rho}, \tilde{u}), \label{eqn:tQ}
\end{align}
where 
\begin{align}
\tilde{R}(\tilde{\rho}, \tilde{u}) = \sum_{\ell=1}^{L}\lambda^{(\ell)}\tilde{R}^{(\ell)}(\tilde{\rho}) + \lambda^{(u)} \tilde{R}^{u}(\tilde{\rho}, \tilde{u}),\label{eqn:tR}
\end{align}
when $\rho$ and $\trho$ are restricted to the space $\{\rho: \rho^{(\ell)}\sim \rho^{(\ell)}_0, \ell \in [L]\}$. Here, we should keep in mind that $\rho_0$ is fixed and known.  It follows that the continuous DNN $f(\rho, u; x)$ is equivalent to a linear system $f(\rhoo,\tilde{u};x)$ parameterized by $\tilde{u}$. Thus, (\ref{eqn:tQ}) demonstrates that we can decouple the probability measure sequence $\rho$ from the loss function $J$, and the effect of $\rho$ in the objective function only shows up through the regularizer $\tilde{R}$ after reparameterizing NN using $(\tilde{\rho},\tu)$. This reparameterization significantly simplifies the objective function.

Based on the new formulation (\ref{eqn:tQ}),  given $\tilde{u}$, the quality of the feature distributions $\rho$ depends on the regularizer.  In the next section, we will  discuss the  properties of continuous DNN with specific regularizers. Especially, we will show in Section \ref{ell 1 2} that the $\ell_{2,1}$ norm regularization  leads to efficient distributions over features in terms of representing a given  target  function.

Moreover, our NFR view implies a process to obtain improved feature representations  starting from any $(\rho_0,\tu)$. See Algorithms \ref{alg:continuous} and \ref{alg:discrete} in the experiment section for details.  

\subsection{ Properties of Continuous DNN with Specific Regularizers }\label{sec:global-optimum} 
In the following, we show some consequences of NFR by specifying the regularizers. We will study the class of $\ell_{1,2}$ norm regularizers proposed in Section \ref{subsec:variance-of-discrete-approximization} and the standard $\ell_{p,1}$ norm regularizers ($p\geq 1$)  commonly used in practice. We show that for the $\ell_{1,2}$ norm regularizers, the overall objective under NFR is convex. And for  the class of $\ell_{p,1}$ ($p\geq1$)  norm regularizers,   the minimization problem for $\trho$ when fixing $\tilde{u}$ is also ``nearly'' convex. 
Moreover,  $\ell_{1,2}$ norm regularizers guarantee  learning efficient feature representations for the underlying learning tasks.

\subsubsection{$\ell_{1,2}$ Norm Regularizers}\label{ell 1 2}




We consider the  regularizers defined in (\ref{eqn:R-specific}), where we pick $r_1(\omega)=|\omega|$, $r_2(\omega) = \omega^2$, and $r^{(u)}(u)=\|u\|^2$ with $\omega \in \R$ and $u \in \R^K$ in (\ref{eqn:R-continuous}). We have the following theorem about the structure of the objective function.

\begin{theorem}\label{thm:convexity}
Assume that the loss function $\phi(\cdot, \cdot)$ is convex in the the first argument. Let $r_1(\omega)= | \omega|$, $r_2 (\omega) =   \omega^2$, and $r^{(u)}(u) = \|u \|^2$ with $\omega\in \R$ and $u\in \R^K$, then $\tilde{Q}(\tilde{\rho}, \tilde{u})$ is a convex function of $(\tilde{\rho}, \tilde{u})$, where $(\tilde{\rho}, \tilde{u})$ is induced by $\{(\rho, u): \rho^{(\ell)}\sim \rho_0^{(\ell)}, \ell \in [L]\}$.
\end{theorem}

\begin{remark}
It can be shown that for some other choices such as $r_1(\omega)=|\omega|^{o_1}$ and $r_2(\omega)=|\omega|^{o_2}$ for $o_1 \in [o_2^{-1},1]$, 
the resulting regularizer is also convex. More details are given  in Appendix~\ref{app:41}.
\end{remark}

 Although Theorem~\ref{thm:convexity} is stated for a fixed $\rho_0$, we can pick $\rho_0$ to be an arbitrary probability measure sequence. In particular, if we take $\rho_0=\rho$ at the current solution, then we can use NFR to study the local behavior of the objective function around $\rho=\rho_0$.
 Since the NFR reparameterization has one-to-one correspondence with the original parameterization locally, we may conclude that a local solution of NN in the original parameterization at $\rho=\rho_0$ is also a local solution with respect to the NFR reparameterization. Since the objective function is still convex with the NFR reparameterization for this $\rho_0$, we conclude that a local solution of NN in the original parameterization is a global solution.  Note that the argument is also used in the proof of Corollary~\ref{KKT:condition} to derive the KKT conditions of such a local solution.
 We summarize the result informally as follows.
 \begin{corollary}
 Under the assumptions of Theorem~\ref{thm:convexity}. If $(\rho,u)$ is a local minimum of \eqref{eqn:Q-continuous}, and assume that $\rho^{(\ell)} \sim \rho_0^{(\ell)}$ for $\ell \in [L]$, then $(\rho,u)$ achieves the global minimum of \eqref{eqn:Q-continuous}.
 \label{cor:local-global}
 \end{corollary}

 Theorem \ref{thm:convexity} shows that a continuous  NN can be reformulated as a convex model under the NFR re-parameterization. This result is quite unexpected, and it can be used to explain mysterious empirical observations in DNN.  For example, it is known that overparameterized DNNs are easier to optimize. 
This can be explained by Corollary~\ref{cor:local-global}.

Compared to the NTK view, our theory is also more consistent with practice observations. First, in the NTK view,  the convexity holds only  when the variables are restricted in an infinitesimal region.  In contrast, our result can be applied globally. In addition, the NTK view  essentially  treats an NN as a linear model on an infinite dimensional space of random features. The random features are not learned from the underlying task. In contrast, our results can explain that NNs learn useful features for the underlying task when  they are fully trained.  In fact, by using convexity and the NFR technique,  we can establish specific properties satisfied by the optimal solutions of DNNs.
\begin{corollary}\label{KKT:condition}
In the space $\{(\rho, u): \rho^{(\ell)}\sim \rho_0^{(\ell)}, \ell \in [L]\}$, if $(\rho_*, u_*)$  is an optimal solution of the DNN equipped with $\ell_{1,2}$ norm regularizers, then there exists a real number sequence $\{\Lambda\}_{\ell\in[L]}$, i.e.  $\Lambda^{(\ell)} \in\RR$ for all $\ell\in[L]$, so that the following equations hold:
\begin{enumerate}[(i)]
    \item For all $\ell\in[L-1]$ and $\tilde{z}^{(\ell)} \in  \mathcal{Z}^{(\ell)}$, we have
    \begin{align}
        &\lambda^{(\ell+1)}\left(\int\left|w (\tilde{z}^{(\ell+1)}, \tilde{z}^{(\ell)}) \right|  d\rho_*^{(\ell+1)}(\tilde{z}^{(\ell+1)})\right)^2\notag\\
        =&\Lambda^{(\ell)}+  2\lambda^{(\ell)}\int\int\left|w (z^{(\ell)}, \tilde{z}^{(\ell-1)}) \right|  d\rho_*^{(\ell)}(z^{(\ell)}) \left|w (\tilde{z}^{(\ell)}, \tilde{z}^{(\ell-1)})\right|d\rho_*^{(\ell-1)}(\tilde{z}^{(\ell-1)})\notag.
    \end{align}
\item   For all  $\tilde{z}^{(L)} \in  \mathcal{Z}^{(L)}$, we have     
$$ \Lambda^{(L)}+  2\lambda^{(L)}\int \int\left|w (z^{(L)}, \tilde{z}^{(L-1)}) \right|  d\rho_*^{(L)}(z^{(L)}) \left|w (\tilde{z}^{(L)}, \tilde{z}^{(L-1)})\right|d\rho_*^{(L-1)}(\tilde{z}^{(L-1)}) =\lambda^{(u)}\left\|u_* (\tilde{z}^{(L)})\right\|^2, $$
 and
 $$  \EE_{x,y} \left[J'(f(\rho_*, u_*;x)) f^{(L)}(\rho_*, \tilde{z}^{(L)};x)\right] =-2\lambda^{(u)} u_*(\tilde{z}^{(L)}). $$
\end{enumerate}
\end{corollary}

The equations in  Corollary \ref{KKT:condition} will be validated in our experiments. They imply that the consequences of the NFR theory are consistent with empirical observations. 

Corollary \ref{KKT:condition} shows that the optimal feature distribution sequence $\rho_*$ relies on $u_*$, where $f(\rho_*,u_*;x)$ represents the target function, and it can be rewritten as $f(\rhoo,\tu;\cdot)$ with a fixed $\rho_0$.    In fact, given the desired target function $f(\rhoo,\tu;\cdot)=f(\rho_*, u_*;\cdot)$,  there can be many equivalent representations $f(\rho,u;\cdot)$ indexed by $\rho$ under NFR (refer to Section \ref{sec:importance-weighting}).
The optimal $\rho_*$ achieves the minimum $\ell_{1,2}$ norm regularization value under this equivalent class of functions
that achieve the same outputs as $f(\rhoo,\tu;\cdot)$. 
Since the  $\ell_{1,2}$ norm regularization upper bounds the variance of discrete approximation of the continuous DNN in Theorem~\ref{var}, a small $\ell_{1,2}$ norm implies that a small number of hidden units are needed to represent $f(\rho_*, u_*;\cdot)$ in the randomly sampled discrete DNN.
This means that $\ell_{1,2}$ norm regularization leads to efficient feature representations. This result generalizes a corresponding result for two-level NNs in \citep{fang2019over}.

\subsubsection{$\ell_{p,1}$ Norm Regularizers}
We also propose some results for the commonly-used $\ell_{p,1}$ ($p\geq 1$) norm regularizers. This type of regularizers can be written in \eqref{eqn:R-continuous} 
by picking $r_1(\omega) = |\omega|^{q^{(\ell)}}$,  $r_2 (\omega) =   |\omega|$, and $r^{(u)}(u) = \|u \|^{q^{(u)}}$.  We have the property below: 

\begin{theorem}\label{theorem:global-optimum}
Given $\tilde{u}$,  suppose $q^{(\ell)}\geq1$ with $\ell\in [ L]$ and $q^{(u)}\geq1$. Then in the space $\{(\rho, u): \rho^{(\ell)}\sim \rho_0^{(\ell)}, \ell \in [L]\}$, if $\tilde{\rho}_*$ is a local solution of
\begin{equation}\label{tilderho}
\min_{\tilde{\rho}} \tilde{Q}(\tilde{\rho}, \tu)   
\end{equation}
 then $Q(\tilde{\rho}_*, \tu ) = \min_{\trho} Q(\tilde{\rho}, \tu )$.
\end{theorem}
Theorem \ref{theorem:global-optimum} shows that given $\tilde{u}$,  minimization of  $\tilde{Q}$ over $\tilde{\rho}$ behaves like ``convex optimization'',  in which any   local solution of  $\tilde{\rho}_*$ is a global solution  that achieves   minimum value of $\tilde{Q}(\cdot,\tu)$. Further remarks are discussed below:
\begin{enumerate}
    \item From Theorem \ref{theorem:global-optimum}, we know that given  $\tilde{u}$,  solving $\tilde{\rho}$ is relatively simple.
    This means  given a target output function $f(\rho_0, \tu;\cdot)$, it is efficient to learn the desired distributions over features under the  $\ell_{p,1}$ norm regularization condition.   
    
    \item  We  note that in the objective $\tilde{Q}$, the loss function   $J(f(\rho_0, \tu;\cdot))$  is convex.   In real applications, the loss function value usually  dominates the regularizers because one needs to choose small regularization parameters $\lambda^{(\ell)}\approx 0$.      In such case,  the objective function is nearly convex and therefore all local minima have loss function values close to that of the global minimum. 
    This explains the empirical observation that for overparameterized NNs, there are no ``bad'' local minima when the networks are fully trained until convergence.
    \item  Theorem \ref{theorem:global-optimum} also indicates that the optimization problem of DNN, when equipped with $\ell_{p,1}$, involves special structures. Therefore solving this class of nonconvex optimization problems is potentially much easier than minimizing a general nonconvex function. A more careful analysis of this observation will be left as a future research direction. 
\end{enumerate}



\section{Experiments}
\label{sec:experiments}

The experiments are designed to qualitatively verify the following.
\begin{enumerate}
\item \emph{Optimality condition}: We demonstrate that fully trained overparameterized DNNs are consistent with the NFR theory by verifying  the optimality condition in Corollary \ref{KKT:condition}. Here we consider the relationship between
\begin{equation}u^{(\ell)}_j=\frac{\lambda^{(\ell)}}{m^{(\ell)} m^{(\ell-1)}}\sum_{k=1}^{m^{(\ell-1)}}\big(\sum_{j'=1}^{m^{(\ell)}}|w_{k,j'}^{(\ell)}|\big)|w_{k,j}^{(\ell)}|\label{exp:u}
\end{equation}
and
\begin{equation}v^{(\ell)}_j=\left\{
\begin{array}{lcl}
\lambda^{(\ell+1)}\big(\frac{1}{m^{(\ell+1)}}\sum_{i=1}^{m^{(\ell+1)}}{|w_{j,i}^{(\ell+1)}|}\big)^2  &      & \ell \in [L-1]\\
\lambda^{(u)} \|u_j\|^2 & & \ell=L
\end{array} \right.\label{exp:v}
\end{equation}
 for one neuron $j$ in layer $\ell \in [L]$, which are the estimates of  \[
 \lambda^{(\ell)}\int \left[\int |w(\tilde{z}^{(\ell)},z^{(\ell-1)})|d\rho(\tilde{z}^{\ell})\right] |w(z^{(\ell)},z^{(\ell-1)})| d\rho(z^{(\ell-1)})
 \]
 and
 \[
 \left\{
 \begin{array}{lcl}
 \lambda^{(\ell+1)}\big( \int |w(z^{(\ell+1)},z^{(\ell)})|d\rho(z^{(\ell+1)})\big)^2 & & \ell \in [L-1]\\
\lambda^{(u)}\big\| u(z^{(L)})\big\|^2
& &
\ell = L
\end{array}
\right.
\]
 respectively.
\item \emph{Deep versus Shallow Networks}: We show that by increasing $L$, the number of hidden layers, fully connected NN can learn hierarchical feature representations that can reduce the variance of approximation described in Theorem~\ref{var}. This verifies the benefit of using deeper networks for certain problems.  
\item\emph{Compactness}: We show that compared with other regularizers, the proposed regularizer can learn better (more compact) feature representations.
\item\emph{NFR process}: We show that a discrete neural feature repopulation algorithm motivated by our theory can effectively reduce the training loss, and especially the regularizer. This leads to faster convergence to better feature representations. 
\end{enumerate}

Note that similar to \citep{fang2019over},
we use the approximation variance of discretization $V(w,u)$ to measure the effectiveness of feature representation,
based on the theoretical findings of Theorem~\ref{var}:
\begin{align}
V(w, u)=&\E_x \sum_{\ell=2}^{L} \frac{1}{(m^{(\ell-1)})^2}
\sum_{j=1}^{m^{(\ell-1)}}  \left(
  \sum_{i=1}^{m^{(\ell)}} a_i (x) \big(\hat{f}^{(\ell-1)}_j(x)
  w_{i,j}^{(\ell)} -\hat{g}^{(\ell)}_i(x)\big)\right)^2 \nonumber \\
 & ~~~~+ \frac{1}{(m^{(L)})^2}\sum_{j=1}^{m^{(L)}} \left\|u_j\hat{f}^{(L)}_j(x)-\hat{f}(x)\right\|^2,\nonumber
 \end{align}
where $a_i(x) = \frac{\partial \hat{f}(x)}{\partial \hat{g}^{(\ell)}_i(x)}$.
 
 \subsection{Neural Feature Repopulation Algorithm}
\label{sec:algo}
We propose a new optimization process inspired by  our NFR view to verify its effectiveness. This process is complementary to the standard SGD procedure and can be used to accelerate the learning of feature distributions.

\begin{algorithm}  
  \caption{Continuous Optimization Procedure}   
  \begin{algorithmic}[1]  
      \STATE {\bfseries Initialization:}  $\tilde{\rho}_1^{(\ell)}=\rho_0^{(\ell)}$, $\ell=0,\ldots, L$
      \FOR{$t=1$ to T}
      \STATE \label{line33} Update $\tilde{\u}_t$ by fixing $\tilde{\rho}_t$ and solving
	\begin{align}
	\tilde{\u}_{t+1} = \arg \min_{\tilde{\u}} J(f(\rho_0, \tilde{\u};\cdot) )
	+ \tilde{R}(\tilde{\rho}_t,\tilde{\u})\nonumber
	\end{align}
	\STATE Update $\tilde{\rho}_t$ by fixing $\tilde{\u}_{t+1}$ and solving
	\begin{align}
	\tilde{\rho}_{t+1} =\arg\min_{\tilde{\rho}} \tilde{R}(\tilde{\rho},\tilde{\u}_{t+1})\nonumber
	\end{align}
      \ENDFOR
  \STATE {\bfseries Return:} $\tilde{\rho}_{T+1}$ and $\tilde{u}_{T+1}$
  \end{algorithmic}  
  \label{alg:continuous}
\end{algorithm}  

We first present our procedure for the continuous DNN in Algorithm \ref{alg:continuous}, in which  we alternatively fix  either $\tilde{\rho}$ or $\tilde{u}$ and update the other to minimize the objective function. Due to our feature repopulation procedure, the loss $J(f(\rho_0, \tilde{u};\cdot))$ would be a constant when $\tilde{u}$ is fixed. Therefore, we only need to minimize the regularizer $\tilde{R}$ when we update $\tilde{\rho}$ (see line 4).  Such process explicitly improves the quality of features in terms of efficient representation. 
Algorithm \ref{alg:discrete} is the discrete version\footnote{A more detailed discrete procedure is shown in Algorithm \ref{alo:alo} in Appendix \ref{sec:alo}. } of Algorithm \ref{alg:continuous}. We combine it with SGD in Line \ref{line33}.

\begin{algorithm}
  \caption{Discrete Optimization Procedure}
  \begin{algorithmic}[1]  
      \STATE {\bfseries Initialization:}  Start with initial $(\hat{w}_0,\hat{u}_0)$
      \FOR{$t=1$ to T}
      \STATE Run multiple steps of SGD to obtain updated $(\tilde{w}_t,\tilde{u}_t)$ from $(\hat{w}_{t-1},\hat{u}_{t-1})$
      \STATE Find $\hat{p}$ to optimize
      \begin{align*}
      \hat{p}=\arg\min_p \hat{R}(p;\tilde{w}_t,\tilde{u}_t)
      \end{align*}
      with $\hat{R}$ defined in \eqref{eq:reg-im}
      \STATE  Set initial weights as $(\hat{w}_t,\hat{u}_t)$
      \FOR{$\ell=1$ to $L$}
      \STATE Resample the hidden layer $\ell$ for $m^{(\ell)}$ nodes according to $\hat{p}^{(\ell)}$. 
      \STATE Duplicate weights connected to each node before sampling to form the updated weights after sampling.
      \ENDFOR
      \ENDFOR
      \STATE {\bfseries Return:} $\hat{w}_{T}$ and $\hat{u}_{T}$
  \end{algorithmic}  
  \label{alg:discrete}
\end{algorithm}  

\begin{figure*}[htb!]
    \begin{center}
        \subfigure[synthetic, $\ell=1$]{\includegraphics[scale=0.24]{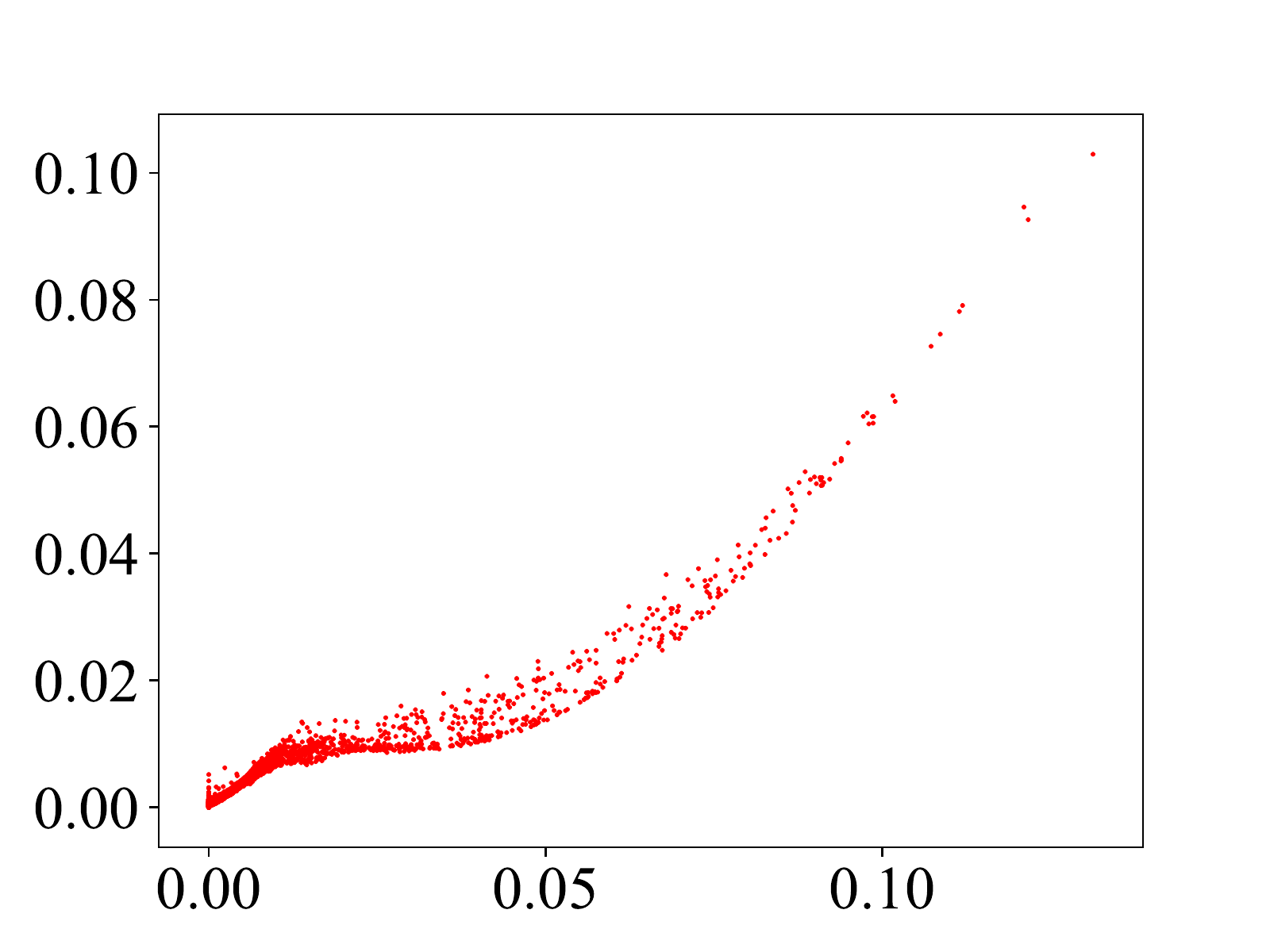}}
        \subfigure[synthetic, $\ell=2$]{\includegraphics[scale=0.24]{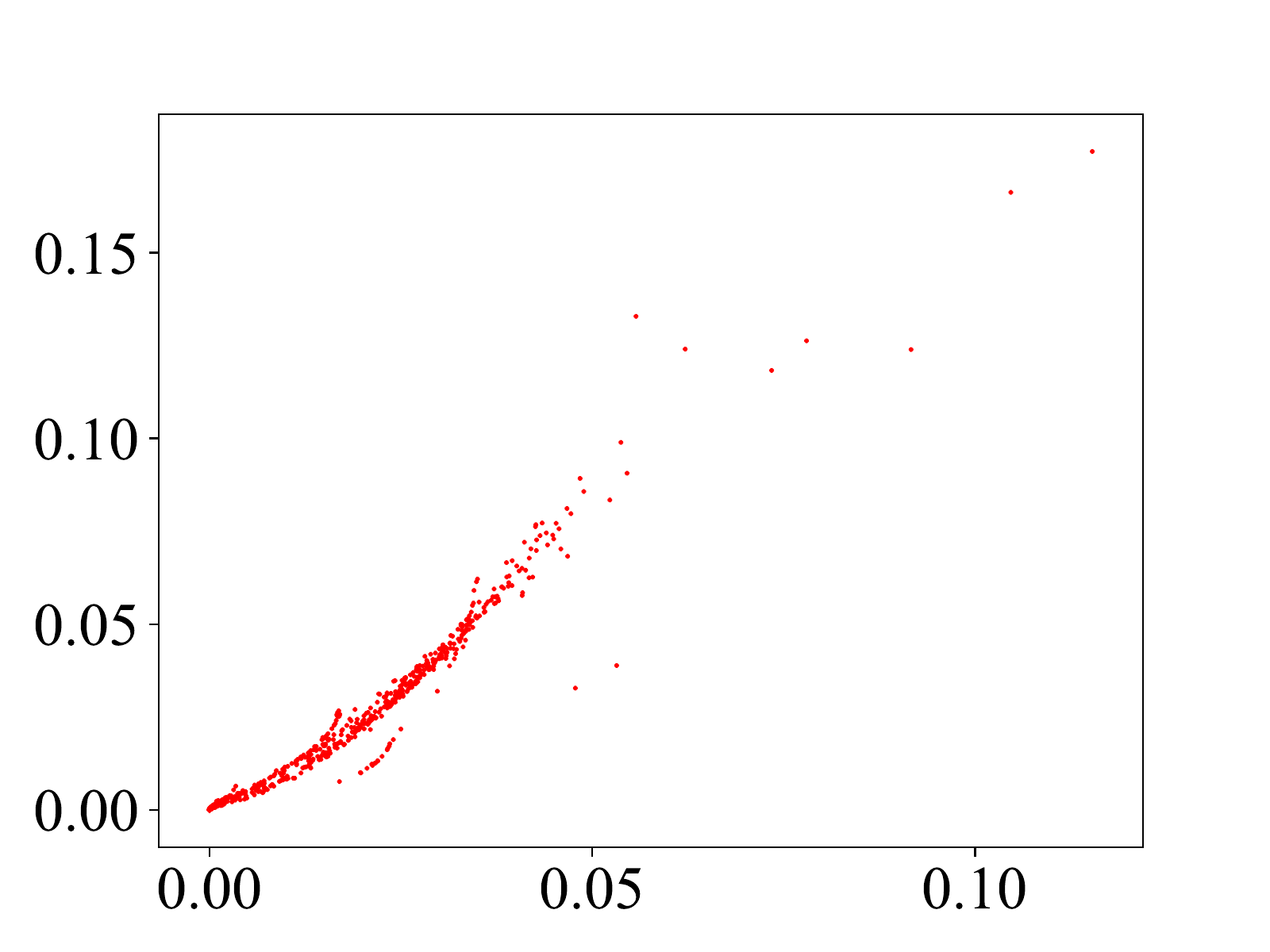}}
        \subfigure[synthetic, $\ell=3$]{\includegraphics[scale=0.24]{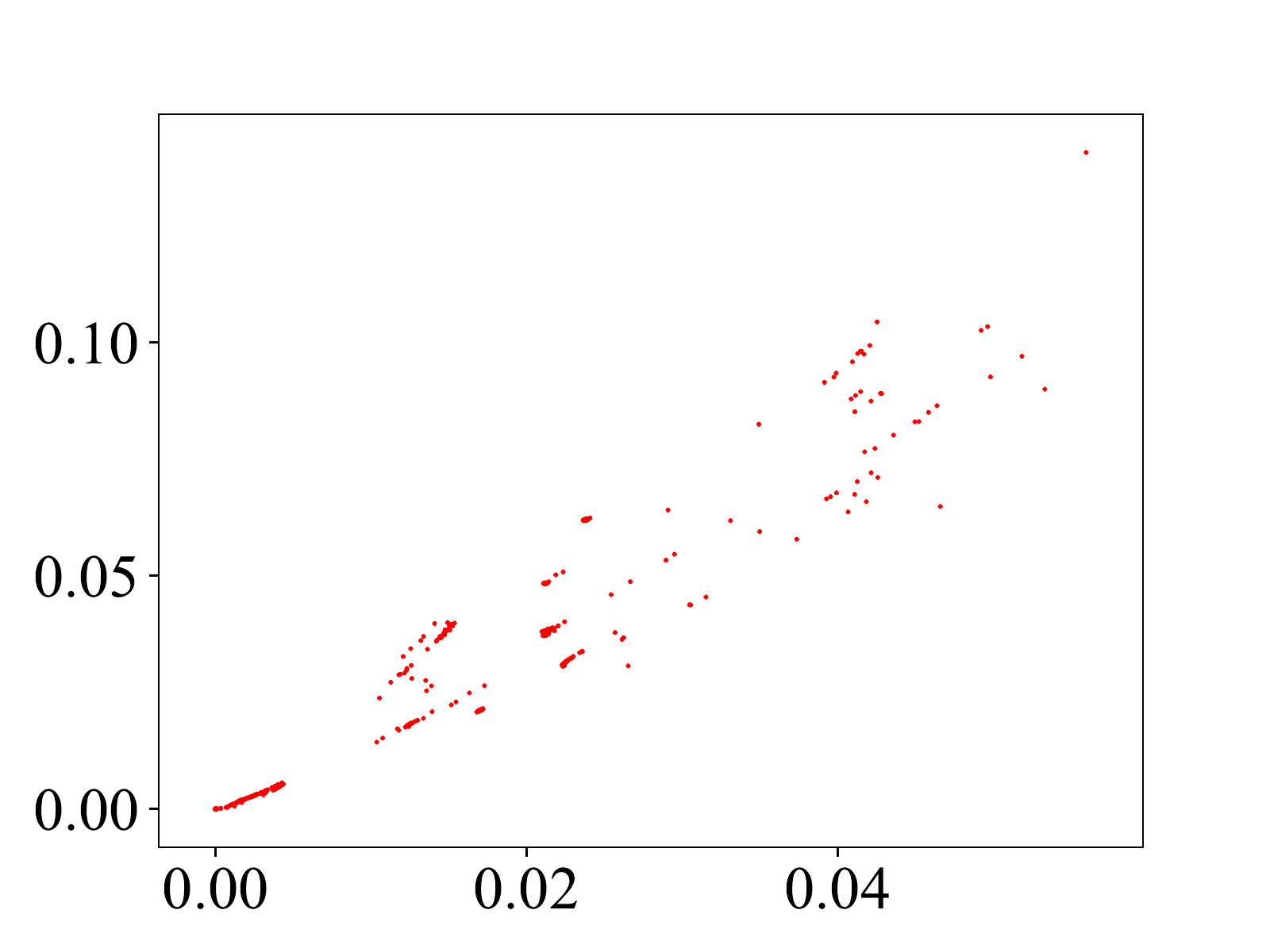}}
        \subfigure[synthetic, $\ell=4$]{\includegraphics[scale=0.24]{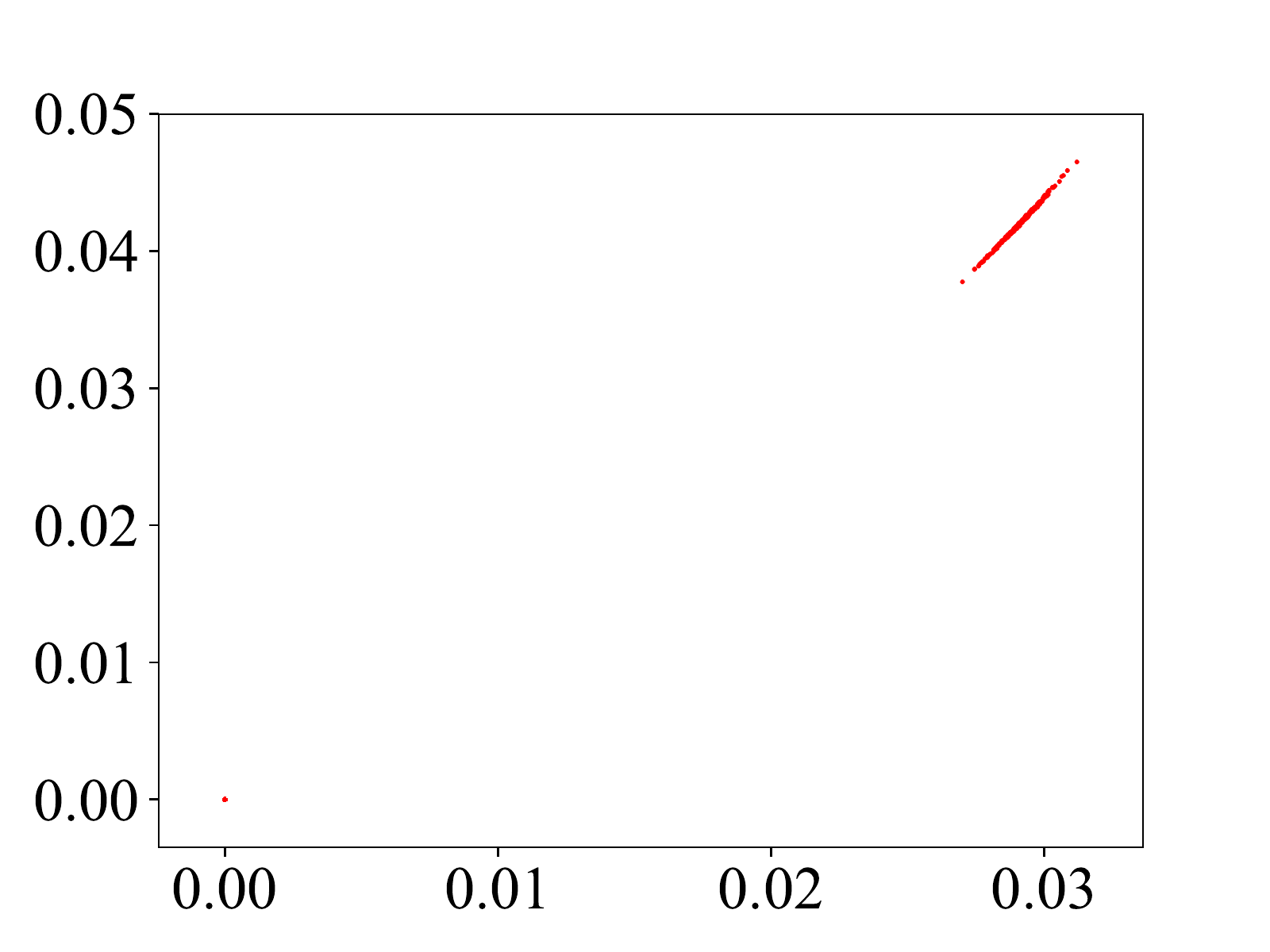}}
        \subfigure[mini-imagenet, $\ell=1$]{\includegraphics[scale=0.24]{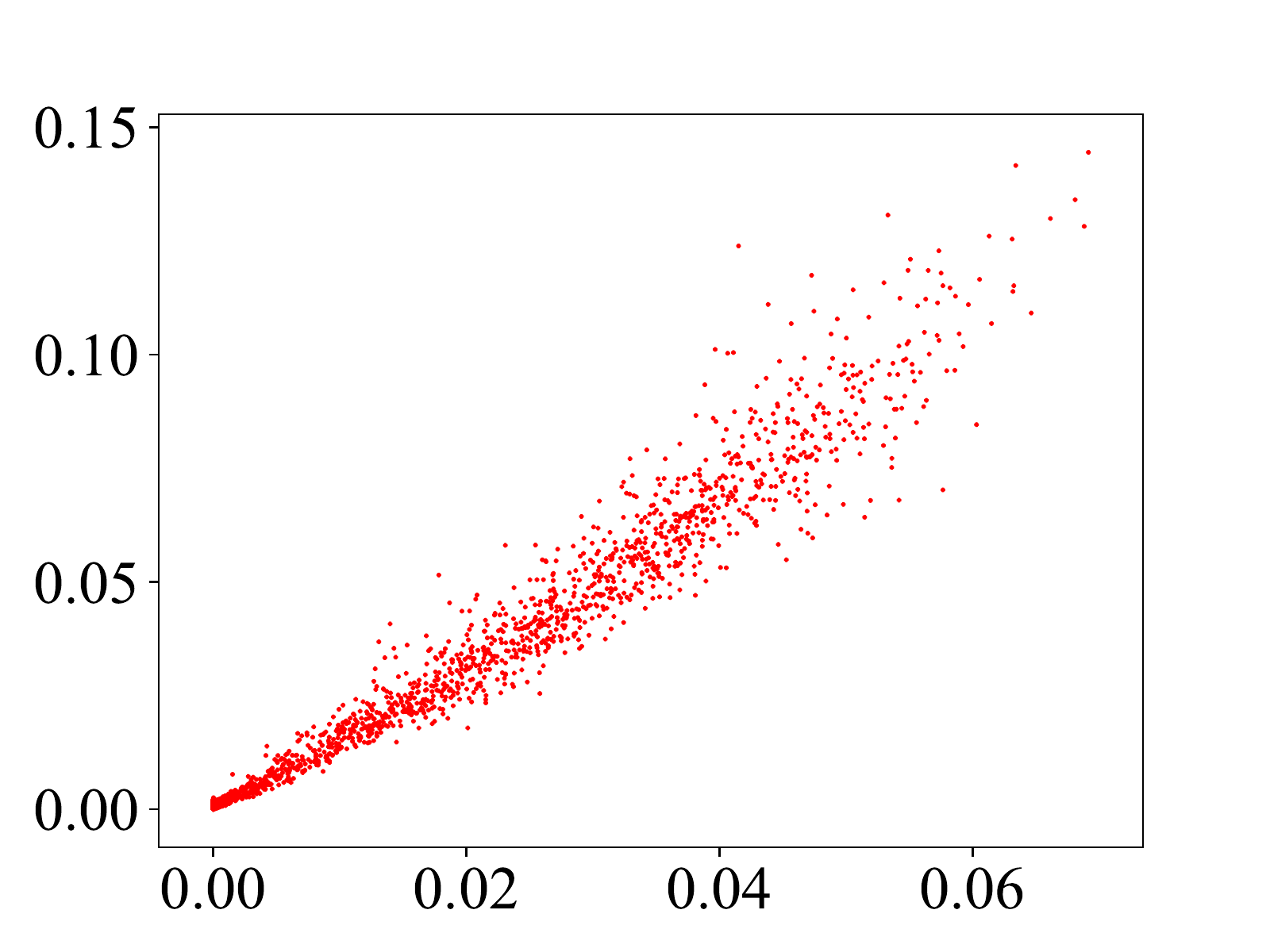}}
        \subfigure[mini-imagenet, $\ell=2$]{\includegraphics[scale=0.24]{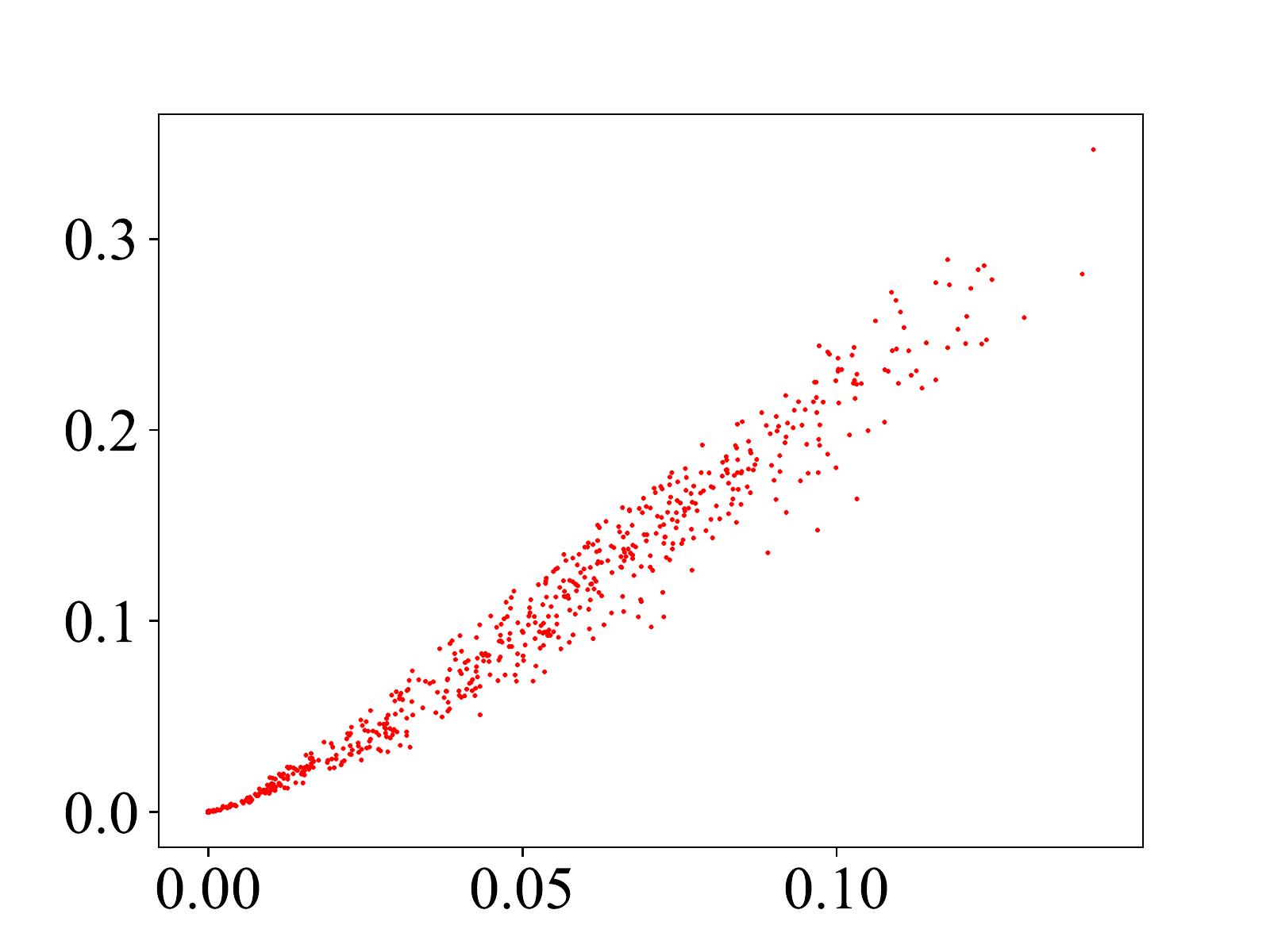}}
        \subfigure[mini-imagenet, $\ell=3$]{\includegraphics[scale=0.24]{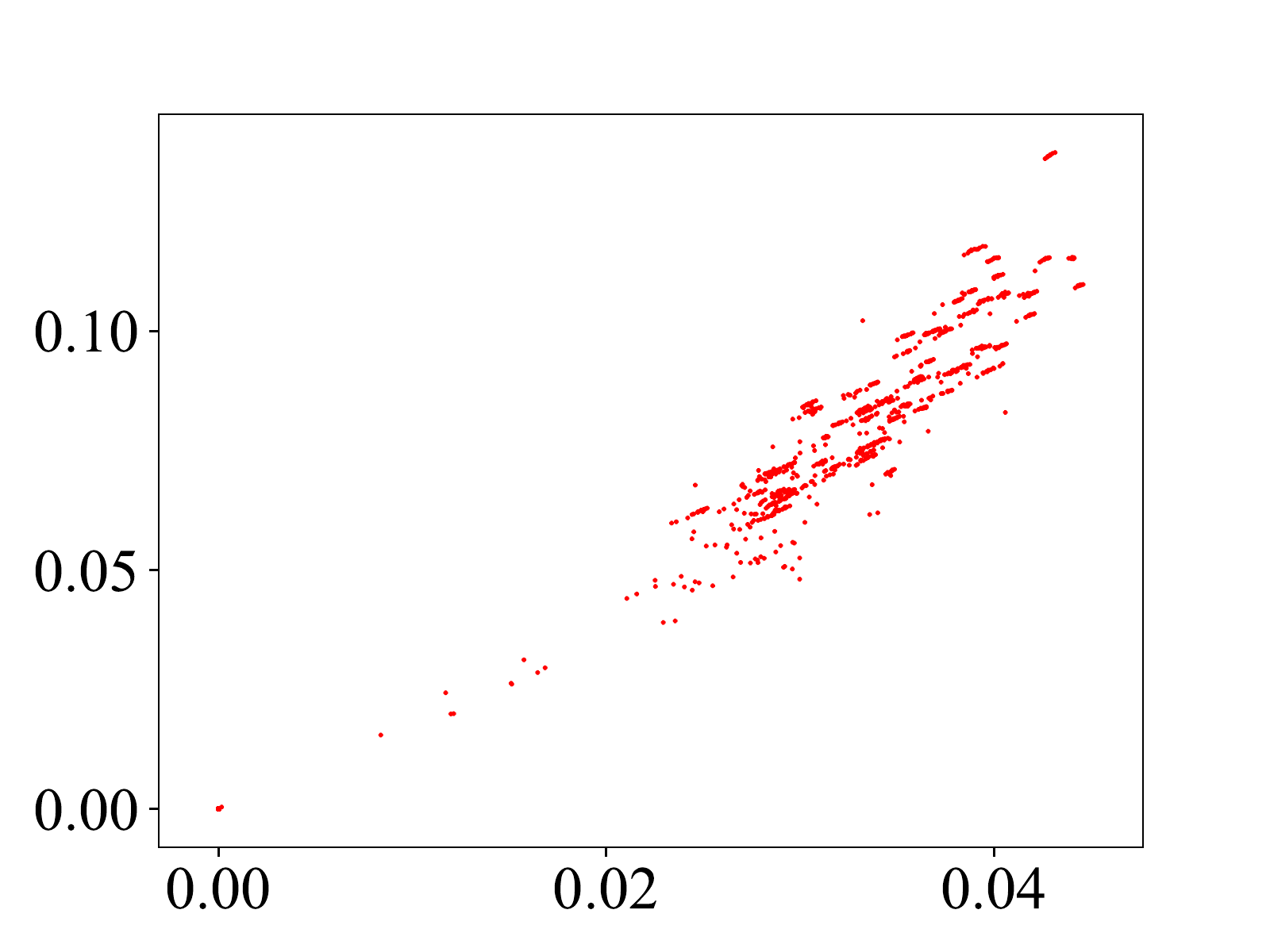}}
        \subfigure[mini-imagenet, $\ell=4$]{\includegraphics[scale=0.24]{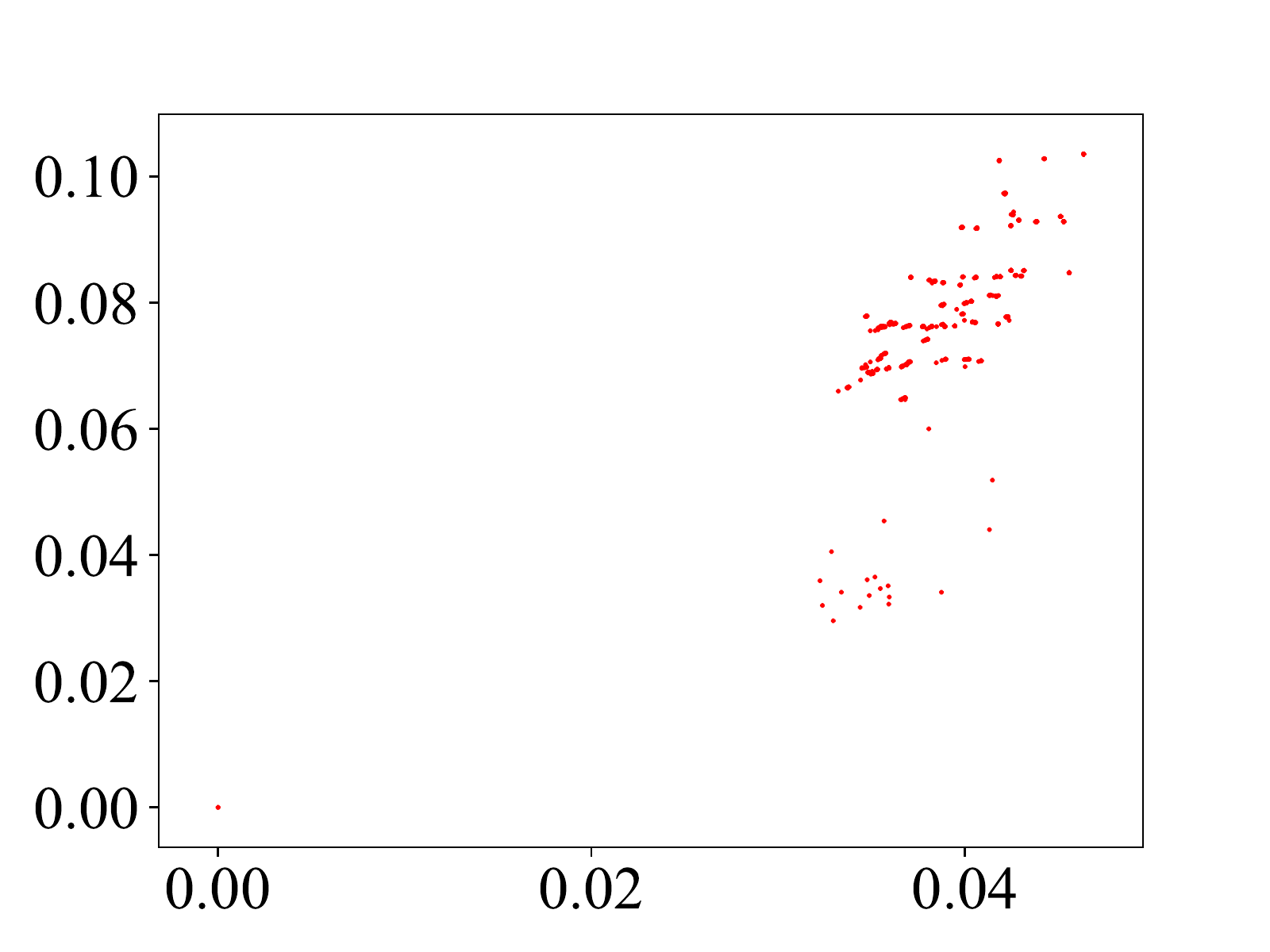}}
    \end{center}
    \vspace*{-20pt}
    \caption{Property of optimal solution for $L=4$ hidden-layer fully connected NN. 
    The top (synthetic 1-d regression) and bottom (mini-imagenet) rows are scatter-plots of the estimated quantities for each sampled neuron after convergence.
    For each subplot, one point $(x, y)$ represents one sampled neuron $j \in [m^{(\ell)}]$, where  $x=u_j^{(\ell)}$ and $y=v_j^{(\ell)}$ defined in \eqref{exp:u} and \eqref{exp:v} respectively. } 
    \label{exp:all:fig:optimal-sol}
\end{figure*}

\subsection{Synthetic 1-D regression task}

We begin to empirically validate our claims in a synthetic 1-D regression task. Since the feature representation $f^{(\ell)}_j(x)$ corresponding to each neuron in each layer $\ell \in [L]$ is a single-variable function, it can be easily visualized.

Here we consider the function $f(x)=2(2\cos^2(x)-1)^2-1$  introduced by \cite{mhaskar2017and}. We draw 60k training samples and 60k test samples uniformly from $[-2\pi, 2\pi]$ for $x$ and set $y=f(x)$. We use a fully-connected NN with $m^{(\ell)}=1000\times 2^{L-\ell}$ hidden units in each hidden layer $\ell$ to learn this target function. We take $L=\{1,\cdots, 4\}$, and use the Adam optimizer with an initial learning rate $1e$-$4$ in our experiments, and let the activation function be $\sigma(x)=\text{tanh}(x)$. For fair comparison, we tune the hyper-parameters of the weight of regularizer so that for different $L$, the NN could reach training RMSE of $1e$-$4$ when converge. This controls the representation power of the NN.

\begin{figure*}[htb!]
    \begin{center}
        \subfigure[]{\includegraphics[scale=0.25]{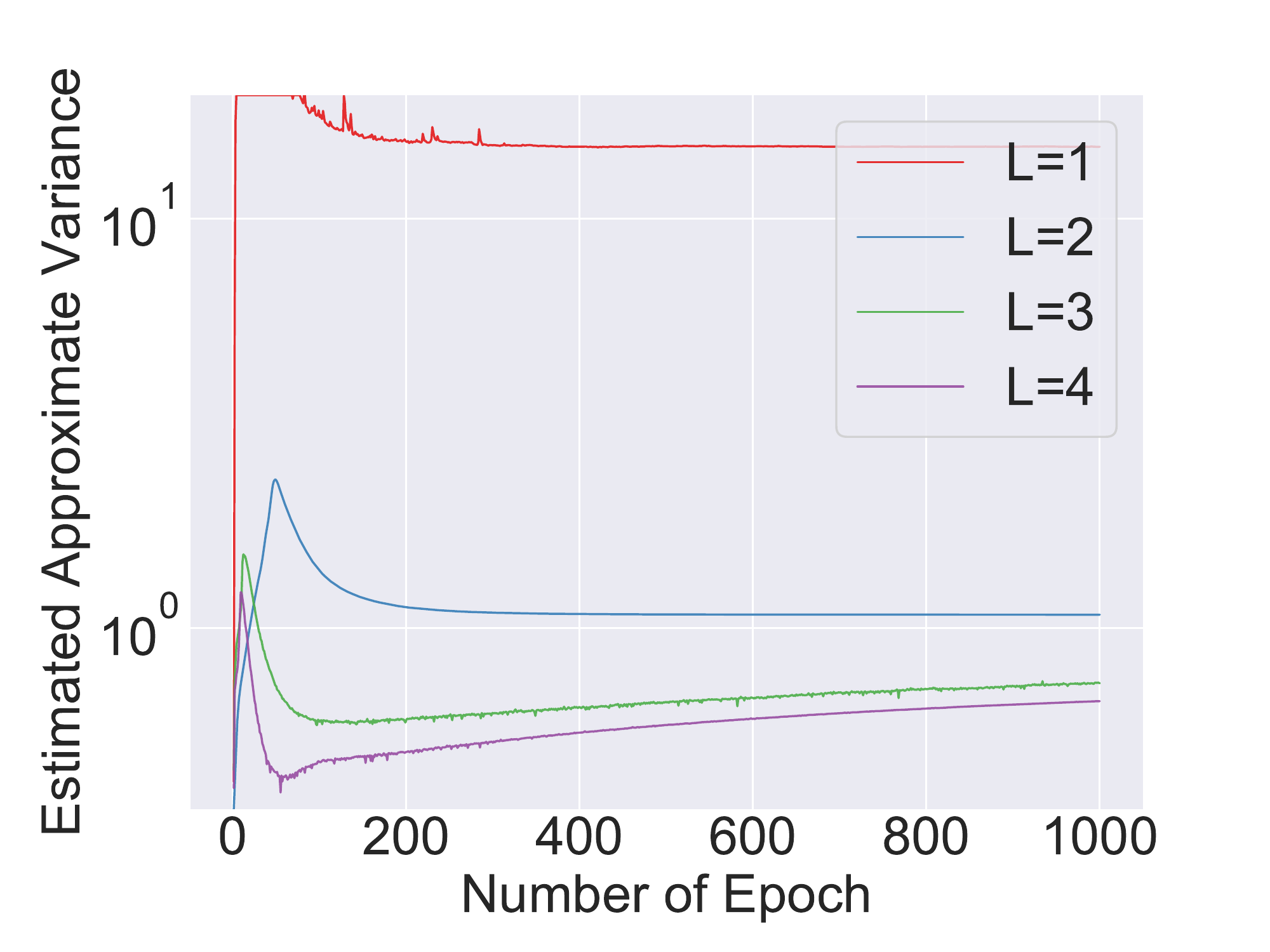}}
        \subfigure[]{\includegraphics[scale=0.25]{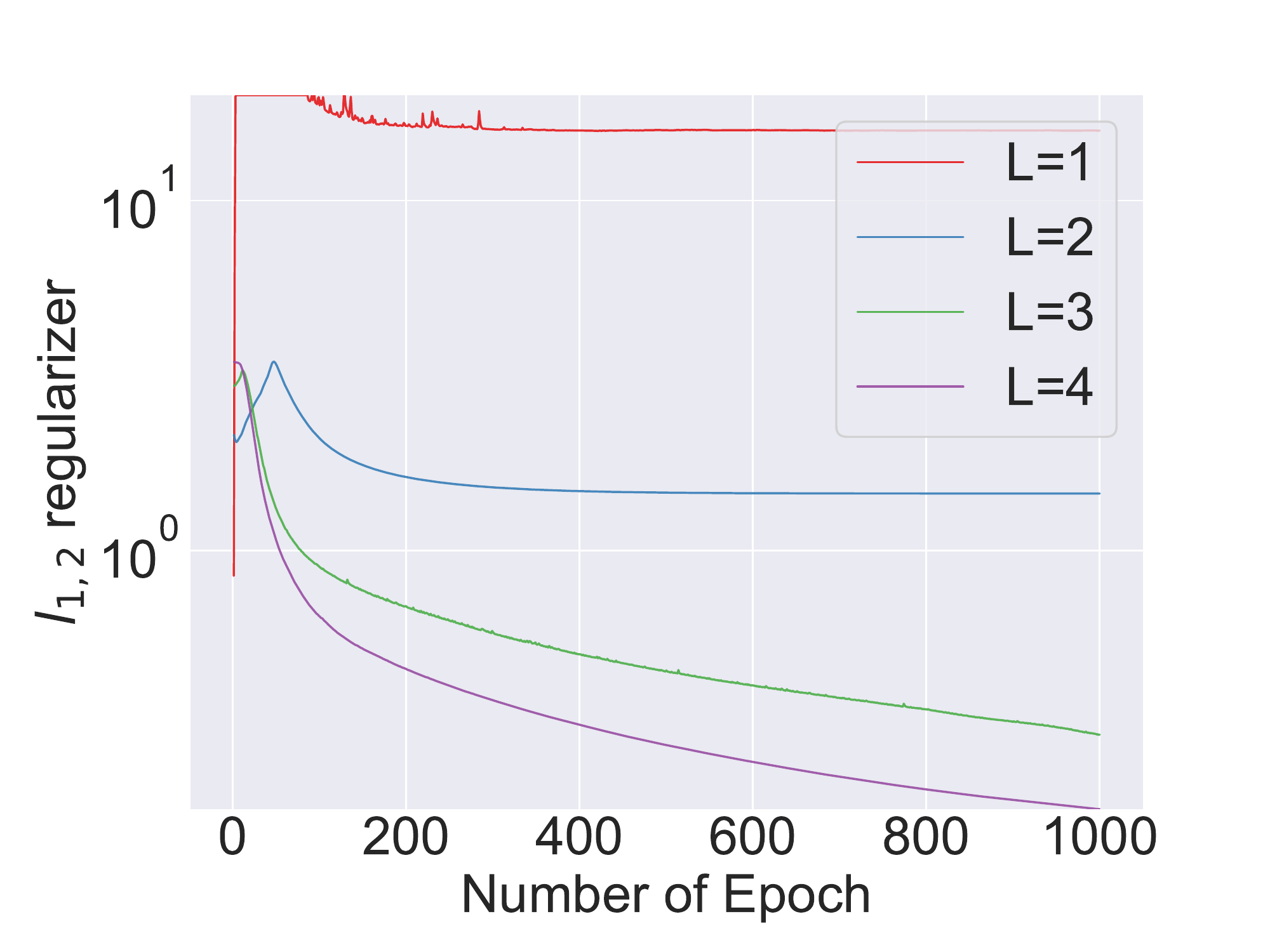}}
        \subfigure[]{\includegraphics[scale=0.25]{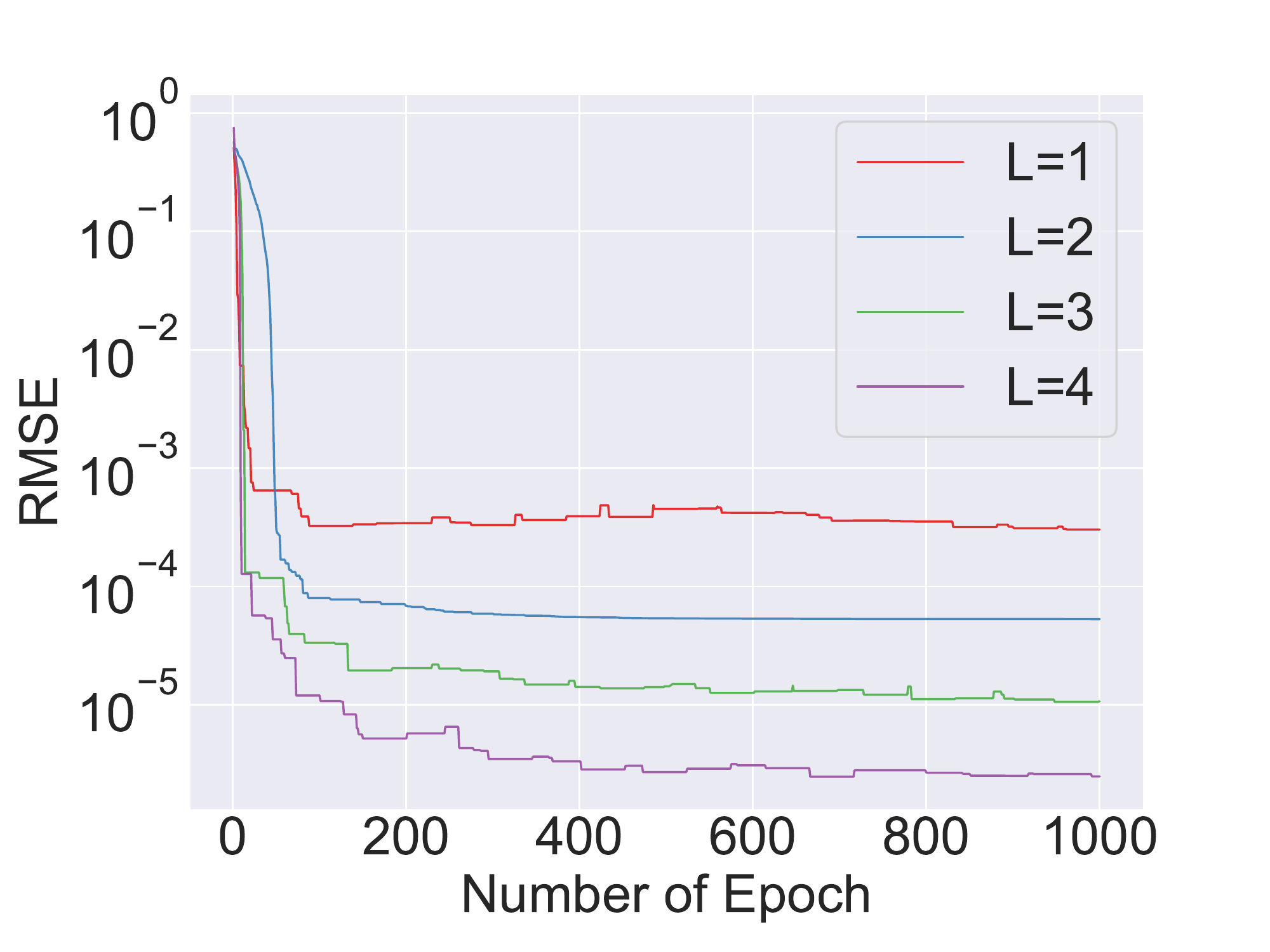}}
    \end{center}
    \vspace*{-20pt}
    \caption{The visualization of optimization process when training multilayer NNs with different $L$.}
    \label{exp:syn:fig:eav}
\end{figure*}

We first validate that fully trained overparameterized NN satisfies the optimality condition of Corollary~\ref{KKT:condition}. Here we consider the case of $L=4$, and the top row in Fig \ref{exp:all:fig:optimal-sol} plots the estimated quantities $u_j^{(\ell)}$ and $v_j^{(\ell)}$. We can see that these two quantities are approximately linearly correlated, as predicted by 
Corollary~\ref{KKT:condition}.

To compare the performance of shallow versus deep networks, Fig \ref{exp:syn:fig:eav} (a) reports how does the approximated variance change when $L$ increases. It demonstrates that the approximated variance decreases as $L$ increases. Moreover, the approximate variance gap between $L=2$ and $L=3$ is very large while that between $L=3$ and $L=4$ is small. This is consistent with the fact that the hierarchical composition of the target function $f(x)$ has depth $3$ (i.e. $f(x)=h(h(\cos x))$, where $h(u)=u^2-1$). At the same time, we can observe that for larger $L$, the regularizer in subplot (b) and training RMSE  in subplot (c) decrease much faster, which also demonstrate the effectiveness of increasing $L$ for this target function. 

\begin{figure}[htb!]
    \begin{center}
        \includegraphics[width=1.0\linewidth]{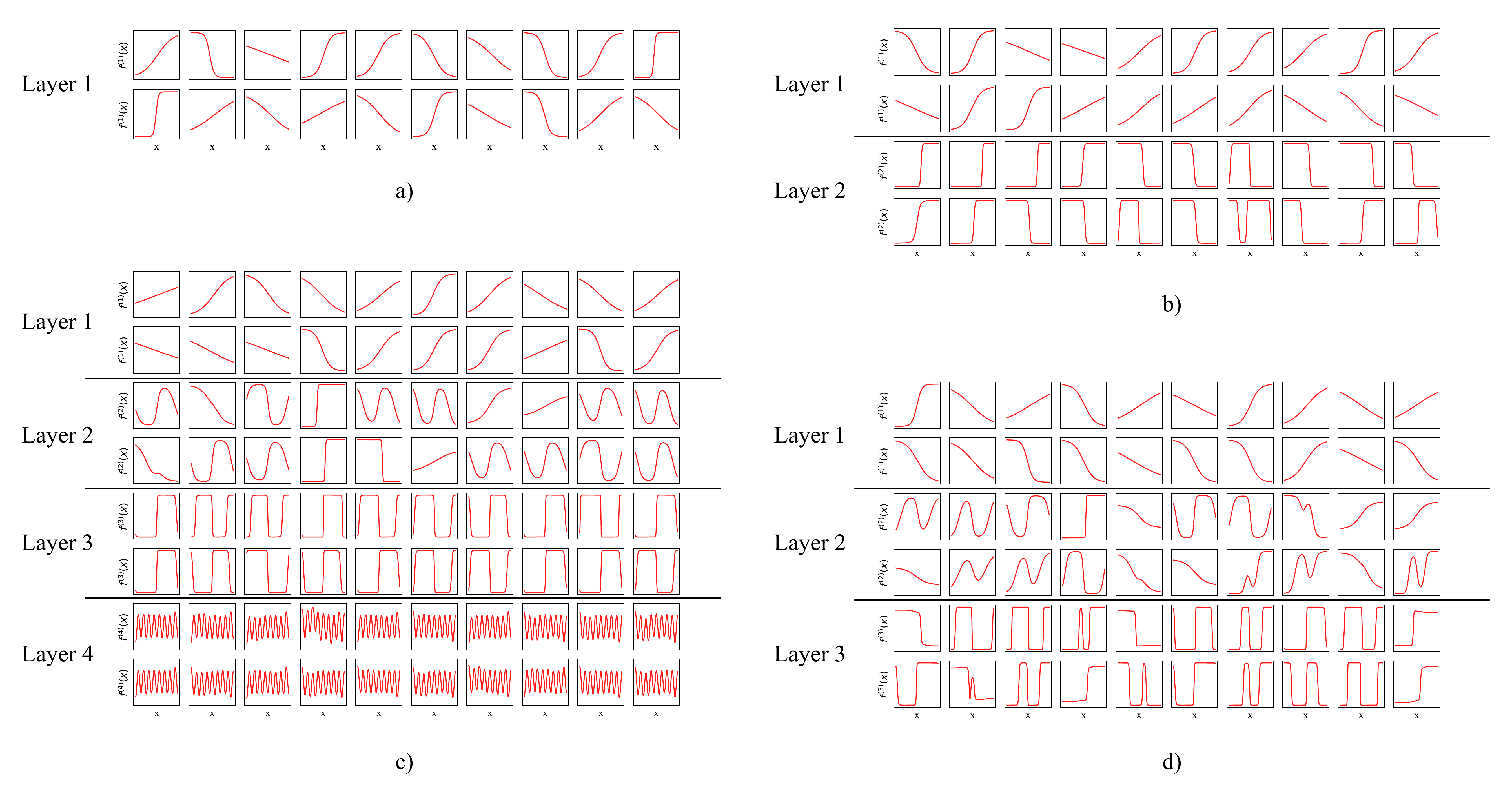}
    \end{center}
    \vspace*{-20pt}
    \caption{Representative features as 1-D functions for fully-connected NNs with different hidden-layers: a) 1-hidden-layer b) 2-hidden-layers c) 3-hidden-layers d) 4-hidden-layers. For each architecture, we sample 20 neurons from each layer and plot the single-variable feature function $f^{(\ell)}_j(x)$ of the sampled neurons.
    }
    \label{exp:syn:fig:feat-l21}
\end{figure}

Fig \ref{exp:syn:fig:feat-l21} shows representative features (as 1D-functions) at each layer after convergence. We reach the following conclusion from visualization of different $L$: DNN is able to learn hierarchical feature representations when we take optimization process into consideration.  To be more specific, the layer next to the input layer tends to learn low-frequency signals while the upper layers take these lower-frequency signals to form higher-frequency signals.

\begin{figure*}[htb!]
    \begin{center}
        \subfigure[weights of $w^{(1)}$]{\includegraphics[scale=0.32]{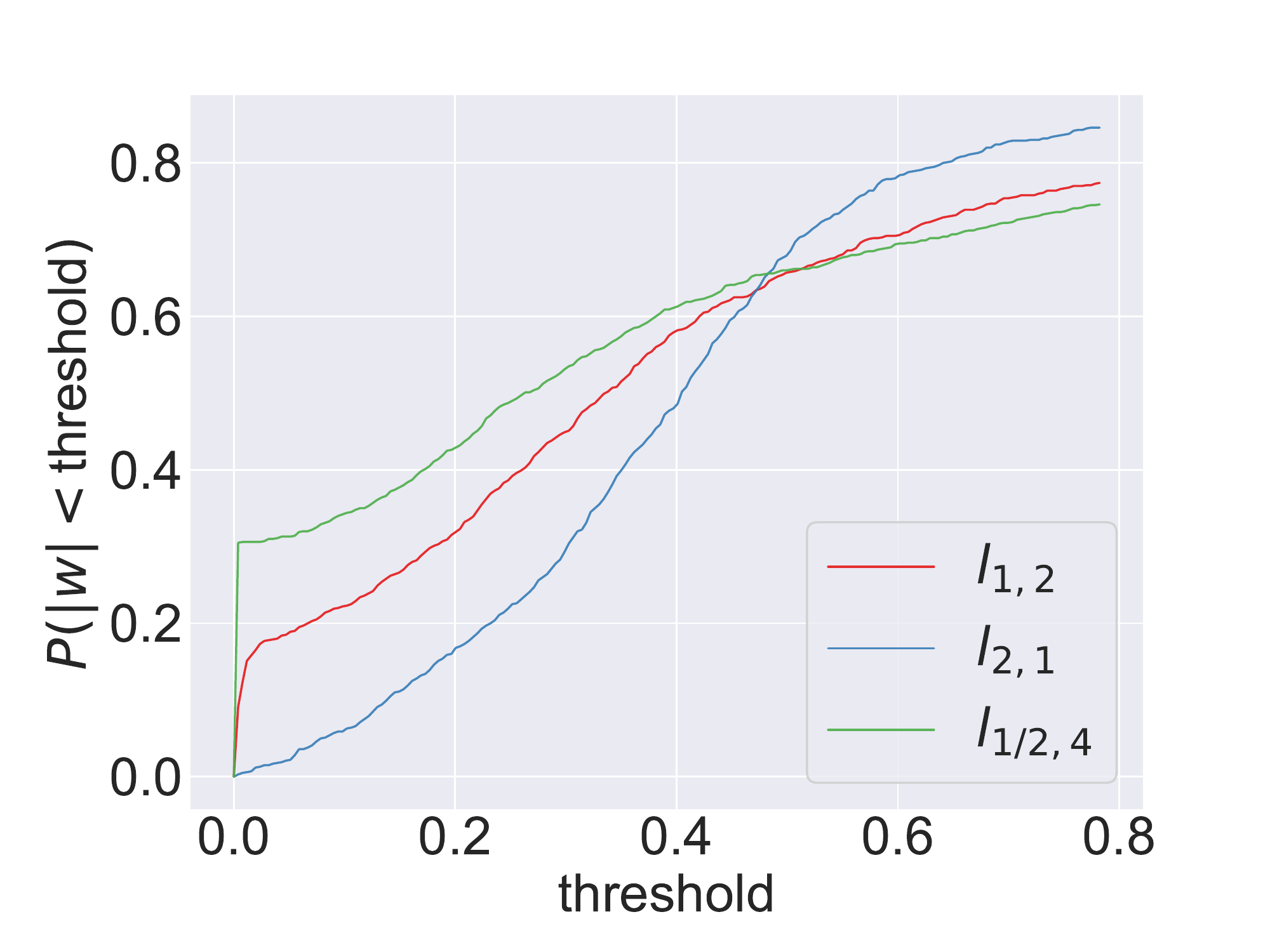}}
        \subfigure[weights of $w^{(2)}$]{\includegraphics[scale=0.32]{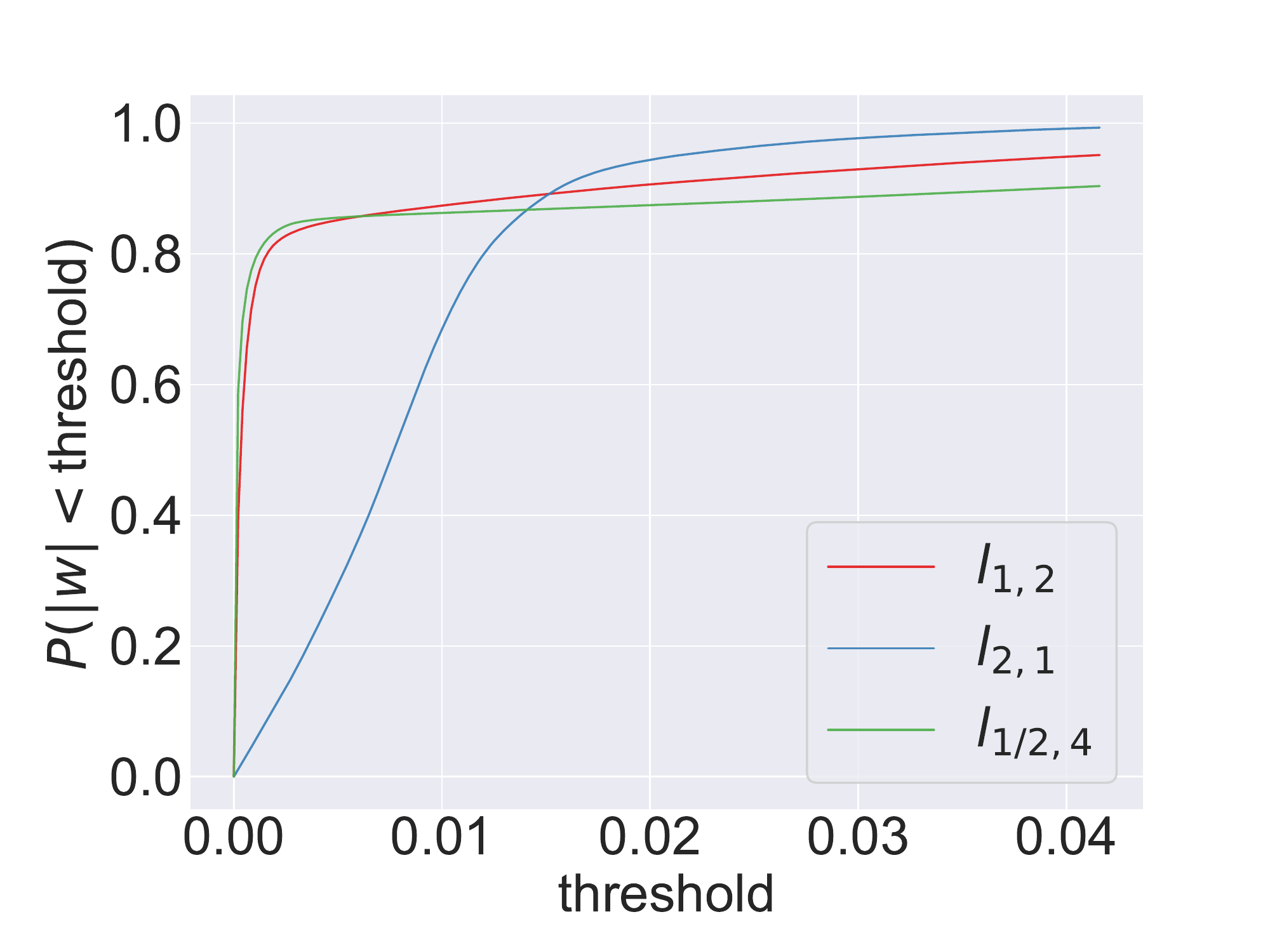}}
        \subfigure[weights of $w^{(3)}$]{\includegraphics[scale=0.32]{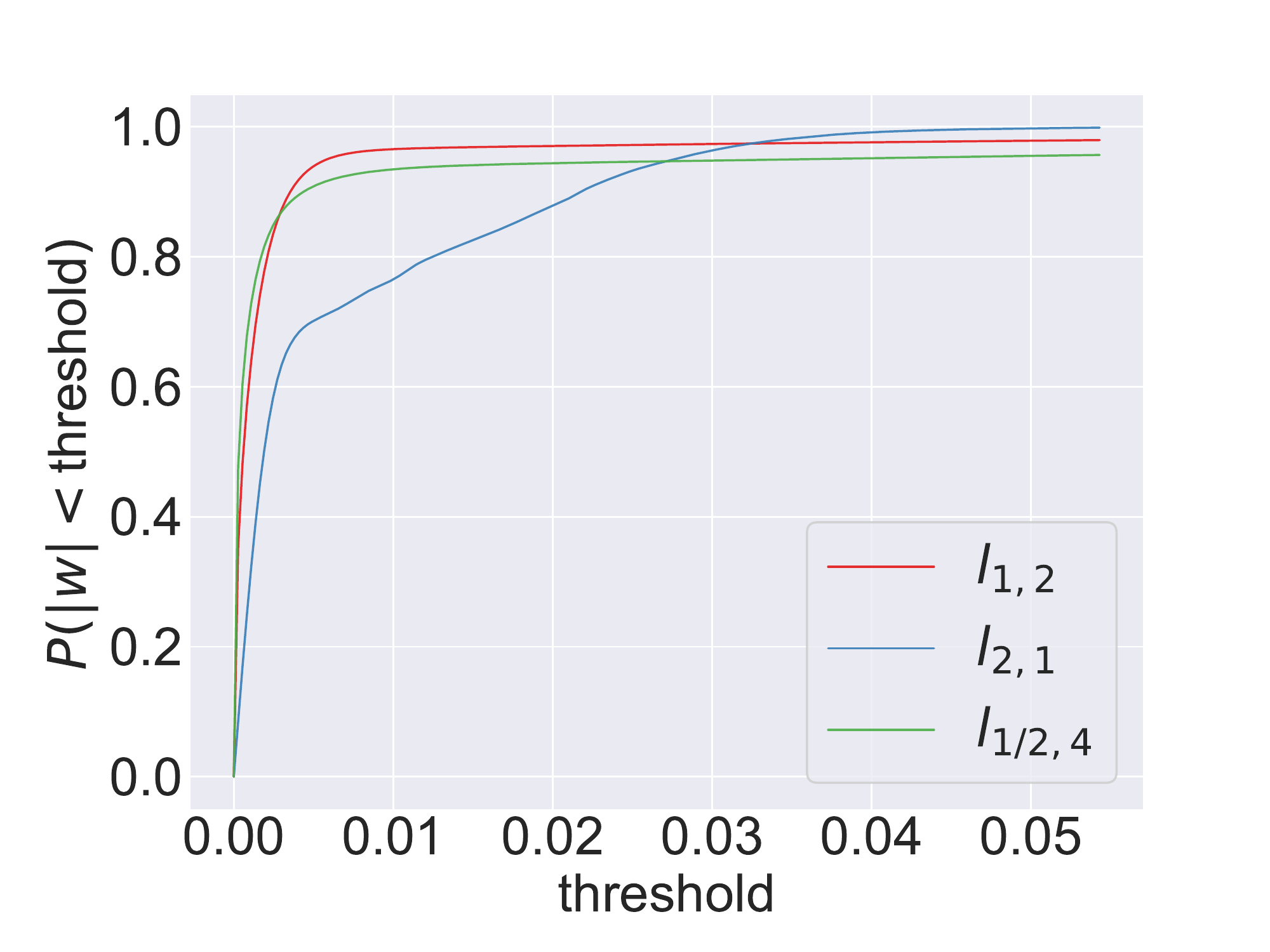}}
        \subfigure[weights of $u$]{\includegraphics[scale=0.32]{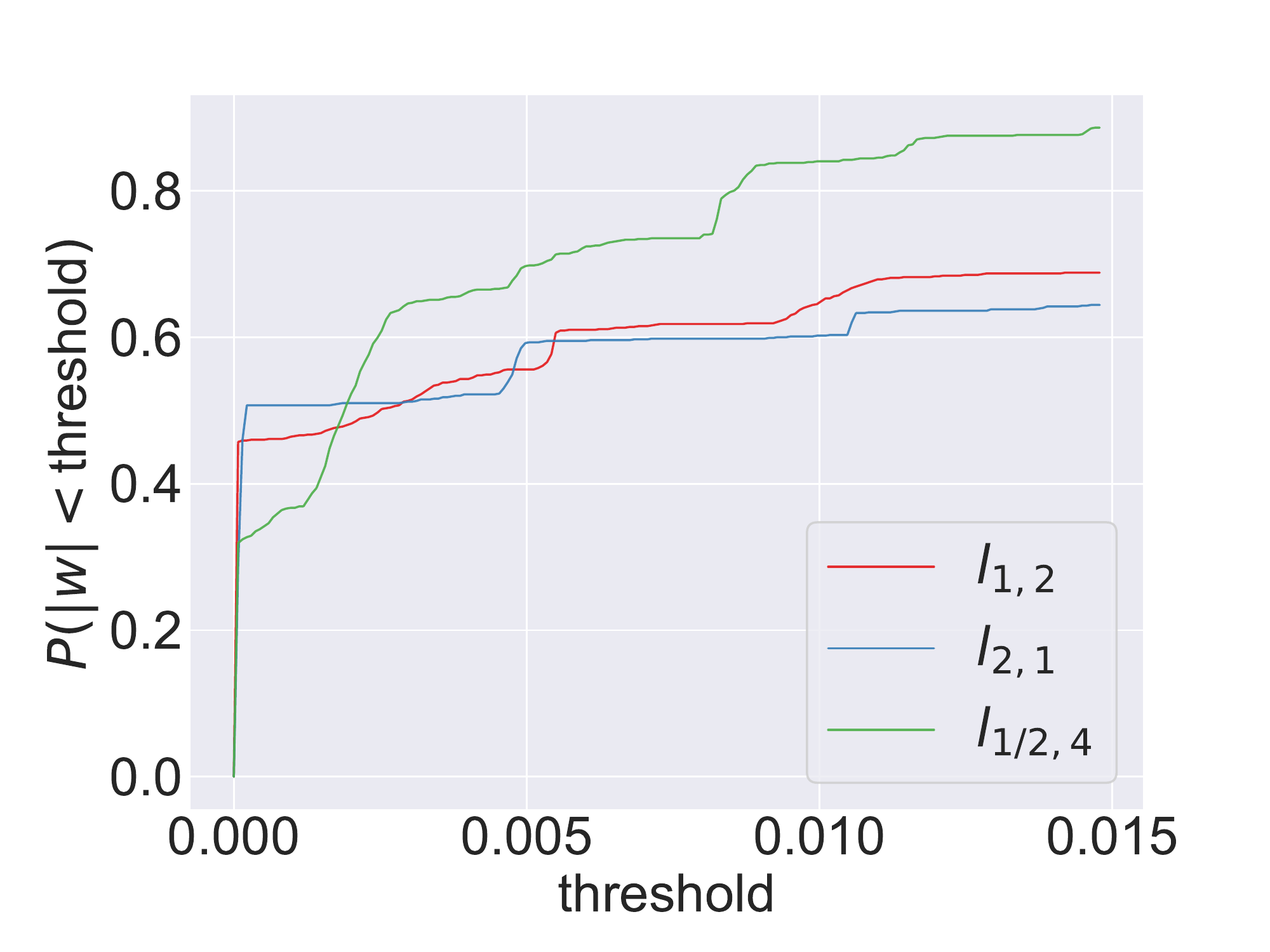}}

    \end{center}
    \vspace*{-20pt}
    \caption{The comparison of sparsity of weights in different layers between different regularizers. The $i$-th figure demonstrate the sparsity of weights in $i$-th layer for different regularizers. The point $(x, y)$ on one curve suggests that there are $100\times y\%$ weights in such layer lies in $[-x, x]$. We could see that the red curve is lower-bounded by the blue one, which suggests that our proposed regularizer will result in more sparse representations compared with traditional regularizer.}
    \label{exp:syn:fig:weight-sparsity}
\end{figure*}

We further compare the compactness of different regularizers. Here we use the notation $\ell_{a,b}$ to represent the regularizer in the form of $r_1(x)=|x|^a$ and $r_2(x)=|x|^b$. $\ell_{1,2}$ regularizer is the proposed regularizer, which is an upper-bound of the approximation variance. $\ell_{2,1}$ regularizer is the traditional $L_2$ regularizer (i.e. $\ell_2$ weight decay). From Fig \ref{exp:syn:fig:weight-sparsity}, we found that the proposed regularizer leads to sparser weights, and thus has a more compact representation. We have also tried the $\ell_{1/2,4}$ regularizer, and found that the sparsity of $\ell_{1/2,4}$ regularizer isn't significantly better than the proposed $\ell_{1,2}$ regularizer. This verifies the effectiveness of the proposed regularizer  to obtain sparse weights.

Finally, we show that the proposed discrete feature repopulation (DFR) process can reduce the training loss, especially the regularizer loss (in Fig \ref{exp:syn:fig:opt-dfr} (b)). This implies that it leads to a better feature representation. In our implementation, we first use Proximal Gradient Descent to optimize the objective \eqref{eq:reg-im} with respect to to the variables $\{\{p_j^{(\ell)}\}_{j=1}^{m^{(\ell)}}\}_{\ell=1}^L$ under the constraints that $\sum_{j=1}^{m^{(\ell)}}p_j^{(\ell)}=1$ for $\ell \in [L]$. When these repopulation weights $\{\{p_j^{(\ell)}\}_{j=1}^{m^{(\ell)}}\}_{\ell=1}^L$ are calculated, we use the algorithm described in Section \ref{sec:algo} to discretely re-sample useful features from top layer to bottom layer. Fig \ref{exp:syn:fig:feat-cp} is a comparison of learned feature functions between vanilla SGD and SGD with DFR process described above. The feature functions learned by vanilla SGD contains some useless feature functions whose variance with respect to input $x$ is near zero, while the DFR process is able to remove these bad feature functions since their importance weights are relatively low. The difference is highly visible in Layer 3.

\begin{figure}[htb!]
    \begin{center}
        \includegraphics[width=1.0\linewidth]{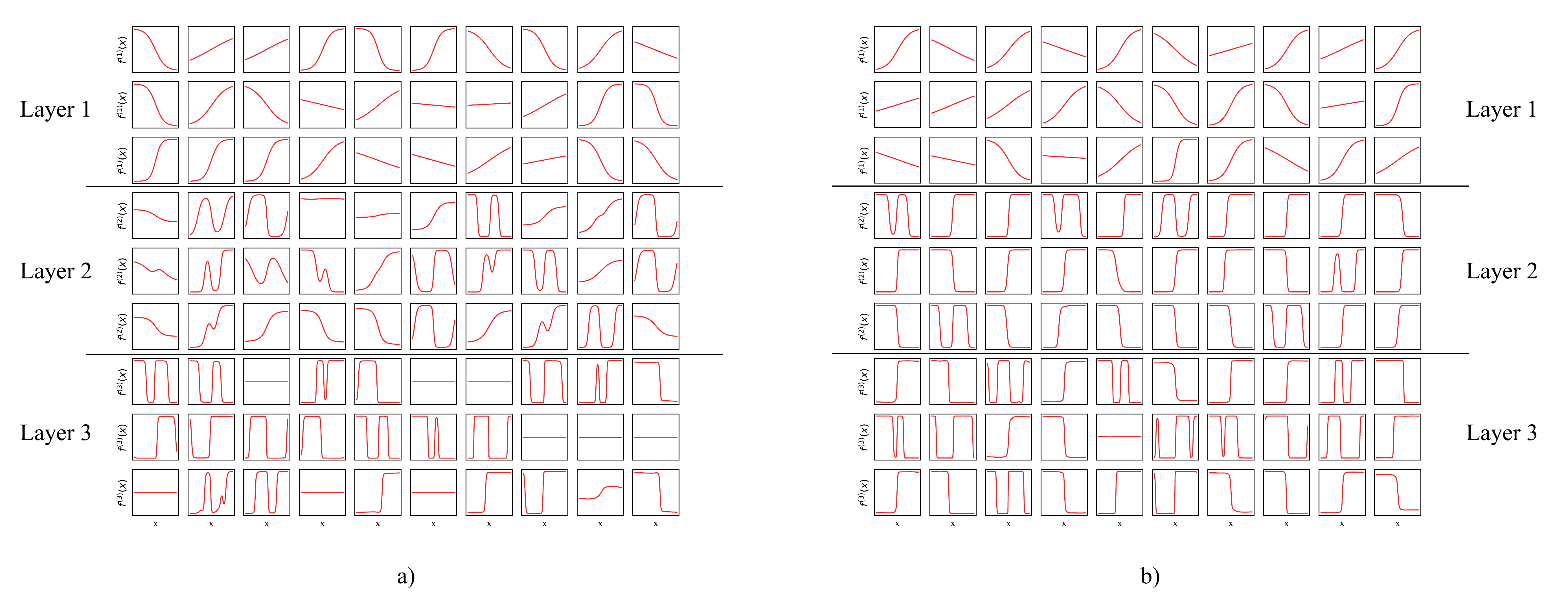}
    \end{center}
    \vspace*{-20pt}
    \caption{Representative feature functions of a) SGD without DFR b) SGD with one DFR process every 10 epochs during the first 50 epochs and every 100 epochs after epoch 50. In each subplot, we sample 30 neurons from each layer and plot the single-variable feature function $f^{(\ell)}_j(x)$ of sampled neurons.}
    \label{exp:syn:fig:feat-cp}
\end{figure}

\begin{figure*}[htb!]
    \begin{center}
        \subfigure[]{\includegraphics[scale=0.25]{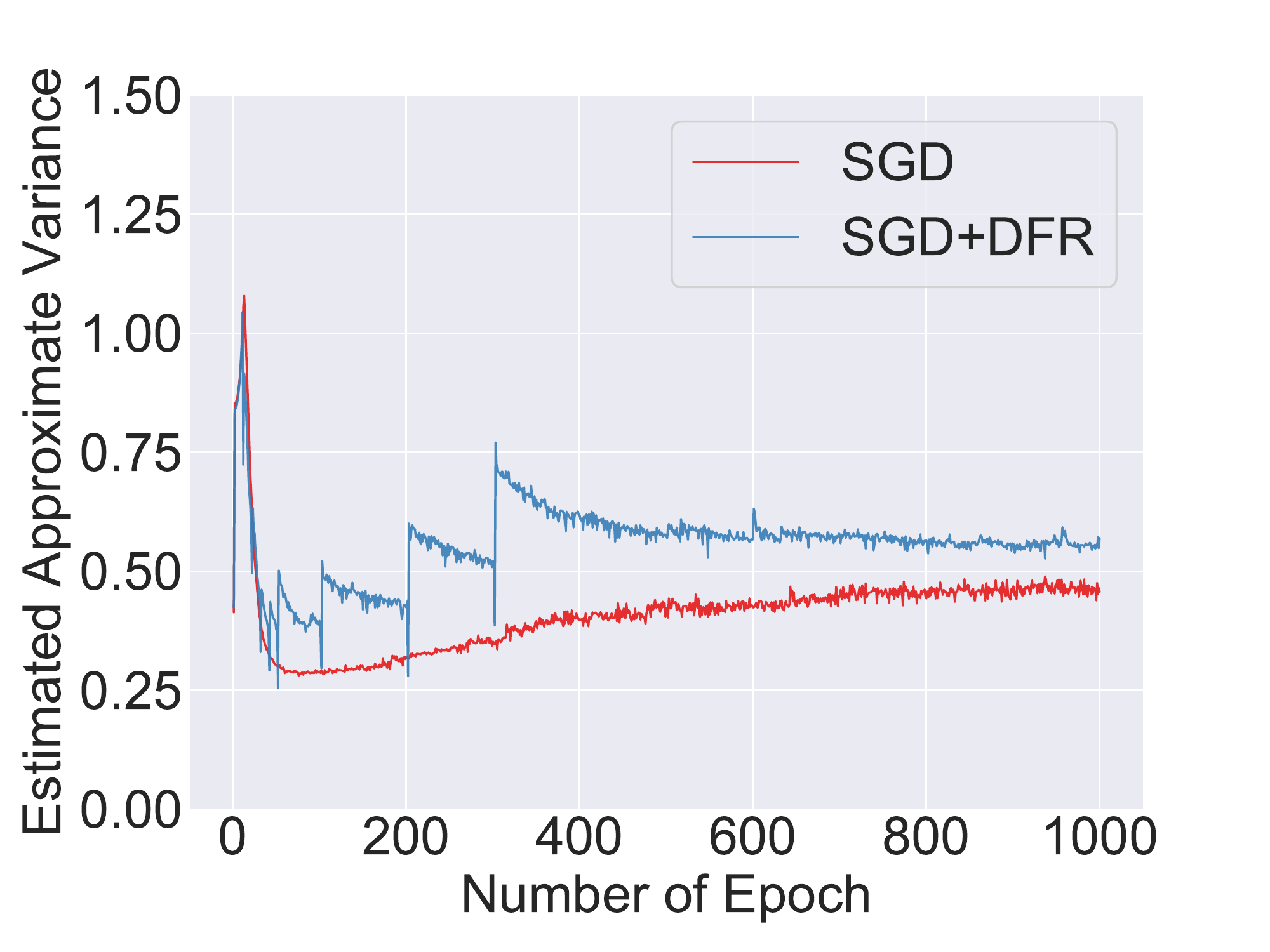}}
        \subfigure[]{\includegraphics[scale=0.25]{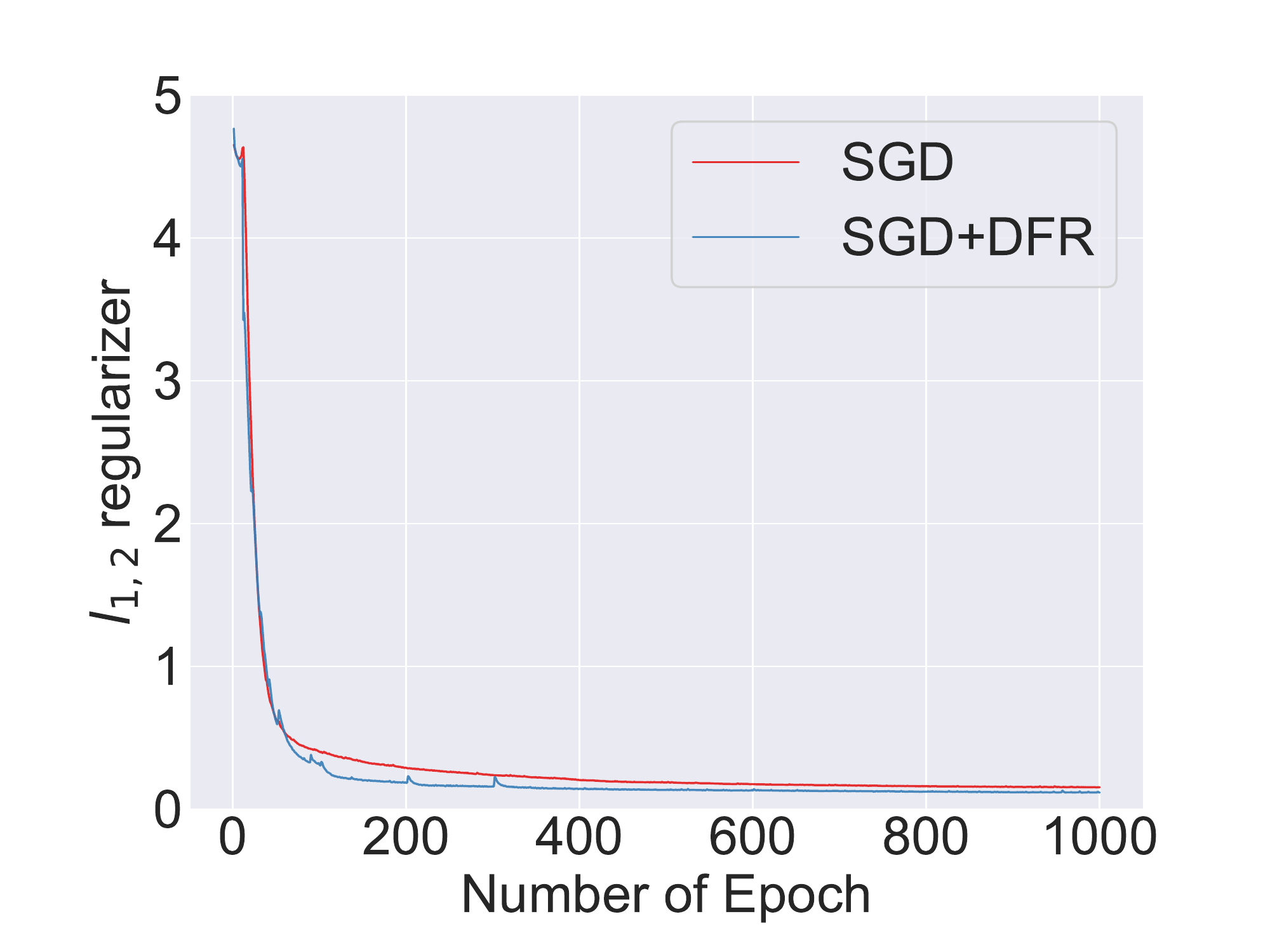}}
        \subfigure[]{\includegraphics[scale=0.25]{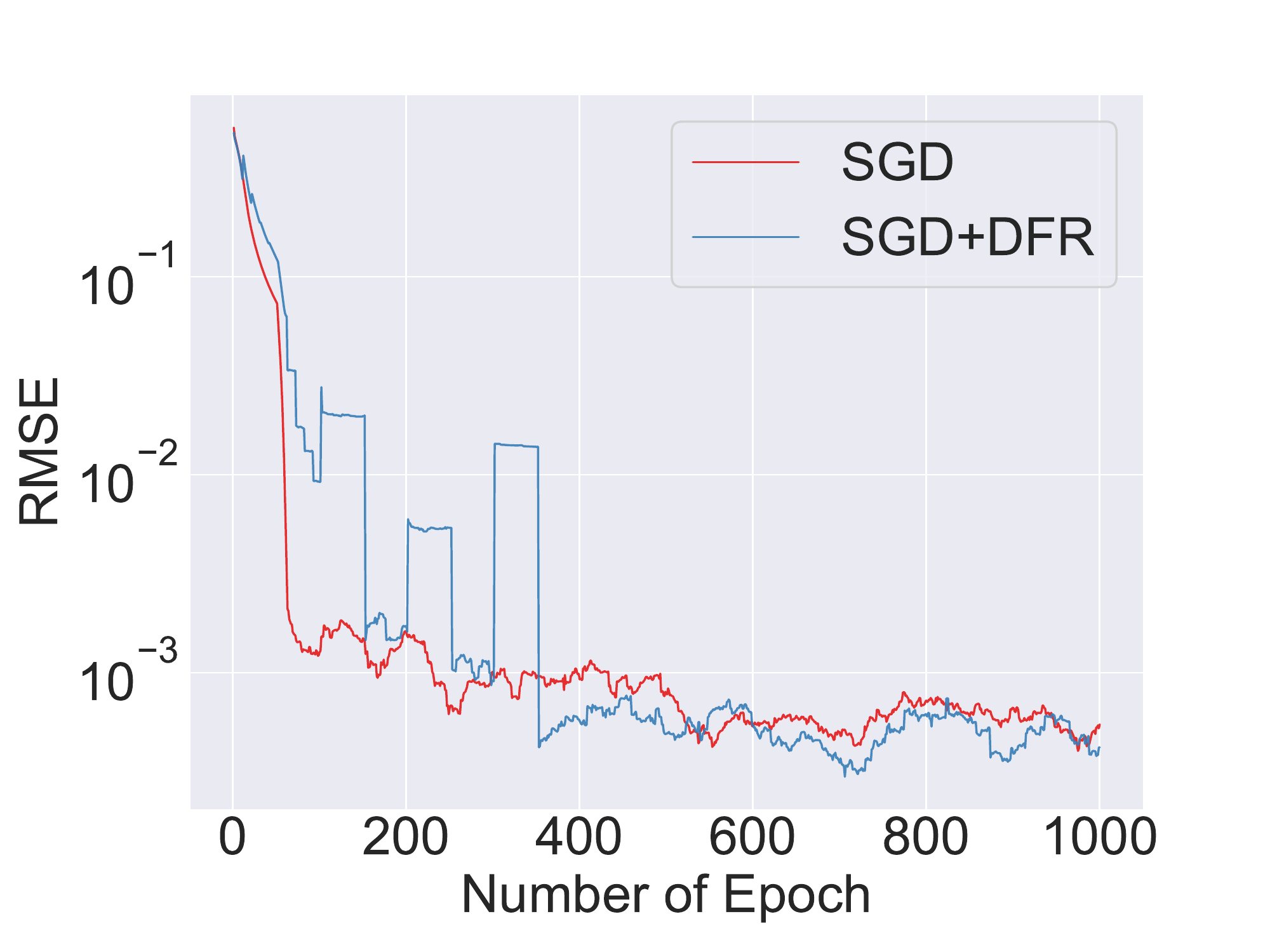}}
    \end{center}
    \vspace*{-20pt}
    \caption{The comparison of optimization process between SGD and SGD with DFR for 3-hidden-layer NN.}
    \label{exp:syn:fig:opt-dfr}
\end{figure*}

\subsection{Mini-Imagenet classification task}

We have also performed experiments on real data. Mini-Imagenet dataset is a simplified version of ILSVRC’12 dataset \citep{ILSVRC15}, which consists of 600 84$\times$84$\times$3 images sampled from 100 classes each. Here we consider the data split introduced by \citep{ravi2016optimization}, which consists of 64 classes and 38.4k images as our full dataset. We divide the dataset into train/valid/test split by 7:1:2. 

Since fully-connected NNs do not have the capacity to deal with such image data, we first train a base CNN embedding network with a four block architecture as in \citep{vinyals2016matching}. We then take the $1600$-dimensional output of the embedding layer and feed it to an $L$ layer NN for classification. The training configurations and network architectures are the same as those for the synthetic 1-D experiment, except that we tune the regularization parameters to achieve the best validation accuracy. Since the feature function of this task is hard to visualize, we only consider  the \emph{optimality condition}, \emph{shallow versus deep networks} and  \emph{compactness}. 

Similar to the results in the synthetic 1-D experiment, the sub-figures in the bottom row of Fig \ref{exp:all:fig:optimal-sol} show that the two quantities we care about are also linearly correlated in each layer, which is consistent with our theory. 
Fig \ref{exp:imagenet:fig:eav} reports how approximation variance, test RMSE, and train RMSE change during the model training procedure. We can see that the approximation variance decreases as $L$ increases, and the gap between $L=1$ and $L=2$ is very large. This demonstrates the great advantage of deeper networks. Moreover, the generalization performance also increases as $L$ increases.

\begin{figure*}[htb!]
    \begin{center}
        \subfigure[]{\includegraphics[scale=0.25]{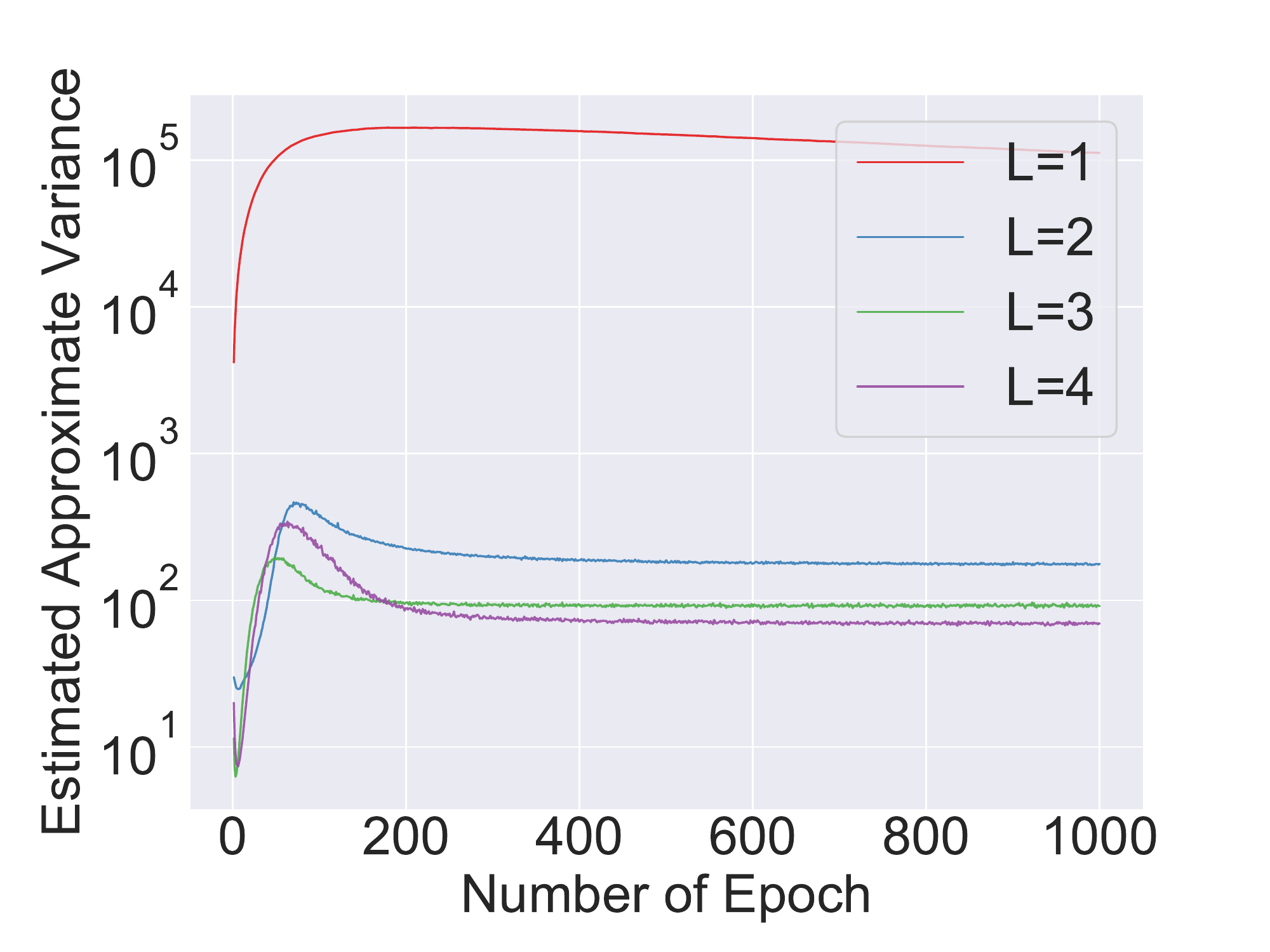}}
        \subfigure[]{\includegraphics[scale=0.25]{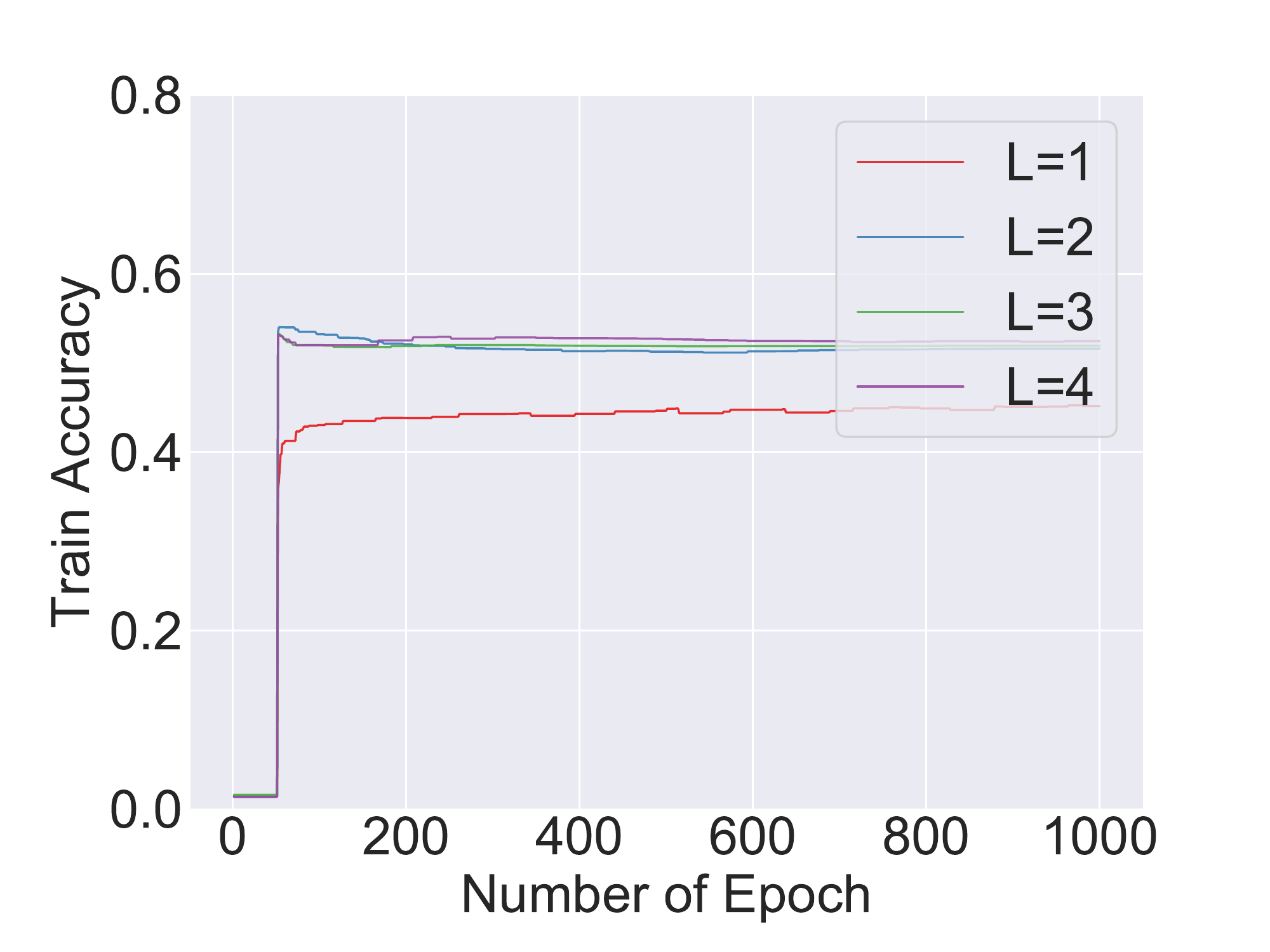}}
        \subfigure[]{\includegraphics[scale=0.25]{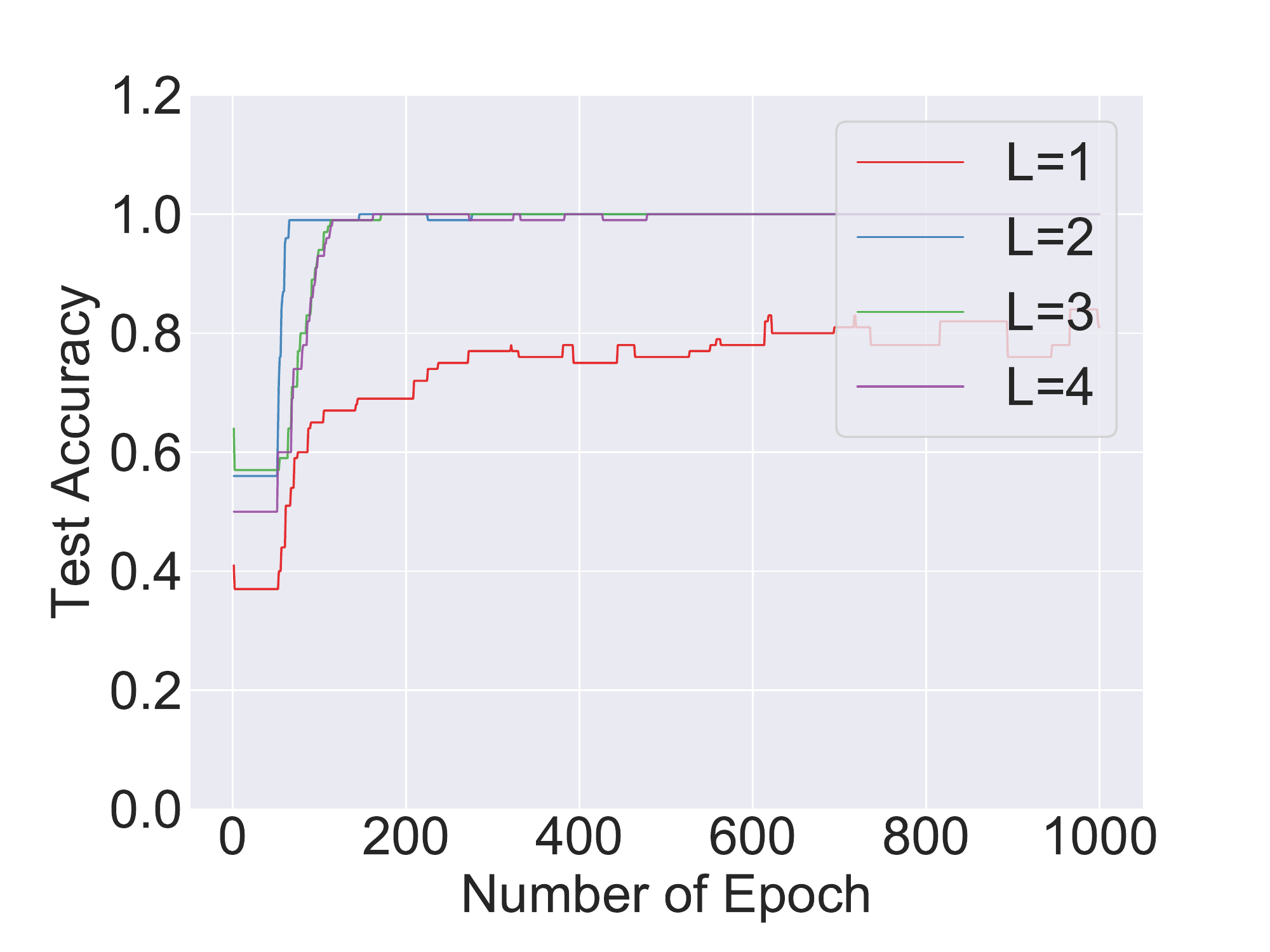}}
    \end{center}
    \vspace*{-20pt}
    \caption{Visualization of the optimization process when training multilayer NNs with different $L$.}
    \label{exp:imagenet:fig:eav}
\end{figure*}

Fig \ref{exp:imagenet:fig:weight-sparsity} shows that the proposed $\ell_{1,2}$ regularizer leads to more compact weights for each layer than the traditional $\ell_{2,1}$. .

\begin{figure*}[htb!]
    \begin{center}
        \subfigure[weights of $w^{(1)}$]{\includegraphics[scale=0.32]{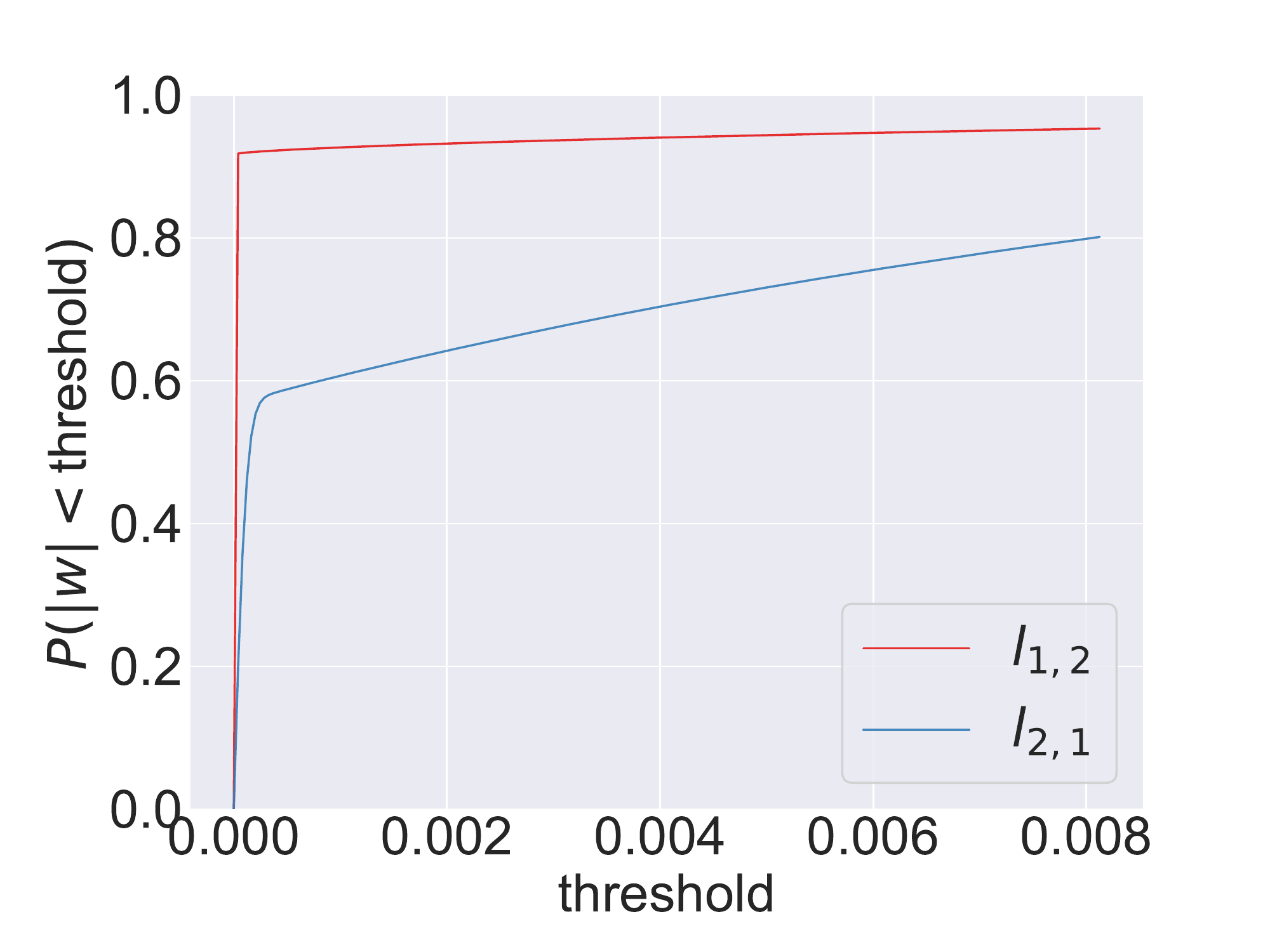}}
        \subfigure[weights of $w^{(2)}$]{\includegraphics[scale=0.32]{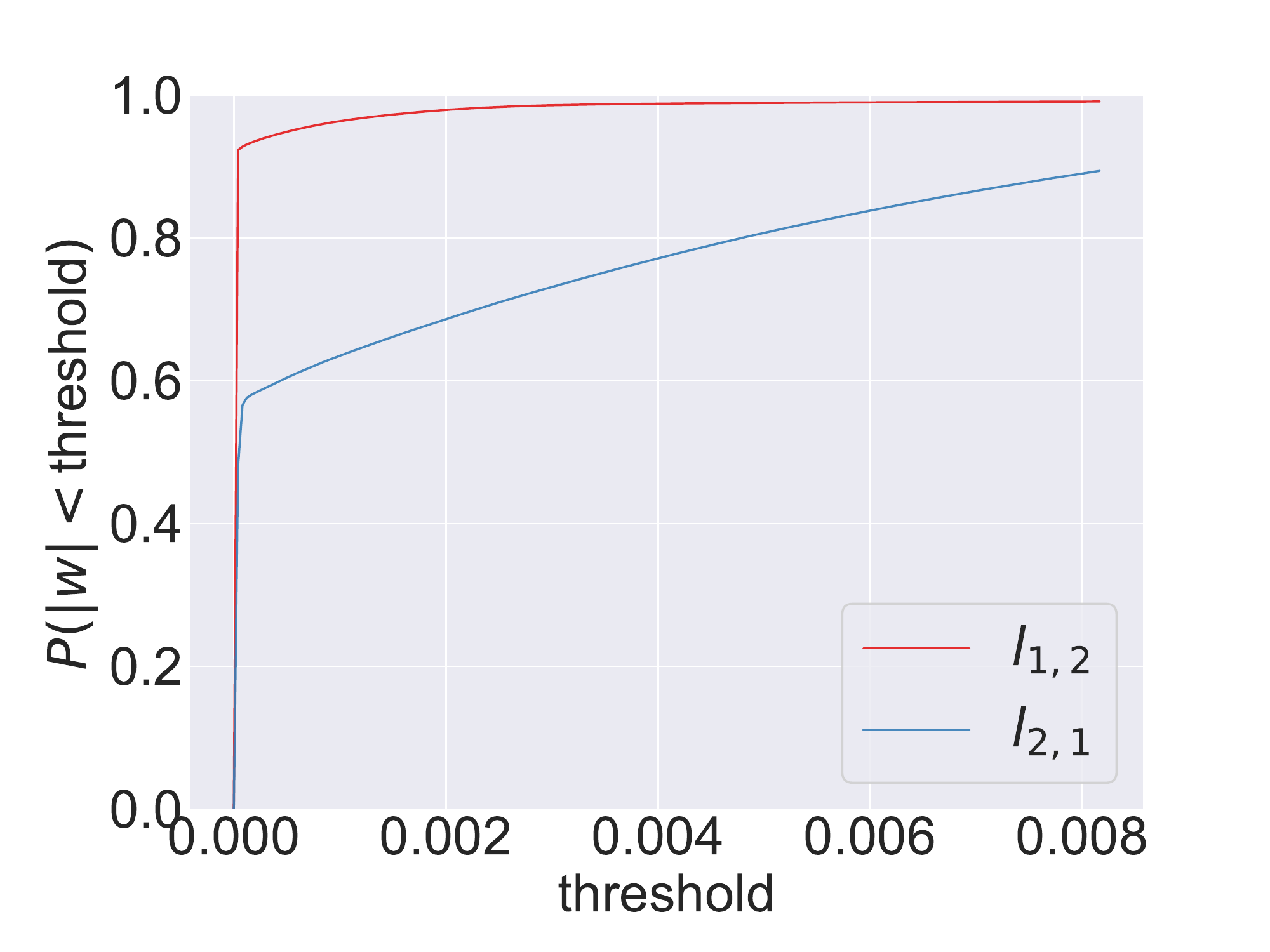}}
        \subfigure[weights of $w^{(3)}$]{\includegraphics[scale=0.32]{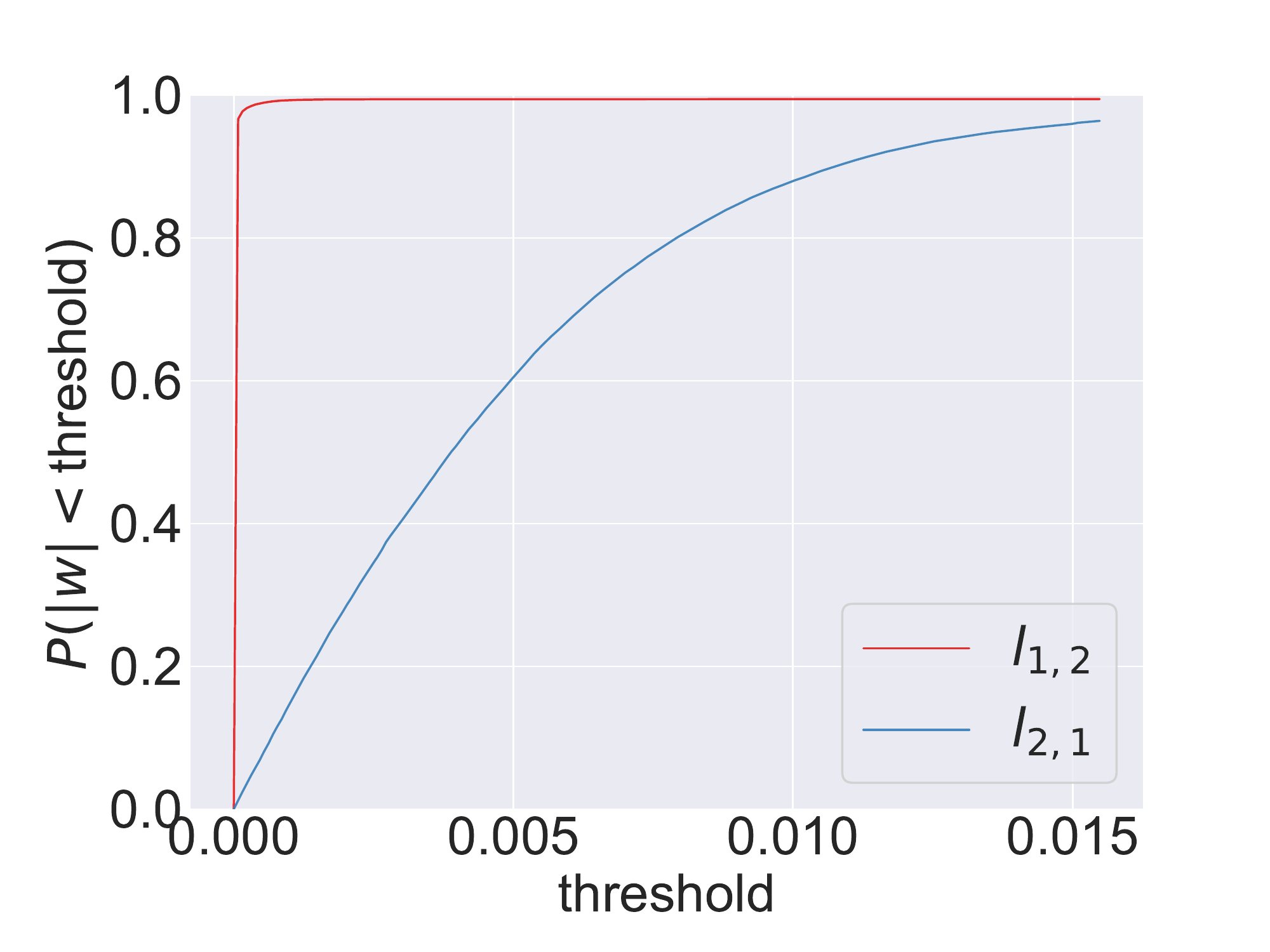}}
        \subfigure[weights of $u$]{\includegraphics[scale=0.32]{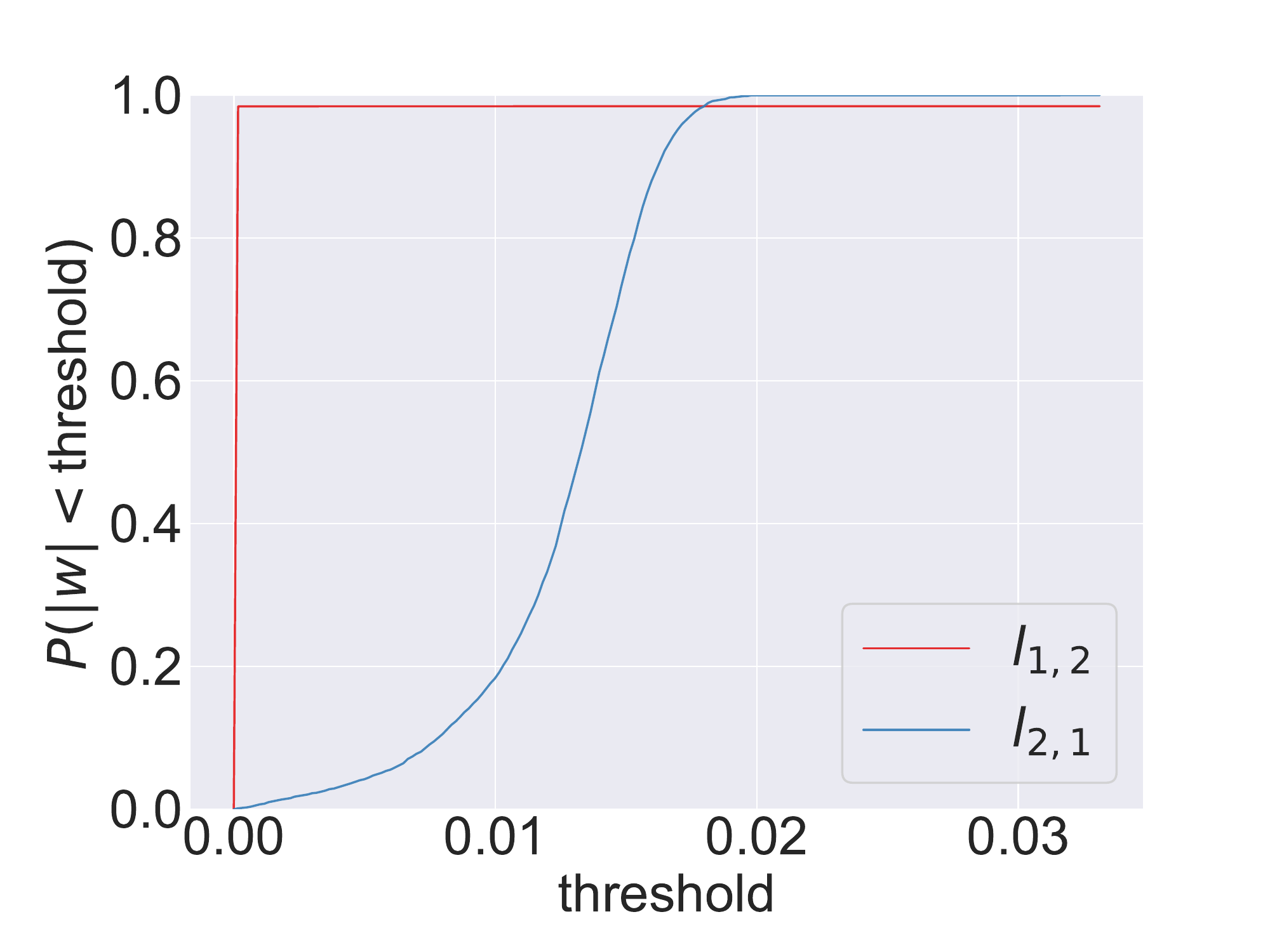}}

    \end{center}
    \vspace*{-20pt}
    \caption{Comparison of weight sparsity at different layers for different regularizers on the Mini-Imagenet classification task. The $i$-th figure demonstrates the weight sparsity at layer $i$ for different regularizers. A point $(x, y)$ indicates that there are $100\times y\%$ weight values lying in $[-x, x]$. We can see that the red curve dominates the blue one, which suggests that the proposed regularizer leads to sparser solutions than those of the traditional $\ell_2$ regularizer.}
    \label{exp:imagenet:fig:weight-sparsity}
\end{figure*}

\section{Conclusion}
\label{sec:conclusion}
 This paper  introduced the NFR technique to analyze  over-parameterized DNNs and 
showed that it is possible to reformulate overparameterized DNNs as convex systems. Moreover, when fully trained, DNNs learn effective feature representations  suitable for the underlying learning task via regularization. Our analysis is consistent with empirical observations.

Similar to the analysis of two-level NN in \citep{fang2019over}, this newly introduced NFR method paves the way for establishing global convergence results of standard optimization algorithms such as (noisy) gradient descent for overparameterized DNNs. We will leave such study as a future work.

\paragraph{Acknowledgement}
The authors would like to  thank   Jason Lee, Xiang Wang, and Pengkun Yang   for  very helpful discussions.


\bibliographystyle{apalike2}
\bibliography{overbib2}

\pagebreak\appendix
\pb
\section{Preliminary}
This section provides some useful known inequalities that will be  later used  in our proofs.
\subsection{Jensen's Inequality}
In our proof, we will frequently use the Jensen's inequality, which relates the value of a convex function of an integral to the integral of the convex function.  In probability theory, it states that   if  $\phi$ is a convex function, for a random variable $x$, we have
$$\phi \left(\EE [x]\right)\leq  \EE [\phi(x)].  $$
We are particularly interested in the case when $\phi(x)=| x|^{p}$ where $p\geq1$. We can obtain 
$$ |\EE [x]| \leq (\EE  [|x|^p])^{1/p}, \quad p\geq 1.  $$
In the finite form, Jensen's inequality takes the form of:
\begin{align}\label{pred}
   \phi\left( \frac{1}{n}\sum_{i=1}^{n} x_i  \right) \leq \frac{1}{n}\sum_{i=1}^{n}\phi(x_i).    
\end{align}
If we  further   assume  $x_i$ with $i\in[n]$ are independent random variables which follow from a same underlying  distribution, by taking expectation on \eqref{pred}, we have
\begin{align}
\EE\left[\phi\left( \frac{1}{n}\sum_{i=1}^{n} x_i  \right)\right] \leq \frac{1}{n}\sum_{i=1}^{n}\EE[\phi(x_i)]=\EE [\phi(x_1)].
\end{align}

\subsection{Rosenthal Inequality}
We will also use the Rosenthal Inequality to build relations between moments of a collection of independent random variables.   It is stated in the following:
\begin{lemma}[Rosenthal Inequality \citep{ibragimov1999analogues}] \label{rosenthal}
	Let $\xi_i$ with $i\in [n]$ be independent with $\E[\xi_i]=0$ and $\E[|\xi_i |^t]< \infty$ for some $t\geq 2$.  Then we have
	$$  \E\left| \sum_{i=1}^n\xi_i \right|^t\leq C^1_t\left(\sum_{i=1}^n \E |\xi_i|^t \right) +C^2_t\left(\sum_{i=1}^n \E [\xi_i^2] \right)^{t/2}, $$
\end{lemma}
where  $C^1_t$ can be token as $(ct)^t$, $C^2_t = (c\sqrt{t})^t$, and $c$ is a constant.

\section{Proof of Theorem \ref{theorem:convergence-discrete-NN}: Consistency of Discretization}
\label{apx:discrete-consistency}
In this section, we prove Theorem \ref{theorem:convergence-discrete-NN} which shows the discrete DNN converges to the corresponding continuous one in $L^1$.
Before giving the detailed proof, we  give the following definitions.

\begin{itemize}
    \item[(1)] We denote $\cI^{(\ell)}$ to be the set consisting of all ordered sequences in the form of $\{k_1, k_2, \dots, k_{\ell} \}$ with   $k_{j}\in [m^{(j)}]$ and $j\in [\ell]$. That is
\begin{align}
    \cI^{(\ell)} := [m^{(1)}] \otimes [m^{(2)}] \otimes \cdots \otimes [m^{(\ell)}],\nonumber
\end{align}
where $\otimes$ means the  Cartesian product.

\item[(2)] We denote $\hat{\mathcal{Z}}^{(\ell)} := \{z^{(\ell)}_{j}:~ j\in[m^{(\ell)}]  \}$,  and let 
\begin{align}
    \{z\}_{\cI^{(\ell)}} :=  \hat{\mathcal{Z}}^{(1)}\cup \cdots \cup \hat{\mathcal{Z}}^{(\ell)}.\nonumber
\end{align}

\item[(3)] $\EE[f]$ means taking full expectation on all the random variables. $\EE_{\hcZ^{(\ell)}}[f]$ (or $\EE_{\{z\}_{\cI^{(\ell)}}}[f]$) stands for taking the (conditional) expectation  on $\hcZ^{(\ell)}$ (or $\{z\}_{\cI^{(\ell)}}$).

\item[(4)] We define 
$$\Delta^{(\ell)}\big(z^{(\ell)}\big):=\left|\frac{1}{m^{(\ell-1)}} \sum_{j=1}^{m^{(\ell-1)}} w(z^{(\ell)}, z^{(\ell-1)}_j) f^{(\ell-1)}(\rho, z^{(\ell-1)}_j;x)-g^{(\ell)}(\rho, z^{(\ell)};x)\right|, \quad \ell\in[L+1],\nonumber $$
and  
\begin{align}
    \Delta^{(L+1)}\big(z^{(L+1)}\big):=\left\|\frac{1}{m^{(L)}} \sum_{j=1}^{m^{(L)}} u( z^{(L)}_j) f^{(L)}(\rho, z^{(L)}_j;x)-f(\rho,u;x)\right\|,\nonumber
\end{align}
where we denote  $z^{(L+1)}=1$. Clearly, we have  $\Delta^{(1)}\big(z^{(1)}\big) = 0$.

\item[(5)] We define 
 $$
A(k,L+1) := 
\begin{cases}
1,&   k=L+1,\\
\left\|u(z^{(L)}) \right\|,&   k= L, \\
\left\|u(z^{(L)}) \right\| \prod_{j=k+1}^{L}|w(z^{(j)}, z^{(j-1)}) |,&   k\in [L-1], \\
\end{cases}
$$
and for $\ell\in \{2, \dots, L \}$, we define 
 $$
A(k,\ell) := 
\begin{cases}
1,&   k=\ell,\\
\prod_{j=k+1}^{\ell}|w(z^{(j)}, z^{(j-1)}) |,&   k< \ell. \\
\end{cases}
$$
\item[(6)] We define  
$$\bar{f}^{(\ell)}(z^{(\ell)};x) = h^{(\ell)}\Big(\frac{1}{m^{(\ell-1)}}\sum_{j=1}^{m^{(\ell-1)}} w(z^{(\ell)},z^{(\ell-1)}_j) \hat{f}^{(\ell-1)}_j(x)\Big), \quad \ell\in [L].\nonumber $$
 It holds that $\bar{f}^{(\ell)}(z^{(\ell)}_j;x)=\hat{f}^{(\ell)}_j(x)$ for all $\ell\in [L]$ and $j\in [m^{(\ell)}]$.  We note that $\Delta^{(\ell)}(z^{(\ell)})$ depends on  random variables in  $\hcZ^{(\ell-1)}$ and $\bar{f}^{(\ell)}(z^{(\ell)};x)$  depends on  the random variables in $\hcZ^{(1)}, \dots, \hcZ^{(\ell-1)}$.
\end{itemize}
Now we begin our proof.\\\\
\textbf{Step 1.} We prove for all $\ell\in \{2, \dots, L \}$,
\begin{align}
\!\!\!\!\EE_{z^{(\ell)}}\EE_{\{z\}_{\cI^{(\ell-1)}}}\left| \bar{f}^{(\ell)}(z^{(\ell)};x) - f^{(\ell)}(\rho, z^{(\ell)};x)\right|\!\leq\! \sum_{j=2}^{\ell} (c_0)^{\ell-j+1}\E_{z^{(j)}, \dots, z^{(\ell)}} \E_{\hcZ^{(j-1)}}  A(j, \ell)  \Delta^{(j)}(z^{(j)}).\label{theoe}
\end{align}
In the inequalities above $\E_{z^{(j)}, \dots, z^{(\ell-1)}} \E_{\hcZ^{(j-1)}}$ means $\E_{\hcZ^{(j-1)}}$ when $j=\ell$, i.e., we only take expectation on $\hcZ^{(j-1)}$.
Moreover, we have
\begin{align}
    &\EE_{\{z\}_{\cI^{(L)}}}  \left\|\hf(x) - f(\rho, u, x)  \right\|\leq \sum_{j=2}^{\ell} (c_0)^{\ell-j+1}\E_{z^{(j)}, \dots, z^{(\ell-1)}} \E_{\hcZ^{(j-1)}}  A(j, \ell)  \Delta^{(j)}(z^{(j)}).
    \label{theo:end1}
\end{align}
\begin{proof}
For all $\ell\in [L] $ and $k\in[\ell]$, we have 
\begin{align}
    &\EE_{\{z\}_{\cI^{(k-1)}}}A(k,\ell)\left| \bar{f}^{(k)}(z^{(k)};x) - f^{(k)}(\rho, z^{(k)};x) \right|\label{the1:a1}\\
    =&\EE_{\{z\}_{\cI^{(k-1)}}}A(k,\ell)\Bigg| h^{(k)} \Big(\frac{1}{m^{(k-1)}}\sum_{j=1}^{m^{(k-1)}} w(z^{(k)},z^{(k-1)}_j) \hat{f}^{(k-1)}_j(x)\Big) \notag\\
    &\quad\quad\quad\quad\quad\quad\quad\quad\quad\quad\quad\quad\quad\quad\quad\quad\quad\quad~~~ - h^{(k)} \Big( \E_{z^{(k-1)}}\left[ w(z^{(k)},z^{(k-1)})f^{(k-1)}(\rho,z^{(k-1)};x)\right] \Big) \Bigg|\notag\\
    \overset{a}\leq& c_0\EE_{\{z\}_{\cI^{(k-1)}}}A(k,\ell)\Bigg|  \underbrace{\frac{1}{m^{(k-1)}}\sum_{j=1}^{m^{(k-1)}} \left[w(z^{(k)},z^{(k-1)}_j) \hat{f}^{(k-1)}_j(x)\right]}_{=a_1}\notag\\
    &\quad\quad\quad\quad\quad\quad\quad\quad\quad\quad\quad\quad\quad\quad\quad\quad\quad\quad\quad\quad\quad~~ -\underbrace{\E_{z^{(k-1)}}\left[ w(z^{(k)},z^{(k-1)})f^{(k-1)}(\rho,z^{(k-1)};x)\right]}_{=g^k(\rho, z^{(k)};x)=b_1}\Bigg|\notag\\
    \overset{b}{\leq}& c_0\EE_{\{z\}_{\cI^{(k-1)}}}A(k,\ell)\Bigg|  \underbrace{\frac{1}{m^{(k-1)}}\sum_{j=1}^{m^{(k-1)}} w(z^{(k)},z^{(k-1)}_j) f^{(k-1)}(\rho, z^{(k-1)}_j;x)}_{=c_1}- g^{(k)}(\rho,z^{(k)};x)\Bigg|\notag\\
    &\quad\quad\quad\quad~~~ +c_0\EE_{\{z\}_{\cI^{(k-1)}}}A(k,\ell)\Bigg|  \underline{\frac{1}{m^{(k-1)}}\sum_{j=1}^{m^{(k-1)}}} w(z^{(k)},z^{(k-1)}_j) \left[\underline{\hat{f}^{(k-1)}_j(x)}- f^{(k-1)}(\rho, z^{(k-1)}_j;x)\right] \Bigg|\notag\\
    \leq&    c_0\EE_{\{z\}_{\cI^{(k-1)}}}A(k,\ell)\Bigg|  \frac{1}{m^{(k-1)}}\sum_{j=1}^{m^{(k-1)}} w(z^{(k)},z^{(k-1)}_j) f^{(k-1)}(\rho, z^{(k-1)}_j;x)- g^{(k)}(\rho,z^{(k)};x)\Bigg|\notag\\
    &\quad\quad~ +c_0\EE_{\{z\}_{\cI^{(k-1)}}}A(k,\ell)\underline{\frac{1}{m^{(k-1)}}\sum_{j=1}^{m^{(k-1)}}}\Bigg|   w(z^{(k)},z^{(k-1)}_j) \left[\underline{\bar{f}^{(k-1)}(z^{(k-1)}_j;x)}- f^{(k-1)}(\rho, z^{(k-1)}_j;x)\right] \Bigg|\notag\\
    \overset{c}{=}&   c_0 \EE_{\{z\}_{\cI^{(k-1)}}} A(k,\ell) \Delta^{(k)}(z^{(k)})\notag\\
    &\quad\quad\quad\quad\quad~ +c_0\EE_{z^{(k-1)}}\EE_{\{z\}_{\cI^{(k-2)}}}A(k,\ell)\Bigg|   w(z^{(k)},z^{(k-1)}) \left[\bar{f}^{(k-1)}(z^{(k-1)};x)- f^{(k-1)}(\rho, z^{(k-1)};x)\right] \Bigg|\notag\\
    \overset{d}=&   c_0\EE_{\hcZ^{(k-1)}}A(k,\ell) \Delta^{(k)}(z^{(k)}) +c_0\EE_{z^{(k-1)}}\EE_{\{z\}_{\cI^{(k-2)}}}A(k-1,\ell)\Bigg|  \bar{f}^{(k-1)}(z^{(k-1)};x)- f^{(k-1)}(\rho, z^{(k-1)};x) \Bigg|, \notag
\end{align}
where $\overset{a}\leq$ we use Assumption \ref{ass:2} and obtain the result by mean value theorem,  $\overset{b}\leq$ uses triangle inequality, i.e., $|a_1-b_1|\leq |c_1-b_1 |+|a_1-c_1|$, in $\overset{c}=$, we take conditional expectation on $\hat{\mathcal{Z}}^{(k-1)}$ for the second term, and in $\overset{d}=$, we use the fact that $\Delta^{(k)}(z^{(k)})$ only depends on random variables  $\hcZ^{(k-1)}$. Thus for all $\ell\in [L]$, we have
\begin{align}
      &\EE_{\{z\}_{\cI^{(\ell-1)}}}\left| \bar{f}^{(\ell)}(z^{(\ell)};x) - f^{(\ell)}(\rho, z^{(\ell)};x)\right|\label{theo1:a2}\\
      =&\EE_{\{z\}_{\cI^{(\ell-1)}}} A(\ell, \ell)\left| \bar{f}^{(\ell)}(z^{(\ell)};x) - f^{(\ell)}(\rho, z^{(\ell)};x)\right|\notag\\
      \overset{\eqref{the1:a1}}{\leq}&c_0\EE_{\hcZ^{(\ell-1)}}A(\ell,\ell) \Delta^{(\ell)}(z^{(\ell)}) +c_0\EE_{z^{(\ell-1)}}\underbrace{\EE_{\{z\}_{\cI^{(\ell-2)}}}A(\ell-1,\ell)\Bigg|  \bar{f}^{(\ell-1)}(z^{(\ell)};x)- f^{(\ell-1)}(\rho, z^{(\ell-1)};x) \Bigg|}_{\eqref{the1:a1}}\notag\\
          \leq & \ldots\nonumber\\
     &\vdots \nonumber\\
     \leq& \sum_{j=1}^{\ell} (c_0)^{\ell-j+1}\E_{z^{(j)}, \dots, z^{(\ell-1)}} \E_{\hcZ^{(j-1)}} A(j,\ell) \Delta^{(j)}(z^{(j)}) + (c_0)^{\ell}\E_{z^{1}, \dots, z^{\ell-1}} \underbrace{\E_{\hcZ^{(0)}}\Bigg|  \bar{f}^{(1)}(z^{(1)};x)- f^{(1)}(\rho, z^{(1)};x) \Bigg|}_{0}\notag\\
     \overset{a}=&\sum_{j=2}^{\ell} (c_0)^{\ell-j+1}\E_{z^{(j)}, \dots, z^{(\ell-1)}} \E_{\hcZ^{(j-1)}}  A(j, \ell)  \Delta^{(j)}(z^{(j)}),\notag
\end{align}
where in $\overset{a}=$ we use $\Delta^{(1)}\big(z^{(1)}\big) = 0$.

We consider the top layer. In fact, using the same technique in \eqref{the1:a1},  we have  

\begin{align}
    &\EE_{\{z\}_{\cI^{(L)}}}  \left\|\hf(x) - f(\rho, u, x)  \right\|\\
    =&\EE_{\{z\}_{\cI^{(L)}}}\left\| \frac{1}{m^{(L)}}\sum_{j=1}^{m^{(L)}} u^{(L)}_j\bar{f}^{(L)}(z^{(L)}_j; x) - \E_{z^{(L)}}\left[ u^{(L)} f^{(L)}(\rho, z^{(L)};x) \right]  \right\|\notag\\
    \leq&\EE_{\{z\}_{\cI^{(L)}}}\left\| \frac{1}{m^{(L)}}\sum_{j=1}^{m^{(L)}} u^{(L)}_j\left[\bar{f}^{(L)}(z^{(L)}_j; x) -f^{(L)}(\rho, z^{(L)}_j;x) \right] \right \|\notag\\
    &\quad\quad\quad\quad\quad\quad\quad\quad\quad\quad\quad~~+ \EE_{\{z\}_{\cI^{(L)}}}\left\| \frac{1}{m^{(L)}}\sum_{j=1}^{m^{(L)}} u^{(L)}_jf^{(L)}(z^{(L)}_j; x) - \E_{z^{(L)}}\left[ u^{(L)} f^{(L)}(\rho, z^{(L)};x) \right] \right \|\notag\\
    \overset{a}{\leq}& \EE_{z^{(L)}}\underbrace{\EE_{\{z\}_{\cI^{(L-1)}}} \left\| u^{(L)}\right\|\left|  \bar{f}^{(L)}(z^{(L)}; x) -f^{(L)}(\rho, z^{(L)};x) \right |}_{\eqref{theo1:a2} \text{ with } \ell=L}\notag\\
    &\quad\quad\quad\quad\quad\quad\quad\quad\quad\quad\quad~~+ \EE_{\{z\}_{\cI^{(L)}}} \left\| \frac{1}{m^{(L)}}\sum_{j=1}^{m^{(L)}} u^{(L)}_jf^{(L)}(z^{(L)}_j; x) - \E_{z^{(L)}}\left[ u^{(L)} f^{(L)}(\rho, z^{(L)};x) \right] \right \|\notag\\
    \leq&\sum_{j=2}^{L+1} (c_0)^{L-j+1}\E_{z^{(j)}, \dots, z^{(L)}}\E_{\hcZ^{(j-1)}} A(j, L+1)  \Delta^{(j)}(z^{(j)}),\notag
\end{align}
 where  $\overset{a}\leq$ is obtained by 
\begin{align}
    &\EE_{\{z\}_{\cI^{(L)}}}\left\| \frac{1}{m^{(L)}}\sum_{j=1}^{m^{(L)}} u^{(L)}_j\left[\bar{f}^{(L)}(z^{(L)}_j; x) -f^{(L)}(\rho, z^{(L)}_j;x) \right] \right \|\notag\\
    \leq&  \EE_{\{z\}_{\cI^{(L)}}}\frac{1}{m^{(L)}}\sum_{j=1}^{m^{(L)}}\left\|  u^{(L)}_j\left[\bar{f}^{(L)}(z^{(L)}_j; x) -f^{(L)}(\rho, z^{(L)}_j;x) \right] \right \|\notag\\
    =&\EE_{z^{(L)}} \EE_{\{z\}_{\cI^{(L-1)}}}\left\|  u^{(L)}\left[\bar{f}^{(L)}(z^{(L)}; x) -f^{(L)}(\rho, z^{(L)};x) \right] \right \|\notag\\
    \leq& \EE_{z^{(L)}} \EE_{\{z\}_{\cI^{(L-1)}}}\left\| u^{(L)}\right\|\left|  \bar{f}^{(L)}(z^{(L)}; x) -f^{(L)}(\rho, z^{(L)};x) \right |.\notag
\end{align}
We obtain \eqref{theo:end1}. On the other hand,  taking expectation on $z^{(\ell)}$ for \eqref{theo1:a2}, we can obtain \eqref{theoe}.
\end{proof}

\noindent\textbf{Step 2.} We prove for  all  $\ell \in \{2, \dots,L\}$ and $j\in \{2, \dots, \ell\}$, we have 
\begin{align}
  \lim_{m^{(j-1)}\to \infty}  \E_{z^{(j)}, \dots, z^{(\ell)}} \E_{\hcZ^{(j-1)}}  A(j, \ell)  \Delta^{(j)}(z^{(j)}) = 0. \label{theo1:3}
\end{align}
For all $j\in \{2, \dots, L+1\}$,
\begin{align}
  \lim_{m^{(j-1)}\to \infty}  \E_{z^{(j)}, \dots, z^{(L)}} \E_{\hcZ^{(j-1)}}  A(j, L+1)  \Delta^{(j)}(z^{(j)}) = 0. \label{theo1:31}.
\end{align}
Note that when $j=L+1$, $\Delta(z^{(L+1)})$ is vector and the expectation is only taken on $\hcZ^{(L)}$.

\begin{proof}
We denote 
$$
Q{(\ell)}:=
\begin{cases}
1,&   \ell= L+1, \\
\left(\|u^{(L)}(z^{(L)})\|\vee 1]\right)\prod_{i=\ell+1}^{L} \left( |w(z^{(i)}, z^{(i-1)})|\vee 1\right),& \ell\in [L].
\end{cases}
$$
In fact, for all $\ell \in \{2, \dots, L\}$ and $j\in \{2, \dots, \ell\}$, we have 
\begin{align}
     \E_{z^{(j)}, \dots, z^{(\ell)}} \E_{\hcZ^{(j-1)}}  A(j, \ell)  \Delta^{(j)}(z^{(j)})\leq \E_{z^{(j)}, \dots, z^{(L)}} \E_{\hcZ^{(j-1)}}  Q{(j)}\Delta^{(j)}(z^{(j)}).\label{theot}
\end{align}
It is also true that all for $j\in [L+1]$, 
\begin{align}
    \E_{z^{(j)}, \dots, z^{(L)}} \E_{\hcZ^{(j-1)}}  A(j, L+1)  \Delta^{(j)}(z^{(j)})\leq \E_{z^{(j)}, \dots, z^{(L)}} \E_{\hcZ^{(j-1)}}  Q{(j)}\Delta^{(j)}(z^{(j)}).\label{theot2}
\end{align}
In the following, we denote
 $$
\xi^{(j)}_k := 
\begin{cases}
Q(j) \left[w(z^{(j)},z^{(j-1)}_k) f^{(j-1)}(\rho, z^{(j-1)}_k;x)- g^{(j)}(\rho,z^{(j)};x)\right],&  j\in \{2, \dots, L\},~~k\in[m^{(j-1)}], \\
u( z^{(L)}_k) f^{(L)}(\rho, z^{(L)}_k;x)-f(\rho,u;x),&    j=L+1,~~k\in[m^{(L)}].
\end{cases}
$$
Note that $\xi^{(L+1)}_k$ with $k\in[m^{(L)}]$ are  $K$ dimensional vectors. Because the real numbers can also be treated as   one dimensional vectors.  Below we treat $\xi^{(j)}_k$ as a vector for the sake of simplicity.  We have for all  $j\in \{2, \dots, L+1\}$
\begin{align}
&  \E_{z^{(j)}, \dots, z^{(L)}} \E_{\hcZ^{(j-1)}}  Q{(j)}\Delta^{(j)}(z^{(j)})\label{theo:a4}\\
  = & \E_{z^{(j)}, \dots, z^{(L)}} \E_{\hcZ^{(j-1)}}  \left\|\frac{1}{m^{(j-1)}} \sum_{k=1}^{m^{(j-1)}} \xi^{(j)}_k\right\|\notag\\
    \overset{a}{\leq}&   \E_{z^{(j)}, \dots, z^{(L)}} \E_{\hcZ^{(j-1)}}  \left\|\frac{1}{m^{(j-1)}} \sum_{k=1}^{m^{(j-1)}} \xi^{(j)}_k \bone_{\|\xi^{(j)}_k\|\leq M}\right\| +  \E_{z^{(j)}, \dots, z^{(L)}}\E_{\hcZ^{(j-1)}}  \left\|\frac{1}{m^{(j-1)}} \sum_{k=1}^{m^{(j-1)}} \xi^{(j)}_k\bone_{\|\xi^{(j)}_k\|> M}\right\|\notag\\
    \overset{b}{\leq}& \E_{z^{(j)}, \dots, z^{(L)}} \sqrt{\E_{\hcZ^{(j-1)}}  \left\|\frac{1}{m^{(j-1)}} \sum_{k=1}^{m^{(j-1)}} \xi^{(j)}_k\bone_{\|\xi^{(j)}_k\|\leq M}\right\|^2}+ \E_{z^{(j)}, \dots, z^{(L)}}\E_{\hcZ^{(j-1)}} \frac{1}{m^{(j-1)}} \sum_{k=1}^{m^{(j-1)}} \left\|\xi^{(j)}_k\bone_{\|\xi^{(j)}_k\|> M} \right\|\notag\\
    \leq& \E_{z^{(j)}, \dots, z^{(L)}} \sqrt{\E_{\hcZ^{(j-1)}}  \left\|\frac{1}{m^{(j-1)}} \sum_{k=1}^{m^{(j-1)}} \xi^{(j)}_k\bone_{\|\xi^{(j)}_k\|\leq M}\right\|^2}+ \E_{z^{(j)}, \dots, z^{(L)}}\E_{z^{(j-1)}_1} \left\|\xi^{(j)}_1\bone_{\|\xi^{(j)}_1\|> M} \right\|,\notag
\end{align}
where in $\overset{a}=$, we use $\bone$ to denote indicator function, then $$\frac{1}{m^{(j-1)}} \sum_{k=1}^{m^{(j-1)}} \xi^{(j)}_k =\frac{1}{m^{(j-1)}} \sum_{k=1}^{m^{(j-1)}} \xi^{(j)}_k \bone_{\|\xi^{(j)}_k\|\leq M}+\frac{1}{m^{(j-1)}} \sum_{k=1}^{m^{(j-1)}} \xi^{(j)}_k \bone_{\|\xi^{(j)}_k\|> M},$$
and obtain the result by triangle inequality, and $\overset{b}\leq$ uses Jensen's inequality.
For the first term in the right hand side of \eqref{theo:a4}, we have 
\begin{align}
    &\E_{z^{(j)}, \dots, z^{(L)}} \sqrt{\E_{\hcZ^{(j-1)}}  \left\|\frac{1}{m^{(j-1)}} \sum_{k=1}^{m^{(j-1)}} \xi^{(j)}_k\bone_{\|\xi^{(j)}_k\|\leq M}\right\|^2}\label{theo:end4}\\
    =&\E_{z^{(j)}, \dots, z^{(L)}} \Bigg(\frac{1}{(m^{(j-1)})^2}\sum_{k_1=1}^{m^{(j-1)}}\E_{z^{(j-1)}_{k_1}}\left\| \xi^{(j)}_{k_1}\bone_{\|\xi^{(j)}_{k_1}\|\leq M }\right\|^2\notag\\
    &\quad\quad\quad\quad\quad\quad\quad\quad +\frac{1}{(m^{(j-1)})^2}\sum_{k_1=1}^{m^{(j-1)}}\sum_{k_2=1, k_2\neq k_1}^{m^{(j-1)}} \E_{z^{(j-1)}_{k_1}}\E_{z^{(j-1)}_{k_2}} \left[\xi^{(j)}_{k_1}\bone_{\|\xi^{(j)}_{k_1}\|\leq M}\right]^\top \left[\xi^{(j)}_{k_2}\bone_{\|\xi^{(j)}_{k_2}\|\leq M}\right]\Bigg)^{1/2} \notag\\
    \overset{a}=&\E_{z^{(j)}, \dots, z^{(L)}} \sqrt{\frac{1}{(m^{(j-1)})}\E_{z^{(j-1)}_{1}}\left\| \xi^{(j)}_{1}\bone_{\|\xi^{(j)}_1\|\leq M }\right\|^2+\frac{m^{(j-1)}-1}{m^{(j-1)}} \left\|\E_{z^{(j-1)}_{1}} \left[\xi^{(j)}_{1}\bone_{\|\xi^{(j)}_1\|\leq M}\right] \right\|^2}\notag\\
    =&\E_{z^{(j)}, \dots, z^{(L)}} \sqrt{\frac{1}{(m^{(j-1)})}\E_{z^{(j-1)}_{1}}\left\| \xi^{(j)}_{1}\bone_{\|\xi^{(j)}_1\|\leq M }\right\|^2+\underbrace{\frac{m^{(j-1)}-1}{m^{(j-1)}}}_{\leq 1} \underbrace{\left\| \E_{z^{(j-1)}_{1}}\left[\xi^{(j)}_1- \xi^{(j)}_1 \bone_{\|\xi^{(j)}_1\|> M}\right] \right\|^2}_{\EE_{z^{(j-1)}_{1}}  \xi_1^{j}=0}}\notag\\
        \leq&\E_{z^{(j)}, \dots, z^{(L)}} \sqrt{\frac{1}{(m^{(j-1)})}\E_{z^{(j-1)}_{1}}\left\| \xi^{(j)}_{1}\bone_{\|\xi^{(j)}_1\|\leq M }\right\|^2+ \left\| \E_{z^{(j-1)}_{1}}\left[\xi^{(j)}_1 \bone_{\|\xi^{(j)}_1\|> M}\right] \right\|^2}\notag\\
        \overset{b}\leq& \frac{1}{\sqrt{m^{(j-1)}}}\E_{z^{(j)}, \dots, z^{(L)}}\sqrt{\E_{z^{(j-1)}_{1}}\left\| \xi^{(j)}_{1}\bone_{\|\xi^{(j)}_1\|\leq M }\right\|^2  }+ \E_{z^{(j)}, \dots, z^{(L)}}\left\| \E_{z^{(j-1)}_{1}} \xi^{(j)}_1 \bone_{\|\xi^{(j)}_1\|> M} \right\|\notag\\
        \overset{c}{\leq}&\frac{1}{\sqrt{m^{(j-1)}}}\sqrt{\E_{z^{(j)}, \dots, z^{(L)}}\E_{z^{(j-1)}_{1}}\underbrace{\left[ \xi^{(j)}_{1}\bone_{\|\xi^{(j)}_1\|\leq M }\right]^2}_{\leq M^2}  }+ \E_{z^{(j)}, \dots, z^{(L)}} \E_{z^{(j-1)}_{1}} \left\|\xi^{(j)}_1 \bone_{\|\xi^{(j)}_1\|> M}\right\| \notag\\
        \leq&\frac{M}{\sqrt{m^{(j-1)}}}+ \E_{z^{(j)}, \dots, z^{(L)}} \E_{z^{(j-1)}_{1}} \left\|\xi^{(j)}_1 \bone_{\|\xi^{(j)}_1\|> M}\right\|,\notag
\end{align}
where in $\overset{a}=$, we use $\left[\xi^{(j)}_{k_1}\bone_{\|\xi^{(j)}_{k_1}\|\leq M}\right]$ and $\left[\xi^{(j)}_{k_2}\bone_{\|\xi^{(j)}_{k_2}\|\leq M}\right]$ are independent given $z^{(j)}, \dots, z^{(L)}$, and obtain the result by:
\begin{align}
    &\sum_{k_1=1}^{m^{(j-1)}}\sum_{k_2=1, k_2\neq k_1}^{m^{(j-1)}} \E_{z^{(j-1)}_{k_1}}\E_{z^{(j-1)}_{k_2}} \left[\xi^{(j)}_{k_1}\bone_{\|\xi^{(j)}_{k_1}\|\leq M}\right]^\top \left[\xi^{(j)}_{k_2}\bone_{\|\xi^{(j)}_{k_2}\|\leq M}\right]\notag\\
    =&\sum_{k_1=1}^{m^{(j-1)}}\sum_{k_2=1, k_2\neq k_1}^{m^{(j-1)}}\left[\E_{z^{(j-1)}_{k_1}} \xi^{(j)}_{k_1}\bone_{\|\xi^{(j)}_{k_1}\|\leq M}  \right]^\top\left[\E_{z^{(j-1)}_{k_2}} \xi^{(j)}_{k_2}\bone_{\|\xi^{(j)}_{k_2}\|\leq M}  \right]\notag\\
    =&m^{(j-1)}(m^{(j-1)}-1) \left\|\E_{z^{(j-1)}_{1}} \xi^{(j)}_{1}\bone_{\|\xi^{(j)}_{1}\|\leq M}  \right\|^2,\notag
\end{align}
$\overset{b}\leq$ uses the fact that
$$\sqrt{a^2+b^2}\leq a+b, ~~~~ a\geq0,~ b\geq0,   $$
and  $\overset{c}\leq$ uses Jensen's inequality.

The rest to do is to  bound the term  $\E_{z^{(j)}, \dots, z^{(L)}}\E_{z^{(j-1)}_1} \left\|\xi^{(j)}_1\bone_{\|\xi^{(j)}_1\|> M} \right\|$. We have 
\begin{align}
    &\E_{z^{(j)}, \dots, z^{(L)}}\E_{z^{(j-1)}_1} \left\|\xi^{(j)}_1\bone_{\|\xi^{(j)}_1\|> M} \right\|\notag\\
    \leq&   \E_{z^{(j)}, \dots, z^{(L)}}\E_{z^{(j-1)}_1} \left\|\xi^{(j)}_1\right\|\bone_{\|\xi^{(j)}_1\|> M}\label{theo:e3}\\
    \overset{a}{\leq}& \E_{z^{(j)}, \dots, z^{(L)}}\E_{z^{(j-1)}_1} \frac{1}{1+c_{\ep}}\left\|\frac{\xi^{(j)}_1}{M^{\frac{c_{\ep}}{(1+c_{\ep})}}}\right\|^{1+c_{\ep}} + \E_{z^{(j)}, \dots, z^{(L)}}\E_{z^{(j-1)}_1} \tilde{c}_{\ep}\left[\bone_{\|\xi^{(j)}_1\|> M}M^{\frac{c_{\ep}}{(1+c_{\ep})}}\right]^{1/\tilde{c}_{\ep}}\notag\\
    =&\frac{\E_{z^{(j)}, \dots, z^{(L)}}\E_{z^{(j-1)}_1} \frac{1}{1+c_{\ep}}\|\xi^{(j)}_1\|^{1+c_{\ep}}}{M^{c_{\ep}}} +\tilde{c}_{\ep} M \E_{z^{j}, \dots, z^{(L)}}\E_{z^{(j-1)}_1} \left[\bone_{\|\xi^{(j)}_1\|> M}\right]\notag\\
    \overset{b}\leq&
\frac{\E_{z^{(j)}, \dots, z^{(L)}}\E_{z^{(j-1)}_1} \frac{1}{1+c_{\ep}}\|\xi^{(j)}_1\|^{1+c_\ep}}{M^{c_{\ep}}}  +\tilde{c}_{\ep}  \frac{ \E_{z^{j}, \dots, z^{(L)}}\E_{z^{(j-1)}_1} \|\xi^{(j)}_1\|^{1+c_{\ep}}}{M^{c_{\ep}}}\notag\\
        \overset{c}\leq& 
    \frac{c_M \vee c_{M_1}}{M^{c_{\ep}}},\notag
\end{align}
where in $\overset{a}\leq$, we set $\tilde{c}_{\ep} = \frac{c_{\ep}}{1+c_{\ep}}$ and obtain the result by Yong's inequality, that is 
$$ | xy |\leq \frac{|x|^{p}}{p} + \frac{|y|^{q}}{q},  \quad\quad \frac{1}{p}+\frac{1}{q}=1,$$
with $x= \left\|\frac{\xi^{(j)}_1}{M^{\frac{c_{\ep}}{(1+c_{\ep})}}}\right\|$, $y=\bone_{\|\xi^{(j)}_1\|> M}M^{\frac{c_{\ep}}{(1+c_{\ep})}}$, $p = 1 + c_{\ep} $ and $q = \frac{1+c_{\ep}}{c_{\ep}}$, in $\overset{b}\leq$, we use Markov's inequality, i.e. for random variable $\xi$,  we have
$$ \PP(|\xi|\geq M) \leq \frac{\E |\xi|^{1+c_{\ep}}}{M^{1+c_{\ep}}}, $$
$\overset{c}{\leq}$ uses Assumption \ref{ass:3}.
Finally,  by plugging \eqref{theo:end4} and \eqref{theo:e3} into  \eqref{theo:a4},  for all  $j\in \{2, \dots, L+1\}$, we have
\begin{align}
 \E_{z^{(j)}, \dots, z^{(L)}} \E_{\hcZ^{(j-1)}}  Q{(j)}\Delta^{(j)}(z^{(j)})
     \leq\frac{M}{\sqrt{m^{(j-1)}}} + \frac{2(c_M \vee c_{M_1})}{M^{c_{\ep}}}\overset{a}= \mathcal{O}\left(\frac{1}{(m^{(j-1)})^{\frac{c_{\ep}}{2(1+c_{\ep})}}}\right).
\end{align}
where in $\overset{a}=$, we set  $M =\left(m^{(j-1)}\right)^{\frac{1}{2(1+c_{\ep})}}$. 
Then with $m^{(j-1)}\to \infty$, from \eqref{theot} and \eqref{theot2}, we obtain \eqref{theo1:3} and \eqref{theo1:31}, respectively. 
\end{proof}

\noindent\textbf{Proof of Theorem \ref{theorem:convergence-discrete-NN}:}
In fact, by plugging \eqref{theo1:31} into \eqref{theo:end1}, we obtain  (\ref{the:con2}) in Theorem \ref{theorem:convergence-discrete-NN}. Similarity,  we obtain  (\ref{the:con1}) in  Theorem \ref{theorem:convergence-discrete-NN} via plugging \eqref{theo1:3} into \eqref{theoe}, which finishes the proof.

\section{Proof of Theorem \ref{var}: Variance of Discrete Approximation}\label{app:2}
In this section, we prove Theorem \ref{var} which estimates the variance of discrete approximation. The proof is more complicated than that of Theorem \ref{theorem:convergence-discrete-NN}, because we  consider the first-order approximation of $\hf^{(\ell)}_j(x)$ with $\ell\in [L]$.     Before giving the detailed proof,  we further define the following:

\begin{itemize}
\item[(1)] We let $\breve{f}^{(0)}(\rho, z^{(0)};x) := x_{z^{(0)}}$ with $z^{(0)}\in \cZ{0}$. For $z^{(\ell)}\in \cZ{\ell}$ and $\ell\in [L]$, we define 
\begin{align}
\!\!\!\! \breve{f}^{(\ell)}(\rho, z^{(\ell)};x) :=&f^{(\ell)}(\rho,z^{(\ell)};x)+\nabla \hl\left( \gl(\rho,\zl;x) \right)\!\left[\breve{g}^{(\ell)}(\rho, z^{(\ell)};x) - \gl(\rho,\zl;x) \right],\label{def bf}
\end{align}
where
\begin{align}
    \breve{g}^{(\ell)}(\rho,z^{(\ell)};x):=\frac{1}{m^{(\ell-1)}} \sum_{i=1}^{m^{(\ell-1)}} w(\zl, \zmi_i)\breve{f}^{(\ell-1)}(\rho,z^{(\ell-1)}_i;x), \mbox{ with } \ell\in [L].\nonumber
\end{align}

Moreover, we define 
\begin{align}\label{ffo}
\breve{f}(\rho,u ; x) := \frac{1}{m^{(L)}}\sum_{i=1}^{m^{(L)}}\breve{f}^{(L)}(\rho,\zL_i;x)u(\zL_i).
\end{align}
Note that $\breve{f}^{(\ell)}$ and $ \breve{g}^{(\ell)}$  depend on random variables $\hcZ^{(1)}, \dots, \hcZ^{(\ell-1)}$.  For  simplicity, we do not show them in the definitions explicitly.

\item[(2)] We define $S^{(\ell)}(\rho,z^{(1)}, \ldots, z^{(\ell)};x)$ for $\ell \in [L]$ by letting
\begin{align}
&S^{(1)}(\rho, z^{(1)};x) :=h^{(1)}\left( \frac{1}{d}\sum_{i=1}^d w(z^{(1)}, z^{(0)}_i) x_{z^{(0)}_i}\right),\nonumber\\
&S^{(\ell)}(\rho,z^{(1)}, \dots, z^{(\ell)};x):=\hl \left( \gl(\rho,\zl;x) \right) \label{def S}\\
&\quad \quad \quad  +\nabla \hl\left(\gl(\rho,\zl;x) \right)\left[ w(\zl, \zmi)S^{(\ell-1)}(\rho, z^{(1)},\ldots, \zmi;x)- \gl(\rho, \zl;x) \right],\nonumber
\end{align}
where $\ell\in\{2, 3,\dots, L\}$. Finally, we define
\begin{align*}
\tilde{S}(\rho,z^{(1)}, \dots, z^{(L)},u;x) := S^{(L)}(\rho,z^{(1)}, \dots, z^{(L)};x)u( z^{(L)}) - f(\rho,u;x). \notag
\end{align*}
We notice that the Assumption \ref{ass:3} is slight stronger than assuming the bounded $q_0$-th moment on  $\tilde{S}(\rho,z^{(1)}, \dots,  z^{(L)},u;x)$. From the bounded gradient of   Assumption \ref{ass:2}, we will prove in \eqref{eq10} that:
\begin{align}
\E_{z^{(1)}, \dots, z^{(L)}} \left\|\tilde{S}(\rho,z^{(1)}, \dots,  z^{(L)},u;x)\right\|^2\leq c_{M_2},\nonumber
\end{align}
where $c_{M_2}$ is a constant, and in \eqref{eq13} that $\E\tilde{S}(\rho,z^{(1)}, \dots, z^{(L)},u;x)=0 $.
\end{itemize}

We give the detailed proof of Theorem \ref{var} below:\\\\
\textbf{Step 1.} We prove 
\begin{align}
 \E_{\{z\}_{\cI^{(\ell-1)}}}\breve{f}^{(\ell)}(\rho,\zl_j;x) = f^{(\ell)}(\rho,\zl_j;x), ~~\ell\in [L],\label{stt1}
\end{align} 
and 
$$ \E_{\{z\}_{\cI^{(L)}}}\breve{f}(\rho,u;x) = f(\rho,u;x). $$
\begin{proof}
We prove it by induction. When $\ell=1$, the statement is true. Suppose at $\ell-1$, the statement is true.  We consider the case of $\ell$. For all $j\in [m^{(\ell)}]$, we have
\begin{align}
&\EE_{\{z\}_{\cI^{(\ell-1)}}}\breve{f}^{(\ell)}(\rho,\zl_j;x)\notag\\
=& \EE_{\hcZ^{(\ell-1)}}\left[\EE_{\{z\}_{\cI^{\ell-2}}}\breve{f}^{(\ell)}(\rho,\zl_j;x)\right]\notag\\
\overset{\eqref{def bf}}{=}& \EE_{\hcZ^{(\ell-1)}}\Bigg[\nabla \hl\left( \gl(\rho,\zl_j;x) \right)\Bigg( \frac{1}{m^{(\ell-1)}} \sum_{i=1}^{m^{(\ell-1)}} w(\zl_j, \zmi_i)\EE_{\{z\}_{\cI^{\ell-2}}}\breve{f}^{(\ell-1)}(\rho,\zmi_i;x)\nonumber\\
&\quad \quad\quad\quad\quad \quad\quad\quad\quad\quad\quad \quad\quad\quad~~~ - \gl(\rho,\zl_j;x) \Bigg)\Bigg]+\EE_{\hcZ^{(\ell-1)}}\left[\hl \left( \gl(\rho,\zl_j;x) \right)\right]\notag\\
\overset{a}=&\hl ( \gl(\rho,\zl_j;x)+ \EE_{\hcZ^{(\ell-1)}}\nabla \hl\left( \gl(\rho,\zl_j;x) \right)\Bigg[ \frac{1}{m^{(\ell-1)}} \sum_{i=1}^{m^{(\ell-1)}} \notag\\
& \quad \quad\quad\quad\quad \quad\quad\quad\quad\quad\quad \quad\quad\quad\quad~~~ w(\zl_j, \zmi_i)f^{(\ell-1)}(\rho,\zmi_i;x)- \gl(\rho,\zl_j;x) \Bigg]\notag\\
 \overset{\eqref{eq:nn-layer}}{=}&\hl \left( \gl(\rho,\zl_j;x )\right) \notag\\
\overset{\eqref{eq:nn-activation}}{=}& f^{(\ell)}(\rho,\zl_j;x),\notag
\end{align}
where  $\overset{a}=$ applies \eqref{stt1} on $\ell-1$.
Then 
\begin{align}
\E \breve{f}(\rho,u;x) =& \E_{\{z\}_{\cI^{(L)}}} \frac{1}{m^{(L)}}\sum_{j=1}^{m^{(L)}}\breve{f}^{(L)}(\rho,z_j^{(L)};x) u_j^{(L)} \nonumber\\
=&\E_{\hcZ^{(L)}} \frac{1}{m^{(L)}}\sum_{j=1}^{m^{(L)}}f^{(L)}(\rho,z_j^{(L)};x) u(z^{(L)}_j)\nonumber\\
=& f(\rho,u;x).\nonumber
\end{align}
\end{proof}

\noindent\textbf{Step 2.}  In this step, we compute 
\begin{align}
\E\left\| \breve{f}(\rho, u;x)-f(\rho,u;x)\right\|^2.\nonumber
\end{align}
\textbf{Step 2(i):}
We  prove that
 \begin{align}\label{bff1}
 \breve{f}(\rho, u;x)=  \prod_{\ell=1}^L \frac{1}{m^{(\ell)}}\left[\sum_{s\in \cI^{(L)}} S^{(L)}(\rho,z^{(1)}_{s_1}, \dots,z^{(L)}_{s_L};x) u(z^{(L)}_{s_L})\right].
 \end{align}

\begin{proof}
Clearly, we have $\breve{f}^{(1)} (\rho,z^{(1)}_j;x) =S^{(1)} (\rho,z^{(1)}_j;x) =f^{(1)} (\rho,z^{(1)}_j;x)$  for all $j\in [m^{(1)}]$.  Then we can  verify that  for all  $j\in  [m^{(2)}]$, 
\begin{align}
\breve{f}^{(2)}(\rho,z^{(2)}_{j};x) = \frac{1}{m^{(1)}} \sum_{i=1}^{m^{(1)}}S^{(2)}(\rho,z^{(1)}_{i}, z^{(2)}_{j};x).\label{ffo0}
\end{align}

We then prove for all $\ell \in \{ 3, \dots, L\}$ and  $j\in [ m^{(\ell)}]$, we have 
\begin{align}
\breve{f}^{(\ell)}(\rho,z^{(\ell)}_{j};x) =  \prod_{k=1}^{\ell-1} \frac{1}{m^{(k)}}\left[ \sum_{s\in\cI^{(\ell-1)}}S^{(\ell)}(\rho,z^{(1)}_{s_1}, \dots,z^{(\ell-1)}_{s_{\ell-1}}, z^{(\ell)}_{j};x)\right]. \label{ffo1}
\end{align}

Suppose the above statement is  true at $\ell-1$, then from \eqref{def bf}, we have
\begin{align}
&\breve{f}^{(\ell)}(\rho,\zl_{j};x)\notag\\
 \overset{\eqref{def bf}}=& f^{(\ell)} (\rho,\zl_j;x ) -\nabla \hl\left( \gl(\rho,\zl_j;x) \right) \gl(\rho,\zl_j;x)+ \nabla \hl\left( \gl(\rho,\zl_j;x) \right)\times \notag \\
 &\quad\quad\quad\quad\quad\quad\quad\quad\quad\quad\quad\quad\quad\quad\quad\quad~~~~ \frac{1}{m^{(\ell-1)}} \sum_{i=1}^{m^{(\ell-1)}} w(\zl_j, \zmi_i)\breve{f}^{(\ell-1)}(\rho,\zmi_i;x)\notag\\
 \overset{a}=&f^{(\ell)} (\rho,\zl_j;x )-\nabla \hl\left( \gl(\rho,\zl_j;x) \right) \gl(\rho,\zl_j;x) + \nabla \hl\left( \gl(\rho,\zl_j;x) \right)\times \nonumber \\
 &~~~\underline{\frac{1}{m^{(\ell-1)}} \sum_{i=1}^{m^{(\ell-1)}}}\left[\underline{\prod_{k=1}^{\ell-2} \frac{1}{m^{(k)}}}\left(\underline{\sum_{s\in\cI^{\ell-2}}} S^{(\ell-1)}(\rho,z^{(1)}_{s_1}, \dots,z^{(\ell-2)}_{s_{\ell-2}}, z^{(\ell-1)}_i;x)\right)w(\zl_{j}, \zmi_{i})\right]\notag\\
=&f^{(\ell)} (\rho,\zl_j;x )-\nabla \hl\left( \gl(\rho,\zl_j;x) \right) \gl(\rho,\zl_j;x)+\nabla \hl\left( \gl(\rho,\zl_j;x) \right) \times \notag\\ &\quad\quad\quad\quad\quad\quad\quad\quad\quad~~~~~\prod_{k=1}^{\ell-1} \frac{1}{m^{(k)}}\left[\sum_{s\in\cI^{(\ell-1)}} S^{(\ell-1)}(\rho,z^{(1)}_{s_1}, \dots,z^{(\ell-1)}_{s_{\ell-1}};x)w(\zl_{j}, \zmi_{s_{\ell-1}})\right]\notag\\
 =&  \prod_{k=1}^{\ell-1} \frac{1}{m^{(k)}}\left[\sum_{s\in\cI^{(\ell-1)}} \left(f^{(\ell)} (\rho,\zl_j;x ) -\nabla \hl\left( \gl(\rho,\zl_j;x) \right) \gl(\rho,\zl_j;x)\right)\right] \notag\\
 &+ \nabla \hl\left( \gl(\rho,\zl_j;x) \right)\prod_{k=1}^{\ell-1} \frac{1}{m^{(k)}}\left[\sum_{s\in\cI^{(\ell-1)}} S^{(\ell-1)}(\rho,z^{(1)}_{s_1}, \dots,z^{(\ell-1)}_{s_{\ell-1}};x)w(\zl_{j}, \zmi_{s_{\ell-1}})\right]\notag\\
 =&\prod_{k=1}^{\ell-1} \frac{1}{m^{(k)}}\left[ \sum_{s\in\cI^{(\ell-1)}}S^{(\ell)}(\rho,z^{(1)}_{s_1}, \dots,z^{(\ell-1)}_{s_{\ell-1}}, z^{(\ell)}_{j};x)\right],\notag
\end{align}
where in $\overset{a}=$, we use \eqref{ffo1} on $\ell-1$. So we obtain \eqref{ffo1}  on $\ell$.  Then plugging \eqref{ffo1} with $\ell =L$ into \eqref{ffo}, we can obtain \eqref{bff1}. Also by the same technique of Step 1, we can obtain 
\begin{align}
\EE_{z^{(1)}, \dots, z^{(\ell-1)}} S^{(\ell)}(\rho,z^{(1)}, \dots, z^{(\ell)};x) = f^{(\ell)}(\rho,z^{(\ell)};x), ~~\ell\in [L],\label{exp z}
\end{align}
and
\begin{align}
\EE_{z^{(1)}, \dots, z^{(L)}} S^{(L)}(\rho,z^{(1)}, \dots, z^{(L)};x)u( z^{(L)})  = f(\rho, u;x).\label{exp l}
\end{align}
Recall that
\begin{align}
\tilde{S}(\rho,z^{(1)}, \dots, z^{(L)},u;x) = S^{(L)}(\rho,z^{(1)}, \dots, z^{(L)};x)u( z^{(L)}) - f(\rho,u;x),\nonumber
\end{align}
we have 
\begin{align}
 \E_{z^{(1)}, \dots, z^{(L)}}\tilde{S}(\rho,z^{(1)}, \dots, z^{(L)},u;x) = 0. \label{eq10}  
\end{align}
\end{proof}

\noindent\textbf{Step 2(ii):} Firstly, We expand $$\E\left\|\breve{f}(\rho,u;x) - f(\rho,u;x)\right\|^2.$$
\begin{proof}
We have
\begin{align}
&\E \left\| \breve{f}(\rho,u;x) - f(\rho,u;x) \right\|^2\label{lin}\\
=&\prod_{\ell=1}^{L}\frac{1}{(m^{(\ell)})^2}\E\left\| \sum_{s\in \cI^{(L)}} \tilde{S}(\rho,z^{(1)}_{s_1}, \dots, z^{(L)}_{s_L},u;x)     \right\|^2\notag\\
=&\prod_{\ell=1}^{L}\frac{1}{(m^{(\ell)})^2}\E \left[ \sum_{s^a\in \cI^{(L)}}\sum_{s^b\in \cI^{(L)}} \tilde{S}^{\top}(\rho,z^{(1)}_{s^a_1}, \dots, z^{(L)}_{s^a_L},u;x) \tilde{S}(\rho,z^{(1)}_{s^b_1}, \dots, z^{(L)}_{s^b_L},u;x)      \right]\notag\\
=&\sum_{s^a\in \cI^{(L)}}\sum_{s^b\in \cI^{(L)}}\left[ \prod_{\ell=1}^{L}\frac{1}{(m^{(\ell)})^2} \E_{z^{(1)}_{s^a_1}, \dots, z^{(L)}_{s^a_L}}\E_{z^{(1)}_{s^b_1}, \dots, z^{(L)}_{s^b_L}}\tilde{S}^{\top}(\rho,z^{(1)}_{s^a_1}, \dots, z^{(L)}_{s^a_L},u;x) \tilde{S}(\rho,z^{(1)}_{s^b_1}, \dots, z^{L}_{s^b_L},u;x) \right]\notag\\
\overset{a}{=}&\sum_{s^a\in \cI^{(L)}}\sum_{s^b\in \cI^{(L)}}\left[ \prod_{\ell=1}^{L}\frac{1}{(m^{(\ell)})^2} \E_{\zl_{s^a_\ell}, \ell\in p^{ab}}\E_{\zl_{s^a_\ell}, \ell\in q^{ab}}\E_{\zl_{s^b_\ell}, \ell\in q^{ab}}\tilde{S}^{\top}(\rho,z^{(1)}_{s^a_1}, \dots, z^{(L)}_{s^a_L},u;x) \tilde{S}(\rho,z^{(1)}_{s^b_1}, \dots, z^{L}_{s^b_L},u;x) \right]\notag\\
\overset{b}{=}&\sum_{s^a\in \cI^{(L)}}\sum_{s^b\in \cI^{(L)}}\left[ \prod_{\ell=1}^{L}\frac{1}{(m^{(\ell)})^2} \E_{\zl_{s^a_\ell}, \ell\in p^{ab}}\left\|\E_{\zl_{s^a_\ell}, \ell\in q^{ab}}\tilde{S}(\rho,z^{(1)}_{s^a_1}, \dots, z^{(L)}_{s^a_L},u;x)\right\|^2  \right]\notag\\
\overset{c}{=}&\sum_{\ell=0}^L \sum_{\mathcal{a}\in \mathcal{C}^{(\ell)}}\left[ \Big(\prod_{k\in \mathcal{a}}\frac{1}{m^{(k)}} \prod_{k\in \mathcal{a}^c}\frac{m^{(k)}-1}{m^{(k)}}\Big) \E_{z^{(i)}, i\in \mathcal{a}}\left\|\E_{ z^{(j)}, j\in \mathcal{a}^c}\tilde{S}(\rho,z^{(1)}, \dots, z^{(L)},u;x)\right\|^2  \right]\notag\\
\overset{d}{=}&\sum_{\ell=1}^L \sum_{\mathcal{a}\in \mathcal{C}^{(\ell)}}\left[ \Big(\prod_{k\in \mathcal{a}}\frac{1}{m^{(k)}} \prod_{k\in \mathcal{a}^c}\frac{m^{(k)}-1}{m^{(k)}}\Big) \E_{z^{(i)}, i\in \mathcal{a}}\left\|\E_{ z^{(j)}, j\in \mathcal{a}^c}\tilde{S}(\rho,z^{(1)}, \dots, z^{(L)},u;x)\right\|^2  \right].\notag
\end{align}
We explain \eqref{lin}. In $\overset{a}=$, we denote $p^{ab} = \{\ell: s^a_\ell = s^b_\ell, ~ \ell \in [L]\}$ and $q^{ab} = \{\ell: s^a_\ell \neq s^b_\ell, \ell \in [L]\}$. In $\overset{b}=$,  given $\zl_{s^a_\ell}$ with $\ell\in p^{ab}$, $\tilde{S}(\rho,z^{(1)}_{s^a_1}, \dots, z^{(L)}_{s^a_L},u;x)$ and $\tilde{S}(\rho,z^{(1)}_{s^b_1}, \dots, z^{L}_{s^b_L},u;x)$ are independent, so 
\begin{align}
&\E_{\zl_{s^a_\ell}, \ell\in q^{ab}}\E_{\zl_{s^b_\ell}, \ell\in q^{ab}}\tilde{S}^{\top}(\rho,z^{(1)}_{s^a_1}, \dots, z^{(L)}_{s^a_L},u;x) \tilde{S}(\rho,z^{(1)}_{s^b_1}, \dots, z^{L}_{s^b_L},u;x)\notag\\
=& \E_{\zl_{s^a_\ell}, \ell\in q^{ab}}\tilde{S}^{\top}(\rho,z^{(1)}_{s^a_1}, \dots, z^{(L)}_{s^a_L},u;x)\E_{\zl_{s^b_\ell}, \ell\in q^{ab}} \tilde{S}(\rho,z^{(1)}_{s^b_1}, \dots, z^{L}_{s^b_L},u;x) \notag\\
=&    \left\|\E_{\zl_{s^a_\ell}, \ell\in q^{ab}} \tilde{S}(\rho,z^{(1)}_{s^a_1}, \dots, z^{(L)}_{s^a_L},u;x)\right\|^2.\nonumber
\end{align}
In $\overset{c}=$, we let $\mathcal{C}^{(\ell)}$
 be the set of all the $\ell$-combinations of $[L]$, specially $\mathcal{C}^{(0)} = \{\varnothing\}$. Then for all $\mathcal{a}\in \mathcal{C}^{(\ell)}$, we denote $\mathcal{a}^c = [L]\setminus \mathcal{a}$. We divide $p^{ab}$ into $L+1$ groups,i.e., $\mathcal{C}^{(\ell)}$ with $\ell=0,\ldots, L$,  according to their set size.  Then for each $k $, consider  we sample $2$ numbers from $m^{(k)}$ with replacement, There are $m^{(k)}$ possibilities if the two numbers are the same and  $m^{(k)}(m^{(k)}-1)$ possibilities  if the two numbers are different. In $\overset{d}{=}$, we use the fact that when $\ell = 0$,  
$$ \E_{ z^{(j)}, j\in \mathcal{a}^c}\tilde{S}(\rho,z^{(1)}, \dots, z^{(L)},u;x) = \E_{ z^{(1)}, \dots, z^{(L)}}\tilde{S}(\rho,z^{(1)}, \dots, z^{(L)},u;x) = 0. $$
\end{proof}

\noindent\textbf{Step 2(iii):}  Secondly, we  bound $$\EE_{z^{(1)},\dots, z^{(L)}}\left\| \tilde{S}(\rho, z^{(1)}, \dots, z^{(L)},u;x)\right\|^2.$$
\begin{proof}
We first prove for all $j\in[L-1]$,
\begin{align}
      &\left\| \tilde{S}(\rho, z^{(1)}, \dots, z^{(L)},u;x)\right\|\label{ttq}\\
  \leq&\left\|u(z^{(L)})f^{(L)}(\rho,z^{(L)};x) - f(\rho, u;x)\right\|+  \sum_{\ell=j+2}^{L}c_o^{L-\ell+1}\times\notag\\
&\Bigg\|u^{(L)}(z^{(L)})\prod_{i=\ell+1}^{L} \left( w(z^{(i)}, z^{(i-1)})\right) \left[ w(z^{(\ell)}, z^{(\ell-1)})f^{(\ell-1)}(\rho,z^{(\ell-1)};x) -g^{(\ell)}(\rho, z^{(\ell)};x)\right]\Bigg\| +c_0^{L-\ell}\times\notag\\
&\Bigg\|u^{(L)}(z^{(L)})\prod_{i=j+2}^{L} \left( w(z^{(i)}, z^{(i-1)})\right) \left[w(z^{(j+1)}, z^{(j)})S^{(j)}(\rho,z^{(1)},\dots, z^{(j)};x) - g^{(j+1)}(\rho, z^{(j+1)};x)\right] \Bigg\|.\notag
\end{align}
When $j=L-1$, we have
\begin{align}
      &\left\| \underline{\tilde{S}(\rho, z^{(1)}, \dots, z^{(L)},u;x)}\right\|\notag\\
  \overset{\eqref{def S}}{=}& \bigg\|u(z^{(L)})f^{(L)}(\rho,z^{(L)};x) - f(\rho, u;x) +u(z^{(L)}) \nabla h^{(L)}\left(g^{(L)}(\rho, z^{(L)};x)\right)\times\notag\\
 & \quad\quad\quad\quad\quad\quad\quad\quad~~~ \left[w(z^{(L)}, z^{(L-1) }) S^{(L-1)}(\rho,z^{(1)}, \dots, z^{(L-1)};x) - g^{(L)}(\rho, z^{(L)};x)\right]  \bigg\|\notag\\
 \leq & \left\|u(z^{(L)})f^{(L)}(\rho,z^{(L)};x) - f(\rho, u;x)\right\|+  \underbrace{\Big| \nabla h^{(L)}\left(g^{(L)}(\rho, z^{(L)};x)\right)\Big|}_{\leq c_0 \text ~~\text{from~ Assum.}~ \ref{ass:2}}\times\notag\\
 & \quad\quad\quad\quad\quad~ \bigg\|u(z^{(L)})\left[w(z^{(L)}, z^{(L-1) }) S^{(L-1)}(\rho,z^{(1)}, \dots, z^{(L-1)};x) - g^{(L)}(\rho, z^{(L)};x)\right]\bigg\|.\notag
\end{align}
We obtain \eqref{ttq} with $j=L-1$. Suppose the statement is true at $j$, we consider $j-1$. In fact,  we have
\begin{align}
    &c_0^{L-j}\Bigg\|u^{(L)}(z^{(L)})\prod_{i=j+2}^{L} \left( w(z^{(i)}, z^{(i-1)})\right) \left[w(z^{(j+1)}, z^{(j)})\underline{S^{(j)}(\rho,z^{(1)}, \dots, z^{(j)};x)} - g^{(j+1)}(\rho, z^{(j+1)};x)\right] \Bigg\|\notag\\
   &\overset{\eqref{def S}}{=}
   c_0^{L-j}\Bigg\|u^{(L)}(z^{(L)})\prod_{i=j+2}^{L} \left( w(z^{(i)}, z^{(i-1)})\right) \left[w(z^{(j+1)}, z^{(j)})f^{(j)}(\rho, z^{(j)};x) - g^{(j+1)}(\rho, z^{(j+1)};x)\right] \notag\\
   &\quad\quad\quad \quad +   u^{(L)}(z^{(L)})\prod_{i=j+1}^{L} \left( w(z^{(i)}, z^{(i-1)})\right)\nabla h^{(j)}\left(g^{(j)}(\rho, z^{(j)};x)\right)\times\notag\\
   &\quad\quad\quad \quad\quad\quad\quad\quad\quad\quad\quad \quad\quad\quad \quad\quad~~~\left[w(z^{(j)}, z^{(j-1)})f^{(j-1)}(\rho, z^{(j-1)};x) - g^{(j)}(\rho, z^{(j)};x)\right] \Bigg\|\notag\\
   &{\leq}
   c_0^{L-j}\Bigg\|u^{(L)}(z^{(L)})\prod_{i=j+2}^{L} \left( w(z^{(i)}, z^{(i-1)})\right) \left[w(z^{(j+1)}, z^{(j)})f^{(j)}(\rho, z^{(j)};x) - g^{(j+1)}(\rho, z^{(j+1)};x)\right]\Bigg\| \notag\\
   &\quad\quad~ +    c_0^{L-j} \underbrace{\left|\nabla h^{(j+1)}\left(g^{(j+1)}(\rho, z^{(j+1)};x)\right)\right|}_{\leq c_0 \text ~~\text{from~ Assum.}~ \ref{ass:2}}\Bigg\|u^{(L)}(z^{(L)})\prod_{i=j+1}^{L} \left( w(z^{(i)}, z^{(i-1)})\right)\times\notag\\
   &\quad\quad\quad \quad\quad\quad\quad\quad\quad\quad\quad \quad\quad\quad~~~\left[w(z^{(j)}, z^{(j-1)})f^{(j-1)}(\rho, z^{(j-1)};x) - g^{(j)}(\rho, z^{(j)};x)\right] \Bigg\|. \label{tt2}
\end{align}
Plugging  \eqref{tt2} into \eqref{ttq} with $\ell$, we can obtain \eqref{ttq} with $\ell-1$.  Thus when $\ell=1$, we have
\begin{align}
  &\left\| \tilde{S}(\rho, z^{(1)}, \dots, z^{(L)},u;x)\right\|\notag\\
\leq&\left\|u(z^{(L)})f^{(L)}(\rho,z^{(L)};x) - f(\rho, u;x)\right\|+  \sum_{\ell=3}^{L}c_o^{L-\ell+1}\notag\\
&\Bigg\|u^{(L)}(z^{(L)})\prod_{i=\ell+1}^{L} \left( w(z^{(i)}, z^{(i-1)})\right) \left[ w(z^{(\ell)}, z^{(\ell-1)})f^{(\ell-1)}(\rho,z^{(\ell-1)};x) -g^{(\ell)}(\rho, z^{(\ell)};x)\right]\Bigg\|\notag\\
&~+c_0^{L-1}\Bigg\|u^{(L)}(z^{(L)})\prod_{i=3}^{L} \left( w(z^{(i)}, z^{(i-1)})\right) \underbrace{\left[w(z^{(2)}, z^{(1)})S^{(1)}(\rho, z^{(1)};x) - g^{(2)}(\rho, z^{(2)};x)\right]}_{S^{(1)}(\rho, z^{(1)};x) = f^{(1)}(\rho, z^{(1)};x)}  \Bigg\|.\notag
\end{align}
It indicates that:
\begin{align}
&\E_{z^{(1)}, \dots, z^{(L)}}\left\| \tilde{S}(\rho, z^{(1)}, \dots, z^{(L)},u;x)\right\|^2\label{eq13}\\
\leq& L\E_{z^{(1)}, \dots, z^{(L)}} \left\|u(z^{(L)})f^{(L)}(\rho,z^{(L)};x) - f(\rho, u;x)\right\|^2+  L\E_{z^{(1)}, \dots, z^{(L)}} \sum_{\ell=2}^{L}c_o^{L-\ell+1}\notag\\
&\quad~~ \Bigg\|u^{(L)}(z^{(L)})\prod_{i=\ell+1}^{L} \left( w(z^{(i)}, z^{(i-1)})\right) \left[ w(z^{(\ell)}, z^{(\ell-1)})f^{(\ell-1)}(\rho,z^{(\ell-1)};x) -g^{(\ell)}(\rho, z^{(\ell)};x)\right]\Bigg\|^2\notag\\
\overset{a}{\leq}& L\E_{z^{(1)}, \dots, z^{(L)}} \left\|u(z^{(L)})f^{(L)}(\rho,z^{(L)};x) - f(\rho, u;x)\right\|^2+  L\sum_{\ell=2}^{L} c_o^{L-\ell+1} \Bigg(\E_{z^{(1)}, \dots, z^{(L)}} \Bigg\|u^{(L)}(z^{(L)})\notag\\
&\prod_{i=\ell+1}^{L} \left( w(z^{(i)}, z^{(i-1)})\right) \left[ w(z^{(\ell)}, z^{(\ell-1)})f^{(\ell-1)}(\rho,z^{(\ell-1)};x) -g^{(\ell)}(\rho, z^{(\ell)};x)\right]\Bigg\|^{2(1+\alpha)^{L}}\Bigg)^{(1+\alpha)^{-L}}\notag\\
\overset{b}{\leq}&L c_{M_1} + L (c_{M}\vee 1) \sum_{\ell=2}^{L-1}c_o^{L-\ell+1},\notag
\end{align}
where $\overset{a}\leq$ uses  Jensen's inequality and  $\overset{b}{\leq}$ uses Assumption \ref{ass:3}, $\|u(z^{(L)}) \|\leq \|u(z^{(L)}) \|\vee 1 $, and $|w(z^{(i)}, z^{(i-1)})|\leq  |w(z^{(i)}, z^{(i-1)})|\vee 1 $ with $i\in \{3, \dots, L \}$. 
\end{proof}

\noindent\textbf{Step 2(iv):} We simplify the terms in the right hand side of \eqref{lin}.
\begin{proof}
For the  terms with $\ell\in\{2,\dots, L\}$ in right hand side   of \eqref{lin}, we have 
\begin{align}
&\sum_{\ell=2}^L \sum_{\mathcal{a}\in \mathcal{C}^{(\ell)}}\left[ \Big(\prod_{k\in \mathcal{a}}\frac{1}{m^{(k)}} \prod_{k\in \mathcal{a}^c}\frac{m^{(k)}-1}{m^{(k)}}\Big) \E_{z^{(i)}, i\in \mathcal{a}}\left\|\E_{ z^{(j)}, j\in \mathcal{a}^c}\tilde{S}(\rho,z^{(1)}, \dots, z^{(L)},u;x)\right\|^2  \right]\label{ell1s}\\
\overset{a}{\leq}&\sum_{\ell=2}^L \sum_{\mathcal{a}\in \mathcal{C}^{(\ell)}}\left[ \Big(\prod_{k\in \mathcal{a}}\frac{1}{m^{(k)}} \prod_{k\in \mathcal{a}^c}\frac{m^{(k)}-1}{m^{(k)}}\Big) \E_{z^{(1)}, \dots, z^{(L)}}\left\|\tilde{S}(\rho,z^{(1)}, \dots, z^{(L)},u;x)\right\|^2  \right]\notag\\
=&\mathcal{O} \left(\frac{1}{m^2}\right),\notag
\end{align}
where $\overset{a}\leq$ uses Jensen's inequality. We analyze the terms with $\ell=1$ in the right hand side of \eqref{lin}.  We have
\begin{align}
&\E_{ z^{(L)} }\left\| \E_{z^{(1)}, \dots, z^{(L-1)} } \tilde{S}(\rho, z^{(1)}, \dots, z^{(L)},u;x)\right\|^2\label{ell3s}\\
  \overset{\eqref{def S} }{=}& E_{ z^{(L)} }\bigg\|  u(z^{(L)})f^{(L)}(\rho,z^{(L)};x) - f(\rho, u;x) +u(z^{(L)}) \nabla h^{(L)}\left(g^{(L)}(\rho, z^{(L)};x)\right)   \E_{z^{(L-1)}}\notag\\
   & \quad\quad\quad~ \left[w(z^{(L)}, z^{(L-1) }) \E_{z^{(1)}, \dots, z^{(L-2)} } \underline{S^{(L-1)}(\rho,z^{(1)}, \dots, z^{(L-1)};x)} - \underline{g^{(L)}(\rho, z^{(L)};x)}\right]  \bigg\|^2\notag\\
    \overset{\eqref{exp z}~\& ~\eqref{eq:nn-layer}}{=}& E_{ z^{(L)} }\left\|  u(z^{(L)})f^{(L)}(\rho,z^{(L)};x) - f(\rho, u;x)\right\|^2.\notag
\end{align}
For all $j\in [L-1]$,  we first use induction to prove for all $\ell\in\{j ,\dots, L-1\}$
\begin{align}
     & \E_{z^{(k)}, k\in[L]\setminus\{j\} } \underline{\tilde{S}(\rho, z^{(1)}, \dots, z^{(L)},u;x)} \label{tuse}\\ =& \E_{z^{(j+1)},\dots, z^{(L)}}\Bigg[u(z^{(L)}) \prod_{k=\ell+2}^{L}\left[\nabla h^{(k)}\left(g^{(k)}(\rho, z^{(k)};x)\right)w(z^{(k)}, z^{(k-1) }) \right]\notag\\ 
     & \quad\quad \nabla h^{(\ell+1)}\left(g^{(\ell+1)}(\rho, z^{(\ell+1)};x)\right)\left(w(z^{(\ell+1)}, z^{(\ell) }) S^{(\ell)}(\rho,z^{(1)}, \dots, z^{(\ell)};x) -g^{(\ell+1)}(\rho, z^{(\ell+1)};x)\right) \Bigg].\notag
\end{align}
When $\ell =L$, we have 
\begin{align}
        &\E_{z^{(k)}, k\in[L]\setminus\{j\} } \tilde{S}(\rho, z^{(1)}, \dots, z^{(L)},u;x) \\
  \overset{\eqref{def S}}{=}&  \underbrace{\E_{z^{(L)}} \left[u(z^{(L)})f^{(L)}(\rho,z^{(L)};x) - f(\rho, u;x)\right]}_{0} +  \E_{z^{(k)}, k\in[L]\setminus\{j\} }\Big[u(z^{(L)})   \notag\\
  & \quad\quad~~~ \nabla h^{(L)}\left(g^{(L)}(\rho, z^{(L)};x)\right)  \left(w(z^{(L)}, z^{(L-1) }) S^{(L-1)}(\rho,z^{(1)}, \dots, z^{(L-1)};x) - g^{(L)}(\rho, z^{(L)};x)\right) \Big] .\notag
\end{align}
So \eqref{tuse} is true when $\ell=L$. Suppose \eqref{tuse} is true for $\ell$ when $\ell\geq j+1$, we consider  $\ell-1$. In fact, plugging \eqref{def S} into \eqref{tuse} with $\ell$, and using the fact that 
\begin{align}
  & \E_{z^{(j+1)},\dots, z^{(L)}}\Bigg[u(z^{(L)}) \prod_{k=\ell+2}^{L}\left[\nabla h^{(k)}\left(g^{(k)}(\rho, z^{(k)};x)\right)w(z^{(k)}, z^{(k-1) }) \right]\notag\\ 
     & \quad\quad\quad\quad \quad\quad\quad\quad \quad \quad\quad\quad\quad \quad\quad~~~ \nabla h^{(\ell+1)}\left(g^{(\ell+1)}(\rho, z^{(\ell+1)};x)\right)w(z^{(\ell+1)}, z^{(\ell) }) f^{(\ell)}(\rho, z^{(\ell)};x)\Bigg]\notag\\
       \overset{\ell \geq j+1}=& \E_{z^{(\ell+1)},\dots, z^{(L)}}\Bigg[u(z^{(L)}) \prod_{k=\ell+2}^{L}\left[\nabla h^{(k)}\left(g^{(k)}(\rho, z^{(k)};x)\right)w(z^{(k)}, z^{(k-1) }) \right]\notag\\ 
     & \quad\quad\quad\quad \quad\quad\quad\quad \quad \quad\quad\quad~~ \nabla h^{(\ell+1)}\left(g^{(\ell+1)}(\rho, z^{(\ell+1)};x)\right)\EE_{z^{(\ell)}}\left(w(z^{(\ell+1)}, z^{(\ell) }) f^{(\ell)}(\rho, z^{(\ell)};x)\right)\Bigg]\notag\\
       =& \E_{z^{(j+1)},\dots, z^{(L)}}\Bigg[u(z^{(L)}) \prod_{k=\ell+2}^{L}\left[\nabla h^{(k)}\left(g^{(k)}(\rho, z^{(k)};x)\right)w(z^{(k)}, z^{(k-1) }) \right]\notag\\ 
     &  \quad\quad\quad\quad \quad\quad\quad\quad \quad \quad\quad\quad\quad \quad\quad\quad\quad\quad \quad\quad~ \nabla h^{(\ell+1)}\left(g^{(\ell+1)}(\rho, z^{(\ell+1)};x)\right)g^{(\ell+1)}(\rho, z^{(\ell+1)};x) \Bigg],\notag
\end{align}
we can obtain \eqref{tuse} with $\ell-1$.  Then using \eqref{tuse} with $\ell=j$, we have
\begin{align}
    &\E_{ z^{(j)} }\left\| \E_{z^{(k)}, k\in[L]\setminus\{j\} } \tilde{S}(\rho, z^{(1)}, \dots, z^{(L)},u;x)\right\|^2 \label{ell2s}\\
     =&\E_{ z^{(j)} } \Bigg\|\E_{z^{(k)}, k\in[L]\setminus\{j\}}\Bigg[u(z^{(L)}) \prod_{k=j+2}^{L}\left[\nabla h^{(k)}\left(g^{(k)}(\rho, z^{(k)};x)\right)w(z^{(k)}, z^{(k-1) }) \right]\notag\\ 
     & \quad~~~ \nabla h^{(j+1)}\left(g^{(j+1)}(\rho, z^{(j+1)};x)\right)\left(w(z^{(j+1)}, z^{(j) }) \underline{S^{(j)}(\rho,z^{(1)}, \dots, z^{(j)};x)} -g^{(j+1)}(\rho, z^{(j+1)};x) \right)\Bigg]\Bigg\|^2\notag\\
     \overset{\eqref{exp z}}=&\E_{ z^{(j)} } \Bigg\|\E_{z^{(j+1)},\dots, z^{(L)}}\Bigg[u(z^{(L)}) \prod_{k=j+2}^{L}\left[\nabla h^{(k)}\left(g^{(k)}(\rho, z^{(k)};x)\right)w(z^{(k)}, z^{(k-1) }) \right]\notag\\ 
     & \quad\quad\quad\quad~~~~~   \nabla h^{(j+1)}\left(g^{(j+1)}(\rho, z^{(j+1)};x)\right)\left(w(z^{(j+1)}, z^{(j) }) f^{(j)}(\rho, z^{(j)};x) -g^{(j+1)}(\rho, z^{(j+1)};x)\right) \Bigg]\Bigg\|^2\notag\\
     =&\E_{z^{(j)}}\left\| \E_{z^{(j+1)}}  \frac{\partial f(\rho,u;x)}{\partial z^{(j+1)}} \left[f^{(j)}(\rho,z^{(j)};x) w(z^{(j+1)}, z^{(j)})-f^{(j+1)}(\rho,z^{(j+1)};x)  \right] \right\|^2.\notag
\end{align}

Finally, plugging \eqref{ell1s}, \eqref{ell3s}, and \eqref{ell2s} into \eqref{lin}, we have
\begin{align}
    &\E \left\| \breve{f}(\rho,u;x) - f(\rho,u;x) \right\|^2\label{theo2:end1}\\
=&\sum_{\ell =1}^{L-1} \left[\prod_{j=1,j\neq\ell}^{L}\left(1-\frac{1}{m^{(j)}}\right)\right]\times\notag\\
&\quad\quad\quad\quad~~\frac{1}{m^{(\ell)}}  \E_{z^{(\ell)}}\left|\E_{z^{(\ell+1)}}\frac{\partial f(\rho,u;x)}{\partial z^{(\ell+1)}}  \left[  f^{(\ell)}(\rho,z^{(\ell)};x)w(z^{(\ell+1)}, z^{(\ell)})- g^{(\ell+1)}(\rho,z^{(\ell+1)};x) \right] \right|^2 \notag\\
& + \frac{1}{m^{(L)}} \left[\prod_{j=1}^{L-1}\left(1-\frac{1}{m^{(j)}}\right)\right]\E_{z^{(L)}} \left\| f^{(L)}(\rho,z^{(L)};x) u(z^{(L)})-f(\rho, u;x)\right\|^2+ \mathcal{O}\left( \sum_{\ell=1}^{L} (m^{(\ell)})^{-2} \right)\notag\\
=&\sum_{\ell =1}^{L-1} \frac{1}{m^{(\ell)}}  \E_{z^{(\ell)}}\left|\E_{z^{(\ell+1)}}\frac{\partial f(\rho,u;x)}{\partial z^{(\ell+1)}}  \left[  f^{(\ell)}(\rho,z^{(\ell)};x)w(z^{(\ell+1)}, z^{(\ell)})- g^{(\ell+1)}(\rho,z^{(\ell+1)};x) \right] \right|^2 \notag\\
& + \frac{1}{m^{(L)}}\E_{z^{(L)}} \left\| f^{(L)}(\rho,z^{(L)};x) u(z^{(L)})-f(\rho, u;x)\right\|^2+ \mathcal{O}\left( \sum_{\ell=1}^{L} (m^{(\ell)})^{-2} \right).\notag
\end{align}
\end{proof}

\noindent\textbf{Step 3.} The remaining to do is to bound
$$\E\left\| \breve{f}(\rho,u;x)-\hf(x)\right\|^2,    $$
which is carried out in this step.

Before that, we denote 
 $$
\bbu(z^{(L)}):= 
\begin{cases}
\|u(z^{(L)})\|,&   \|u(z^{(L)}) \|\geq 1,\\
1,&   \text{otherwise}, \\
\end{cases}
$$
for all $z^{(L)}\in \mathcal{Z}^{(L)}$, and 
 $$
\bbw(z^{(\ell)}, z^{(\ell-1)}):= 
\begin{cases}
|w(z^{(\ell)}, z^{(\ell-1)})|,& \mbox{ if }   |w(z^{(\ell)}, z^{(\ell-1)})|\geq 1,\\
1,&   \text{ otherwise}, \\
\end{cases},
$$
for all $z^{(\ell)}\in \mathcal{Z}^{(\ell)}$ and  $\ell \in [L]$. We  denote $\cA^{(\ell)}$ to be
\begin{align} \cA^{(\ell)} := [m^{(\ell)}] \otimes [m^{(\ell+1)}] \otimes \cdots \otimes [m^{(L)}].\nonumber
\end{align}
 Then, for $\ell \in \{2, \cdots, L \}$ and  $\ts\in \cA^{(\ell-1)}$,  we define
 \begin{align}
 \Psi^{(\ell)}(\ts) := \bbu(z^{(L)}_{\ts_L})  \left(\prod_{i = \ell+1}^{L} \bbw(z^{(i)}_{\ts_i}, z^{(i-1)}_{\ts_{i-1}})\right),\nonumber
 \end{align}
 and when $\ell=L+1$, we let
 \begin{align}
     \Psi^{(L+1)}(\ts) := \bbu(z^{(L)}_{\ts_L}), ~~ \ts\in \cA^{(L)},\nonumber
 \end{align}
 where for the sake of simplicity, $\ts_{i}$ is the value of the $(i-\ell+1)$-th dimension of $\ts$ with $i\in\{\ell,\dots,L\}$.\\\\
 \textbf{Step 3(i):} 
For  $\ell\in [L]$ and $t\in \{0, \dots, L\}$,  by setting $\beta_t= (1+\alpha)^{t}\geq 1$, we first  bound
$$
\EE\sum_{\ts\in \cA^{(\ell)}}  \Big(\prod_{i=\ell}^L \frac{1}{m^{(i)}}\Big) \left| 
\Psi^{(\ell+1)}(\ts)\frac{1}{m^{(\ell-1)}}\sum_{j=1}^{m^{(\ell-1)}}\left[ w(\zl_{\ts_{\ell}},\zmi_{j} )
\breve{f}^{(\ell-1)}(\rho,\zmi_{j};x) -  g^{(\ell)}(\rho,\zl_{\ts_{\ell}};x)\right] \right|^{2\beta_t}.
$$

\begin{proof}
When $\ell=1$, \eqref{l use} equals to $0$.  When $\ell\in \{2,\dots, L\}$, from  \eqref{ffo0} and \eqref{ffo1}, we have
\begin{align}
&\EE\sum_{\ts\in \cA^{(\ell)}}  \Big(\prod_{i=\ell}^L \frac{1}{m^{(i)}}\Big) \left| 
\Psi^{(\ell+1)}(\ts)\frac{1}{m^{(\ell-1)}}\sum_{j=1}^{m^{(\ell-1)}}\left[ w(\zl_{\ts_{\ell}},\zmi_{j} )
\underbrace{\breve{f}^{(\ell-1)}(\rho,\zmi_{j};x)}_{\eqref{ffo0}~\&~\eqref{ffo1}} -  g^{(\ell)}(\rho,\zl_{\ts_{\ell}};x)\right]     \right|^{2\beta_t}\notag\\
=&
\prod_{i=\ell}^L \frac{1}{m^{(i)}}\underline{\EE}\sum_{\ts\in \cA^{(\ell)}} \Bigg|  \underline{\prod^{\ell-1}_{j=1}\frac{1}{m^{(j)}}}\Psi^{(\ell+1)}(\ts)\sum_{\hs\in \cI^{(\ell-2)}}\sum_{j=1}^{m^{(\ell-1)}}\notag\\
&\quad\quad\quad\quad\quad\quad\quad\quad\quad\quad\quad~~~
\bigg[w(z^{(\ell)}_{\ts_{\ell}}, z^{(\ell-1)}_j)S^{(\ell-1)}(\rho,z^{(1)}_{\hs_1}, \dots, z^{(\ell-2)}_{\hs_\ell-2}, \zmi_{j};x) - g^{(\ell)}(\rho,\zl_{\ts_{\ell}};x)\bigg] \Bigg|^{2\beta_t}\notag\\
\overset{a}=&
\prod^{\ell-1}_{j=1}\frac{1}{(m^{(j)})^{2\beta_t}} 
\EE_{z^{(\ell)}, \dots, z^{(L)}}\E_{\{z\}_{\cI^{(\ell-1)}}}\Bigg|  \bbu(z^{(L)}) \prod_{i = \ell+1}^{L} \bbw(z^{(i)}, z^{(i-1)})\sum_{\hs\in \cI^{(\ell-1)}} \notag\\
&\quad\quad\quad\quad\quad\quad\quad\quad\quad\quad\quad\quad\quad\quad~~~\bigg[w(z^{(\ell)}, z^{(\ell-1)}_{\hs_{\ell-1}})S^{(\ell-1)}(\rho,z^{(1)}_{\hs_1}, \dots, z^{(\ell-1)}_{\hs_{\ell-1}};x) - g^{(\ell)}(\rho,\zl;x)\bigg] \Bigg|^{2\beta_t}\notag\\
\overset{b}{=}&\prod^{\ell-1}_{j=1}\frac{1}{(m^{(j)})^{2\beta_t}}\EE_{z^{(\ell)}, \dots, z^{(L)}} \E_{\{z\}_{\cI^{(\ell-1)}}}\left|\sum_{\hs\in \cI^{(\ell-1)}}  \left[\bar{S}^{(\ell)}(\rho,z^{(1)}_{\hs_1}, \dots, z^{(\ell-1)}_{\hs_{\ell-1}},u;x)\right] \right|^{2\beta_t},\label{ttle}
\end{align}
where in $\overset{a}=$, we take expectation on $\{z\}_{\cA^{(\ell)}}$ and in $\overset{b}=$, we denote
$$ \!\!\bar{S}^{(\ell)}(\rho,z^{(1)}_{\hs_1}, \dots, z^{(\ell-1)}_{_{\hs_{\ell-1}}};x) :=  \bbu(z^{(L)}) \!\prod_{i = \ell+1}^{L} \bbw(z^{(i)}, z^{(i-1)})\!\left[w(z^{(\ell)}, z^{(\ell-1)}_{\hs_{\ell-1}})S^{(\ell-1)}(z^{(1)}_{\hs_1}, \dots, z^{(\ell-1)}_{\hs_{\ell-1}}) - g^{(\ell)}(\zl)\right]. $$
 We then set $$J^{(t)} :=\argmin_{t\in[\ell-1]} m^{(t)}$$
and $$\mathcal{B}^{(\ell-1)} := [m^{(1)} ] \otimes \cdots \otimes [m^{(J^{(\ell)}-1)}]\otimes [m^{(J^{(\ell)}+1)}]\cdots\otimes [m^{(\ell-1)}].$$
We consider dividing $  
\sum_{\hs\in \cI^{(\ell-1)}} \bar{S}^{(\ell)}(\rho,z^{(1)}_{\hs_1}, \dots, z^{(\ell-1)}_{\hs_{\ell-1}},u;x)
$ into $|\mathcal{B}^{(\ell-1)}|$  groups. We require    all the terms in each group are independent with each other  given $\{z^{(\ell)}, \dots, z^{(L)}\}$. In fact, for  any $\dot{s}\in \mathcal{B}^{(\ell-1)}$ and $k\in [J^{(\ell)}]$, we denote the ordered list $\dot{s}^{+k}$ as
 $$
\dot{s}^{+k}_i:= 
\begin{cases}
(\dot{s}_{i}+k)\%m^{(i)}+1,&   i\in [J^{(\ell)-1}],\\
k,&   i= J^{(\ell)} , \\
(\dot{s}_{i}+k)\%m^{(i)}+1,&   i\in [L]\setminus[J^{(\ell)}],
\end{cases}
$$
then we can divide $  
\sum_{\hs\in \cI^{(\ell-1)}} \bar{S}^{(\ell)}(\rho,z^{(1)}_{\hs_1}, \dots, z^{(\ell-1)}_{\hs_{\ell-1}},u;x)
$ into $|\mathcal{B}^{(\ell-1)}|$ groups as
$$\sum_{\hs\in \cI^{(\ell-1)}} \bar{S}^{(\ell)}(\rho,z^{(1)}_{\hs_1}, \dots, z^{(\ell-1)}_{\hs_{\ell-1}},u;x) :=  \sum_{\dot{s}\in\mathcal{B}^{(\ell-1)}}\sum_{k=1}^{m^{(J^{(\ell)})}}  \left[\bar{S}^{(\ell)}(\rho,z^{(1)}_{\dot{s}^{+k}_1}, \dots,   z^{(\ell-1)}_{\dot{s}^{+k}_{\ell-1}},u;x)\right].  $$
Therefor from \eqref{ttle} we further have 
\begin{align}
&\EE\sum_{\ts\in \cA^{(\ell)}}  \Big(\prod_{i=\ell}^L \frac{1}{m^{(i)}}\Big) \left| 
\Psi^{(\ell+1)}(\ts)\frac{1}{m^{(\ell-1)}}\sum_{j=1}^{m^{(\ell-1)}}\left[ w(\zl_{\ts_{\ell}},\zmi_{j} )
\breve{f}^{(\ell-1)}(\rho,\zmi_{j};x) -  g^{(\ell)}(\rho,\zl_{\ts_{\ell}};x)\right]     \right|^{2\beta_t}\notag\\
=&\EE_{z^{(\ell)}, \dots, z^{(L)}}\E_{\{z\}_{\cI^{(\ell-1)}}}\left| \underline{\prod^{\ell-1}_{j=1}\frac{1}{(m^{(j)})}\sum_{\dot{s}\in\mathcal{B}^{(\ell-1)}}}\sum_{k=1}^{m^{(J^{(\ell)})}}  \left[\bar{S}^{(\ell)}(\rho,z^{(1)}_{\dot{s}^{+k}_1}, \dots,   z^{(\ell-1)}_{\dot{s}^{+k}_{\ell-1}},u;x)\right] \right|^{2\beta_t}\notag\\
\overset{a}{\leq}&  \EE_{z^{(\ell)}, \dots, z^{(L)}} \E_{\{z\}_{\cI^{(\ell-1)}}}  \prod^{J^{(\ell)}-1}_{j=1}\frac{1}{(m^{(j)})}\prod^{\ell-1}_{j=J^{(\ell)}+1}\frac{1}{(m^{(j)})} \sum_{\dot{s}\in\mathcal{B}^{(\ell-1)}}\notag\\
&\quad\quad\quad\quad\quad\quad\quad\quad\quad\quad\quad\quad\quad\quad\quad\quad\quad\quad\quad\quad~~~\left| \frac{1}{(m^{(J^{(\ell)})})}\sum_{k=1}^{m^{(J^{(\ell)})}}  \left[\bar{S}^{(\ell)}(\rho,z^{(1)}_{\dot{s}^{+k}_1}, \dots,   z^{(\ell-1)}_{\dot{s}^{+k}_{\ell-1}},u;x)\right] \right|^{2\beta_t} \notag\\
\overset{b}{=}&    \EE_{z^{(\ell)}, \dots, z^{(L)}} \E_{\{z\}_{s^{+k}, k\in \left[m^{(J^{(\ell)})}\right]}}  \left| \frac{1}{(m^{(J^{(\ell)})})}\sum_{k=1}^{m^{(J^{(\ell)})}}  \left[\bar{S}^{(\ell)}(\rho,z^{(1)}_{s^{+k}_1}, \dots,   z^{(\ell-1)}_{s^{+k}_{\ell-1}},u;x)\right] \right|^{2\beta_t}
\notag\\
\overset{c}{\leq}&\frac{1}{(m^{(J^{(\ell)})})^{2\beta_t}}\EE_{z^{(\ell)}, \dots, z^{(L)}} C^1_{\beta_t}\left[ m^{(J^{(\ell)})}\E_{z^{(1)}, \dots, z^{(\ell-1)}}\left| \bar{S}^{(\ell)}(\rho,z^{(1)}, \dots, z^{(\ell-1)},u;x) \right|^{2\beta_t}  \right]\notag\\
&\quad\quad\quad\quad\quad\quad\quad~+ \frac{1}{(m^{(J^{(\ell)})})^{2\beta_t}}\EE_{z^{(\ell)}, \dots, z^{(L)}} C^2_{\beta_t}\left[ m^{(J^{(\ell)})}\E_{z^{(1)}, \dots, z^{(\ell-1)}}\left| \bar{S}^{(\ell)}(\rho,z^{(1)}, \dots, z^{(\ell-1)},u;x) \right|^2  \right]^{\beta_t}\notag\\
\overset{d}{\leq}&  C^0_{\beta_t}\frac{1}{(m^{(J^{(\ell)})})^{\beta_t}}\E_{z^{(1)}, \dots, z^{(L)}}\left|\bar{S}^{(\ell)}(\rho,z^{(1)}, \dots, z^{\ell-1)},u;x)\right|^{2\beta_t}\notag\\
\overset{e}{\leq}& C^0_{\beta_t} (c_{M}\vee 1)\frac{1}{(m^{(J^{(\ell)})})^{\beta_t}}. \label{l use} 
\end{align}
We explain \eqref{l use}. 
In $\overset{a}\leq$m we use Jensen's inequality, 
$$ \left|\frac{1}{m}\sum_{i=1}^m x_i \right|^p \leq  \frac{1}{m}\sum_{i=1}^m \left|x_i\right|^p, ~~ p\geq 1.   $$
In $\overset{b}=$, we take expectation on all $\{z\}_{\mathcal{B}^{(\ell-1)}}$, i.e. $$\hat{\mathcal{Z}}^{(1)}\cup \cdots \cup \hat{\mathcal{Z}}^{(J^{(\ell)-1})}\cup\hat{\mathcal{Z}}^{(J^{(\ell)+1})}\cup \cdots \cup \hat{\mathcal{Z}}^{(\ell-1)},$$ and denote the ordered list $s^{+k}$ as
 $$
s^{+k}:= 
\begin{cases}
(1+k)\%m^{(i)}+1,&   i\in [J^{(\ell)-1}],\\
k,&   i= J^{(\ell)} , \\
(1+k)\%m^{(i)}+1,&   i\in [L]\setminus[J^{(\ell)}].
\end{cases}
$$
In $\overset{c}=$, we use
all $\bar{S}^{(\ell)}(\rho,z^{(1)}_{s^{+k}_1}, \dots, z^{(L)}_{s^{+k}_L};x)$ with $k\in m^{(J^{(\ell)})}$are independent  given $\{z^{(\ell)}, \dots, z^{(L)}\}$ (all the indexes are different) and is obtained by Rosenthal Inequality (refer to Lemma \ref{rosenthal}).  In $\overset{d}=$, we use Jensen's inequality,   given $\{z^{(\ell)}, \dots, z^{(L)}\}$,  we have
\begin{align}
\left[\E_{z^{(1)}, \dots, z^{(\ell-1)}}\left| \bar{S}^{(\ell)}(\rho,z^{(1)}, \dots, z^{(\ell-1)};x)   \right|^2\right]^{\beta_t} \leq  \E_{z^{(1)}, \dots, z^{(\ell-1)}}\left| \bar{S}^{(\ell)}(\rho,z^{(1)}, \dots, z^{(\ell-1)};x)\right|^{2\beta_t},\nonumber
\end{align}
and 
\begin{align}
    \frac{1}{(m^{(J^{(\ell)})})^{2\beta_t-1}} \leq \frac{1}{(m^{(J^{(\ell)}})^{\beta_t}},\nonumber
\end{align}
since $\beta_t\geq1$. The inequality $\overset{e}\leq$ comes from Assumption \ref{ass:3}, we have
\begin{align}
    &\E_{z^{(1)}, \dots, z^{(L)}}\left|\bar{S}^{(\ell)}(\rho,z^{(1)}, \dots, z^{\ell-1)},u;x)\right|^{2\beta_t}\notag\\ \overset{t\leq L}{\leq}& \left(\E_{z^{(1)}, \dots, z^{(L)}}\left|\bar{S}^{(\ell)}(\rho,z^{(1)}, \dots, z^{\ell-1)},u;x)\right|^{2\beta_L}\right)^{\frac{\beta_t}{\beta_L}}  \notag\\
    =&(c_{M})^{\frac{\beta_t}{\beta_L}}\notag\\
    \leq& c_{M}\vee 1. \notag
\end{align}
\end{proof}
\noindent\textbf{Step 3(ii):} Denote 
$$\tPsi_s(z^{(\ell)}) := \Psi^{(\ell+2)}(s)\bbw(z^{\ell+1}_{s_{\ell+1}}, \zl),  ~~ \ell \in \{2, \dots, L-1 \},~ s\in \cA^{(\ell+1)}.$$ We let $\cA^{(L+1)} = \{1\}$ and denote
$$\tPsi_s(z^{(L)}) := \bbu(z^{(L)}),~~ s\in \cA^{(L+1)}.$$ Secondly, for all $\ell\in \{2, \dots, L\}$ and $s\in \cA^{(\ell+1)}$, we bound
$$ \left|\frac{1}{m^{(\ell)}} \sum^{m^{(\ell)}}_{j=1}\tPsi_s(\zl_j)\left|\hf_j^{(\ell)}(x) - \breve{f}^{(\ell)}(\rho,\zl_j;x)\right|\right|^{2\beta_t} $$
by the continuity of $\nabla h^{(\ell)}(\cdot)$.
\begin{proof}
 For any $\ell\in \{2, \dots, L\}$ and $j\in [m^{(\ell)}]$, from Lemma \ref{lip} in the end of the section, we have
\begin{align}
&\left|\hf_j^{(\ell)}(x)  -  \hl\left(\gl(\rho,\zl_j;x)\right) - \nabla  \hl\left(\gl(\rho,\zl_j;x)\right)\left[\hat{g}^{(\ell)}_j(x) -  \gl(\rho,\zl_j;x) \right]    \right|\notag\\
& \quad\quad\quad\quad\quad\quad\quad\quad\quad\quad\quad~ \leq c_1\left| \frac{1}{m^{(\ell-1)}}\sum_{i=1}^{m^{(\ell-1)}} w(\zl_j, \zmi_i)\hf_i^{(\ell-1)}(x) -  \gl(\rho,\zl_j;x)     \right|^{1+\alpha}.\label{tmt}
\end{align}

So for all $\ell \in \{2, \dots, L\}$,
\begin{align}
&\left|\frac{1}{m^{(\ell)}} \sum^{m^{(\ell)}}_{j=1}\tPsi_s(\zl_j)\left|\hf^{(\ell)}_j(x) - \breve{f}^{(\ell)}(\rho,\zl_j;x)\right|\right|\label{onet}\\
\leq&\frac{1}{m^{(\ell)}}\sum^{m^{(\ell)}}_{j=1}\left|\tPsi_s(\zl_j)\right|\left| \hf^{(\ell)}_j(x) - \breve{f}^{(\ell)}(\rho,\zl_j;x)\right|\notag\notag\\
\leq& \frac{1}{m^{(\ell)}}\sum^{m^{(\ell)}}_{j=1}\left|\tPsi_s(\zl_j)\right|\times\notag\\
& \quad\quad\quad\underbrace{\left|\hf_j^{(\ell)}(x)  -  \hl(\gl(\rho,\zl_j;x)) - \nabla  \hl\left(\gl(\rho,\zl_j;x)\right)\left[  \hat{g}^{(\ell)}_j(x)-  \gl(\rho,\zl_j;x) \right]    \right|}_{\eqref{tmt}}\notag\\
&+ \frac{1}{m^{(\ell)}}\sum^{m^{(\ell)}}_{j=1}\left|\tPsi_s(\zl_j)\right|\times\notag \\
&\left|\underbrace{\breve{f}^{(\ell)}(\rho,\zl_j;x)}_{\eqref{def bf}}  -  \hl(\gl(\rho,\zl_j;x)) - \nabla  \hl\left(\gl(\rho,\zl_j;x)\right)\left[ \underbrace{\hat{g}^{(\ell)}_j(x)}_{\eqref{eqn:hatf-l-j}} -  \gl(\rho,\zl_j;x) \right]   \right|\notag\\
\leq&\frac{c_1}{m^{(\ell)}}\sum^{m^{(\ell)}}_{j=1}\left|\tPsi_s(\zl_j)\right|\left| \frac{1}{m^{(\ell-1)}}\sum_{i=1}^{m^{(\ell-1)}} w(\zl_j, \zmi_i)\hf^{(\ell-1)}_i(x) -  \gl(\rho,\zl_j;x)     \right|^{1+\alpha}\notag\\
&+\frac{1}{m^{(\ell)}}\sum^{m^{(\ell)}}_{j=1}\left|\tPsi_s(\zl_j)\right|\times \notag\\
&\quad~\left| \nabla  \hl\left(\gl(\rho,\zl_j;x)\right)  \frac{1}{m^{(\ell-1)}}\sum_{i=1}^{m^{(\ell-1)}} w(\zl_j, \zmi_i)\Big[\hf^{(\ell-1)}_i(x) -  \breve{f}^{(\ell-1)}(\rho,\zmi_i;x)\Big]   \right|\notag\\
\overset{a}{\leq}&\frac{c_1}{m^{(\ell)}}\sum^{m^{(\ell)}}_{j=1}\left|\tPsi_s(\zl_j) \frac{1}{m^{(\ell-1)}}\sum_{i=1}^{m^{(\ell-1)}}\left[ w(\zl_j, \zmi_i)\hf_i^{(\ell-1)}(x) -  \gl(\rho,\zl_j;x) \right]    \right|^{1+\alpha}\notag\\
&\quad~~~~ +\frac{c_0}{m^{(\ell)}}\sum^{m^{(\ell)}}_{j=1}\left|\tPsi_s(\zl_j) \frac{1}{m^{(\ell-1)}}\sum_{i=1}^{m^{(\ell-1)}} \underline{w(\zl_j, \zmi_i)\left[\hf^{(\ell-1)}_i(x) -  \breve{f}^{(\ell-1)}(\rho,\zmi_i;x)\right]}  \right|\notag \\
\leq&\frac{c_1}{m^{(\ell)}}\sum^{m^{(\ell)}}_{j=1}\left|\tPsi_s(\zl_j) \frac{1}{m^{(\ell-1)}}\sum_{i=1}^{m^{(\ell-1)}}\left[ w(\zl_j, \zmi_i)\hf_i^{(\ell-1)}(x) -  \gl(\rho,\zl_j;x) \right]    \right|^{1+\alpha}\notag\\
&\quad~~~~~~ +\frac{c_0}{m^{(\ell)}}\sum^{m^{(\ell)}}_{j=1}\left|\tPsi_s(\zl_j) \frac{1}{m^{(\ell-1)}}\sum_{i=1}^{m^{(\ell-1)}} \bbw(\zl_j, \zmi_i)\left|\hf^{(\ell-1)}_i(x) -  \breve{f}^{(\ell-1)}(\rho,\zmi_i;x)\right|  \right|, \notag
\end{align}
where   $\overset{a}\leq$ uses Assumption \ref{ass:2} and $\| \tPsi_s(\zl_j)\|\geq 1$.
Furthermore, for the first term on the right hand side of \eqref{onet},  we have
\begin{align}
&\left|\tPsi_s(\zl_j) \frac{1}{m^{(\ell-1)}}\sum_{i=1}^{m^{(\ell-1)}} \left[w(\zl_j, \zmi_i)\hf_i^{(\ell-1)}(x) -  \gl(\rho,\zl_j;x) \right]    \right|\notag\\
\leq&  \left| \tPsi_s(\zl_j)  \frac{1}{m^{(\ell-1)}}\sum_{i=1}^{m^{(\ell-1)}} \underline{w(\zl_j, \zmi_i)\left[\hf_i^{(\ell-1)}(x) - \breve{f}^{(\ell-1)}(\rho,\zmi_i;x) \right] }   \right|\notag\\
&\quad\quad\quad\quad~ +\left| \tPsi_s(\zl_j) \frac{1}{m^{(\ell-1)}}\sum_{i=1}^{m^{(\ell-1)}} \left[w(\zl_j, \zmi_i)\breve{f}^{(\ell-1)}(\rho,\zmi_i;x) -  \gl(\rho,\zl_j;x)   \right] \right|\notag\\
\leq&  \left| \tPsi_s(\zl_j)  \frac{1}{m^{(\ell-1)}}\sum_{i=1}^{m^{(\ell-1)}} \bbw(\zl_j, \zmi_i)\left|\hf_i^{(\ell-1)}(x) - \breve{f}^{(\ell-1)}(\rho,\zmi_i;x) \right|    \right|\notag\\
&\quad\quad\quad\quad~ +\left| \tPsi_s(\zl_j) \frac{1}{m^{(\ell-1)}}\sum_{i=1}^{m^{(\ell-1)}} \left[w(\zl_j, \zmi_i)\breve{f}^{(\ell-1)}(\rho,\zmi_i;x) -  \gl(\rho,\zl_j;x)   \right] \right|.\notag
\end{align}
From Lemma \ref{jcon} in the end of the section, we have
\begin{align}
&\left|\tPsi_s(\zl_j)  \frac{1}{m^{(\ell-1)}}\sum_{i=1}^{m^{(\ell-1)}} \left[w(\zl_j, \zmi_i)\hf_i^{(\ell-1)}(x) -  \gl(\rho,\zl_j;x) \right]    \right|^{1+\alpha}\label{onea}\\
\leq&  ~(1-q_0)^{-\alpha}\left| \tPsi_s(\zl_j) \frac{1}{m^{(\ell-1)}}\sum_{i=1}^{m^{(\ell-1)}} \bbw(\zl_j, \zmi_i)\Big|\hf_i^{(\ell-1)}(x) - \breve{f}^{(\ell-1)}(\rho,\zmi_i;x) \Big|    \right|^{1+\alpha}\notag\\
&+q_0^{-\alpha}\left| \tPsi_s(\zl_j)  \frac{1}{m^{(\ell-1)}}\sum_{i=1}^{m^{(\ell-1)}}  \left[w(\zl_j, \zmi_i)\breve{f}^{(\ell-1)}(\rho,\zmi_i;x) -  \gl(\rho,\zl_j;x)\right]     \right|^{1+\alpha},\notag
\end{align}
where we simply set  $q_0 = \frac{3}{4}$. 

Plugging \eqref{onea} into \eqref{onet}, we have  for all $\ell\in \{2, \dots, L\}$ and $s\in \cA^{(\ell+1)}$, 
\begin{align}
&\left|\frac{1}{m^{(\ell)}} \sum^{m^{(\ell)}}_{j=1}\tPsi_s(\zl_j)\left|\hf_j^{(\ell)}(x) - \breve{f}^{(\ell)}(\rho,\zl_j;x)\right|\right|\notag\\
\leq&c_1(1-q_0)^{-\alpha}\frac{1}{m^{(\ell)}}\sum_{j=1}^{m^{(\ell)}}\left| \tPsi_s(\zl_j)  \frac{1}{m^{(\ell-1)}}\sum_{i=1}^{m^{(\ell-1)}} \bbw(\zl_j, \zmi_i)\left|\hf_i^{(\ell-1)}(x) - \breve{f}^{(\ell-1)}(\rho,\zmi_i;x) \right|    \right|^{1+\alpha}\notag\\
&+c_1 q_0^{-\alpha}\frac{1}{m^{(\ell)}}\sum_{j=1}^{m^{(\ell)}}\left| \tPsi_s(\zl_j)  \frac{1}{m^{(\ell-1)}}\sum_{i=1}^{m^{(\ell-1)}} \left[w(\zl_j, \zmi_i)\breve{f}^{(\ell-1)}(\rho,\zmi_i;x) -  \gl(\rho,\zl_j;x)\right]     \right|^{1+\alpha}\notag\\
&\quad\quad\quad\quad~ + \frac{c_0}{m^{(\ell)}}\sum_{j=1}^{m^{(\ell)}}\left|\tPsi_s(\zl_j)   \frac{1}{m^{(\ell-1)}}\sum_{i=1}^{m^{(\ell-1)}} \bbw(\zl_{j}, \zmi_i)\left|\hf_i^{(\ell-1)}(x) -  \breve{f}^{(\ell-1)}(\rho,\zmi_i;x)\right|   \right|.\notag
\end{align}
Finally using  Lemma \ref{jcon} again, for all $\ell\in \{2, \dots, L\}$ and $s\in \cA^{(\ell+1)}$,  we obtain
\begin{align}
&\left|\frac{1}{m^{(\ell)}} \sum^{m^{(\ell)}}_{j=1}\tPsi_s(\zl_j)\left|\hf_j^{(\ell)}(x) - \breve{f}^{(\ell)}(\rho,\zl_j;x)\right|\right|^{2\beta_t}\label{ttend}\\
\leq&c_1^{2\beta_t}q_1^{-2\alpha\beta_t-2\beta_t +1}\times\notag\\
&\quad~~ \left(\underline{\frac{1}{m^{(\ell)}}\sum_{j=1}^{m^{(\ell)}}}\left| \tPsi_s(\zl_j)  \frac{1}{m^{(\ell-1)}}\sum_{i=1}^{m^{(\ell-1)}} \bbw(\zl_j, \zmi_i)\left|\hf_i^{(\ell-1)}(x) - \breve{f}^{(\ell-1)}(\rho,\zmi_i;x) \right|    \right|^{1+\alpha}\right)^{2\beta_t}\notag\\
&+c_1^{2\beta_t}q_0^{-2\alpha\beta_t}(1-q_1)^{-2\beta_t +1}\times\notag\\
&\left(\underline{\frac{1}{m^{(\ell)}}\sum_{j=1}^{m^{(\ell)}}}\left| \tPsi_s(\zl_j)  \frac{1}{m^{(\ell-1)}}\sum_{i=1}^{m^{(\ell-1)}} \left[w(\zl_j, \zmi_i)\breve{f}^{(\ell-1)}(\rho,\zmi_i;x) -  \gl(\rho,\zl_j;x)  \right]   \right|^{1+\alpha}\right)^{2\beta_t}\notag\\
&+ c_0^{2\beta_t}q_1^{-2\beta_t +1}\times\notag\\
&\quad\quad\quad~~\left(\underline{\frac{1}{m^{(\ell)}}\sum_{j=1}^{m^{(\ell)}}}\left|\tPsi_s(\zl_j) \frac{1}{m^{(\ell-1)}}\sum_{i=1}^{m^{(\ell-1)}} \bbw(\zl_{j}, \zmi_i)\Big|\hf_i^{(\ell-1)}(x) -  \breve{f}^{(\ell-1)}(\rho,\zmi_i;x)\Big|  \right|\right)^{2\beta_t}\notag\\
\overset{a}{\leq}&c_1^{2\beta_t}q_1^{-2\alpha\beta_t-2\beta_t +1}\times\notag\\
&\quad\quad\quad\quad~~~\frac{1}{m^{(\ell)}}\sum_{j=1}^{m^{(\ell)}}\left| \tPsi_s(\zl_j)  \frac{1}{m^{(\ell-1)}}\sum_{i=1}^{m^{(\ell-1)}} \bbw(\zl_j, \zmi_i)\left|\hf_i^{(\ell-1)}(x) - \breve{f}^{(\ell-1)}(\rho,\zmi_i;x) \right|    \right|^{2\beta_{t+1}}\notag\\
&+c_1^{2\beta_t}q_0^{-2\alpha\beta_t}(1-q_1)^{-2\beta_t +1}\times\notag\\
&\quad\quad\quad\frac{1}{m^{(\ell)}}\sum_{j=1}^{m^{(\ell)}}\left| \tPsi_s(\zl_j)  \frac{1}{m^{(\ell-1)}}\sum_{i=1}^{m^{(\ell-1)}} \left[w(\zl_j, \zmi_i)\breve{f}^{(\ell-1)}(\rho,\zmi_i;x) -  \gl(\rho,\zl_j;x) \right]    \right|^{2\beta_{t+1}}\notag\\
&+  c_0^{2\beta_t}q_1^{-2\beta_t +1}\times\notag\\
&\quad\quad\quad\quad\quad~ \frac{1}{m^{(\ell)}}\sum_{j=1}^{m^{(\ell)}}\left| \tPsi_s(\zl_j)   \frac{1}{m^{(\ell-1)}}\sum_{i=1}^{m^{(\ell-1)}} \bbw(\zl_{j}, \zmi_i)\Big|\hf_i^{(\ell-1)}(x) -  \breve{f}^{(\ell-1)}(\rho,\zmi_i;x)\Big|   \right|^{2\beta_t},\notag
\end{align}
where we simply set $q_1 = \frac{1}{4}$ and  $\overset{a}\leq$ uses Jensen's inequality which is
$$ \left|\frac{1}{m}\sum_{i=1}^m x_i \right|^p \leq  \frac{1}{m}\sum_{i=1}^m \left|x_i\right|^p, ~~ p\geq 1.   $$
\end{proof}
\noindent\textbf{Step 3(iii):} Now we are ready to bound
$$\E\left\| \breve{f}(\rho,u;x)-\hf(x)\right\|^2.    $$
 \begin{proof}
 Denote:
\begin{align*}
&\Gamma(t, \ell) := \EE \prod_{i=\ell+1}^L\frac{1}{m^{(i)}}\sum_{\ts\in \cA^{(\ell+1)}}\left|\frac{1}{m^{(\ell)}}\sum^{m^{(\ell)}}_{j=1}\tPsi_{\ts}(\zl_j) \left|\hf_j^{(\ell)}(x) - \breve{f}^{(\ell)}(\rho,\zl_j;x)\right|\right|^{2\beta_t},\\
&\Phi(t, \ell) := \EE \prod_{i=\ell+1}^L \frac{1}{m^{(i)}} \sum_{\ts\in \cA^{(\ell)}} \left| 
\Psi^{(\ell+1)}(\ts)\frac{1}{m^{(\ell-1)}}\sum_{j=1}^{m^{(\ell-1)}}\left[ w(\zl_{\ts_{\ell}},\zmi_{j} )
\breve{f}^{(\ell-1)}(\rho,\zmi_{j};x) -  g^{(\ell)}(\rho,\zl_{\ts_{\ell}};x)\right]     \right|^{2\beta_t},
\end{align*}
and
$$ d_t := c_1^{2\beta_t}q_1^{-2\alpha\beta_t-2\beta_t +1}, \quad e_t:=  c_1^{2\beta_t}q_0^{-2\alpha\beta_t}(1-q_1)^{-2\beta_t +1}, \quad  f_t:=  c_0^{2\beta_t}q_1^{-2\beta_t +1}.$$

Note that because 
\begin{align}
    &\left\| \frac{1}{m^{(L)}}\sum_{j=1}^{m^{(L)}}u^{(L)}(z_j^{(L)}) \left[\hat{f}^{(L)}_j(x)-\breve{f}(\rho, z^{(L)};x)\right]  \right\|\notag\\
    \leq&   \frac{1}{m^{(L)}}\sum_{j=1}^{m^{(L)}}\left\|u^{(L)}(z_j^{(L)}) \right\| \left|\hat{f}^{(L)}_j(x)-\breve{f}(\rho, z^{(L)};x)\right|\notag\\
    \leq& \frac{1}{m^{(L)}}\sum_{j=1}^{m^{(L)}}\bbu(z_j^{(L)})\left|\hat{f}^{(L)}_j(x)-\breve{f}(\rho, z^{(L)};x)\right|\notag\\
    =& \left|\frac{1}{m^{(L)}}\sum_{j=1}^{m^{(L)}}\bbu(z_j^{(L)})\left|\hat{f}^{(L)}_j(x)-\breve{f}(\rho, z^{(L)};x)\right|\right|,\notag
\end{align}
we have
\begin{align}
    &\E \|\breve{f}(\rho,u;x)-\hf(x) \|^2\label{endd}\\
    =&\E\left\| \frac{1}{m^{(L)}}\sum_{j=1}^{m^{(L)}}u^{(L)}(z_j^{(L)}) \left[\hat{f}^{(L)}_j(x)-\breve{f}(\rho, z^{(L)};x)\right]  \right\|^2\notag\\
    \leq & \E\left|\frac{1}{m^{(L)}}\sum_{j=1}^{m^{(L)}}\bbu(z_j^{(L)})\left|\hat{f}^{(L)}_j(x)-\breve{f}(\rho, z^{(L)};x)\right|\right|^2\notag\\
    =&\Gamma(0, L).\notag
\end{align}
Moreover, $\Phi(t,l)$ is bounded in \eqref{l use} and $d_t$, $e_t$, and $f_t$ can be treated as constants when $t\leq L$. Finally, integrating \eqref{ttend} with $s\in \cA^{(\ell+1)}$, multiplying the both sides by  $\prod_{i=\ell+1}^L\frac{1}{m^{(i)}}$, and taking expectation, for all $\ell\in \{2, \dots, L\}$ and $t\geq 0$, we have
\begin{align}
\Gamma(t, \ell) \leq d_t  \Gamma(t+1, \ell-1) + e_t  \Phi(t+1, \ell) + f_t \Gamma(t, \ell-1).
\end{align} 
For all $t\leq L$, we have
$$ d_t \leq d_{L}\vee d_{0}, \quad f_t\leq f_{L}\vee f_{0}.     $$
Let $\gamma := d_{T} \vee f_{T} \vee d_{0} \vee f_{0} \vee 1$,   for all $t\leq L$, we have
\begin{align}
\Gamma(t, \ell+1) \leq \gamma(\Gamma(t+1, \ell)+ \Gamma(t, \ell)) + \tPhi(t, \ell+1),
\end{align}
where $ \tPhi(t, \ell) = \Phi(t+1,\ell)e_t$ and $\ell\in [L-1]$.

When $\ell=1$, we have  $\Gamma(t, 1)=0$ for all $t\geq 0$. Then we show that 
\begin{align}
\Gamma(t, \ell) \leq \sum_{i=1}^{\ell}\gamma^{\ell-i}\sum_{j=0}^{\ell-i}\left(\begin{matrix}
\ell-i\\j
\end{matrix}\right)\tPhi(j+t, i).
\end{align}
We prove it by induction. When $\ell=1$, the above statement is true.  Suppose at $\ell-1$, the statement is true.  We have
\begin{align}
\Gamma(t, \ell+1)\leq& \sum_{i=1}^{\ell} \gamma^{\ell-i+1} \sum_{j=0}^{\ell-i}\left(\begin{matrix}
\ell-i\\j
\end{matrix}\right)\tPhi(j+t+1,i) +\sum_{i=1}^{\ell} \gamma^{\ell-i+1} \sum_{j=0}^{\ell-i}\left(\begin{matrix}
\ell-i\\j
\end{matrix}\right)\tPhi(j+t,i)+ \tPhi(t, {\ell+1})\notag\\
\overset{a}{=}&\sum_{i=1}^{\ell} \gamma^{\ell-i+1}\left[ \sum_{j=1}^{\ell-i}\left(\begin{matrix}
\ell-i+1\\j
\end{matrix}\right)\tPhi({j+t}, i)+ \tPhi(t,i)+ \tPhi({\ell-i+1+t},i) \right] +\tPhi(t,{\ell+1})\notag\\
=&\sum_{i=1}^{\ell+1}\gamma^{\ell+1-i}\sum_{j=0}^{\ell+1-i}\left(\begin{matrix}
\ell+1-i\\j
\end{matrix}\right)\tPhi(j+t, i), \notag
\end{align} 
where in $\overset{a}=$, we use 
$$ \left(\begin{matrix}
n\\k
\end{matrix}\right) +\left(\begin{matrix}
n\\k+1
\end{matrix}\right) = \left(\begin{matrix}
n+1\\k+1
\end{matrix}\right).$$
In all, we have
\begin{align}
&\E\left\| \breve{f}(\rho,u;x)-\hf(x)\right\|^2 \label{2end}\\
\overset{a}{\leq}&\Gamma(0, L) \notag\\
\overset{b}{\leq}&  L\gamma^{L} C^0_{B_T}(c_{M}\vee 1) (e_L\vee e_0) \mathcal{O}\left(\frac{1}{(m^{J^{(\ell)}})^{1+\alpha}}\right)\notag\\
\!=&\mathcal{O}\left(  \frac{1}{(m^{(J^{(\ell)})})^{1+\alpha}} \right).\notag
\end{align}
where in $\overset{a}{\leq}$, we use \eqref{endd} and in $\overset{b}=$, we use the fact that for all $\Phi(t+1, \ell)$ has the lowest order $\frac{1}{(m^{J^{(\ell)}})^{1+\alpha}}$  only when $t=0$ and $j=0$
and  $\gamma$, $C^0_{B_T}$, $e_L$, and $e_0$ only depend on $L$, $\alpha$, $c_0$, $c_1$,  $c_M$, and $c_{M_1}$, and can be treated as constants.

Furthermore, by setting $\alpha = \mathcal{O}(\frac{1}{L})$, we explicitly have
\begin{align}
\E\left\| \breve{f}(\rho,u;x)-\hf(x)\right\|^2\leq \mathcal{O}\left( \frac{L(cc_0\vee c_1^2)^{eL} (c_{M}\vee 1)} {(m^{(J^{(\ell)})})^{1+\mathcal{O}(\frac{1}{L})}} \right).
\end{align}
\end{proof}
\noindent\textbf{Step 4.}  We prove  Theorem \ref{var}.
\begin{proof}
We can directly obtain  Theorem \ref{var} by the facts that 
\begin{align}
&\E\left\| \hf(x)-  f(\rho,u;x)\right\|^2\notag\\
=& \E\left\| f(\rho,u;x)-\breve{f}(\rho,u;x)\right\|^2 +\underbrace{\E\left\| \breve{f}(\rho,u;x)-\hf(x)\right\|^2}_{\eqref{2end}} + 2\E \left[f(\rho,u;x)-\breve{f}(\rho,u;x)\right]^\top\left[ \breve{f}(\rho,u;x)-\hf(x)\right]\notag\\
=&\E\left\| f(\rho,u;x)-\breve{f}(\rho,u;x)\right\|^2 +\mathcal{O}\left(\left(\min_{\ell\in [L]}m^{(\ell)}\right)^{-1-\alpha}\right)+ 2\E \left[f(\rho,u;x)-\breve{f}(\rho,u;x)\right]^\top\left[ \breve{f}(\rho,u;x)-\hf(x)\right]\notag\notag\\
\overset{a}=&\underbrace{\E\left\| f(\rho,u;x)-\breve{f}(\rho,u;x)\right\|^2}_{\eqref{theo2:end1}} +\mathcal{O}\left( \sum_{\ell=1}^{L} (m^{(\ell)})^{-(1+\alpha)} \right)+ 2\E \left[f(\rho,u;x)-\breve{f}(\rho,u;x)\right]^\top\left[ \breve{f}(\rho,u;x)-\hf(x)\right]\notag
\end{align}
and
\begin{align}
    &\left|\E \left[f(\rho,u;x)-\breve{f}(\rho,u;x)\right]^\top\left[ \breve{f}(\rho,u;x)-\hf(x)\right]\right|\notag\\
    \leq&\E \left\|f(\rho,u;x)-\breve{f}(\rho,u;x)\right\|\left\|\breve{f}(\rho,u;x)-\hf(x)\right\|\notag\\
    \overset{b}{\leq}& \left( \underbrace{\E \left\|f(\rho,u;x)-\breve{f}(\rho,u;x)\right\|^2}_{\eqref{theo2:end1}} \right)^{1/2} \left( \underbrace{\E \left\|\breve{f}(\rho,u;x)-\hf(x)\right\|^2}_{\eqref{2end}} \right)^{1/2}\notag\\
    \overset{c}{=}&\mathcal{O}\left(\left(\min_{\ell\in [L]}m^{(\ell)}\right)^{-0.5}\right) \times \mathcal{O}\left(\left(\min_{\ell\in [L]}m^{(\ell)}\right)^{-0.5-\alpha/2}\right)\notag\\
        =&\mathcal{O}\left(\left(\min_{\ell\in [L]}m^{(\ell)}\right)^{-(1+\alpha/2)}\right)\notag\\
       \overset{d}{=}&\mathcal{O}\left( \sum_{\ell=1}^{L} (m^{(\ell)})^{-(1+\alpha/2)} \right),\notag
\end{align}
where $\overset{a}=$  and $\overset{d}=$ use the inequality that $ (\min_{\ell\in [L]}m^{(\ell)})^{-p}\leq \sum_{\ell=1}^{L}(m^{(\ell)})^{-p} $ with $p>0$, $\overset{b}=$ uses the Holder's inequality, and  $\overset{c}=$ uses the inequality that $ \sum_{\ell=1}^{L}(m^{(\ell)})^{-p} \leq L(\min_{\ell\in [L]}m^{(\ell)})^{-p}$ with $p>0$.
\end{proof}

\subsection{Proofs of (\ref{eqn:R-l-specific}) and (\ref{eqn:R-u-specific})}\label{app:regg}
\begin{proof}
For all $\ell\in[L-1]$, by assuming $f^{(\ell)}(\rho,z^{(\ell)};x)$ and $\frac{\partial f(\rho,u;x)}{\partial z^{(\ell+1)}}$ are bounded by $C$, we have
\begin{align}
&\E_{z^{(\ell)}}\left(\E_{z^{(\ell+1)}}\frac{\partial f(\rho,u;x)}{\partial z^{(\ell+1)}}  \left[  f^{(\ell)}(\rho,z^{(\ell)};x)w(z^{(\ell+1)}, z^{(\ell)})- g^{(\ell+1)}(\rho,z^{(\ell+1)};x) \right] \right)^2\notag\\
    \overset{a}{\leq}&\E_{z^{(\ell)}}\left(\E_{z^{(\ell+1)}} \frac{\partial f(\rho,u;x)}{\partial z^{(\ell+1)}}   f^{(\ell)}(\rho,z^{(\ell)};x) w(z^{(\ell+1)}, z^{(\ell)}) \right)^2 \nonumber\\
    \leq &C^4\E_{z^{(\ell)}}\left(\E_{z^{(\ell+1)}}\left|  w(z^{(\ell+1)}, z^{(\ell)}) \right|\right)^2 \nonumber\\
    =& C^4\int \left(\int |w(z^{(\ell+1)},z^{(\ell)})|d\rho^{(\ell+1)}(z^{(\ell+1)})\right)^2d\rho^{(\ell)}(z^{(\ell)})=C^4 R^{(\ell+1)}(\rho),\nonumber
    \end{align}
    where $\overset{a}\leq$ uses the facts that $\E\| \tilde{\xi} - \E \tilde{\xi}\|^2\leq\E \|\tilde{\xi} \|^2$ with $\tilde{\xi} =  g^{(\ell+1)}(\rho,z^{(\ell+1)};x)$ and 
    \begin{align}
        &\E_{z^{(\ell)}}\E_{z^{(\ell+1)}}\left[\frac{\partial f(\rho,u;x)}{\partial z^{(\ell+1)}}  f^{(\ell)}(\rho,z^{(\ell)};x)w(z^{(\ell+1)}, z^{(\ell)}\right]\notag\\
        =&\E_{z^{(\ell+1)}}\left[\frac{\partial f(\rho,u;x)}{\partial z^{(\ell+1)}}   \E_{z^{(\ell)}}\left(f^{(\ell)}(\rho,z^{(\ell)};x)w(z^{(\ell+1)}, z^{(\ell)})\right) \right]\notag\\
        =&\E_{z^{(\ell+1)}}\left[\frac{\partial f(\rho,u;x)}{\partial z^{(\ell+1)}}    g^{(\ell+1)}(\rho,z^{(\ell+1)};x) \right].\notag
    \end{align}
    Similarly, we have
\begin{align}
    &\E_{z^{(L)}} \left\| f^{(L)}(\rho,z^{(L)};x) u(z^{(L)})- f(\rho, u;x)\right\|^2\nonumber\\
    \leq&\E_{z^{(L)}} \left\| f^{(L)}(\rho,z^{(L)};x) u(z^{(L)})\right\|^2\nonumber\\
    \leq & C^2\E_{z^{(L)}} \left\| u(z^{(L)})\right\|^2=C^2\int \left\| u(z^{(L)})\right\|^2d\rho^{(L)}(z^{(L)})=C^2R^{(u)}(\rho, u).\nonumber
\end{align}
\end{proof}

\subsection{Additional Lemmas in Appendix \ref{app:2}}
\begin{lemma}\label{lip}
	Suppose $\tilde{h}$ satisfies
	$$| \nabla \tilde{h} (x) -  \nabla \tilde{h} (y) |\leq \tilde{c} | x-y |^{\alpha}$$
	with $\alpha> 0$, then we have
	\begin{align}
	\left|\tilde{h}(x) - \tilde{h}(y) - \<\nabla \tilde{h}(x), x-y\>   \right| \leq \frac{ \tilde{c}}{1+\alpha}|x-y |^{1+\alpha}\leq \tilde{c}| x - y |^{1+\alpha}.
	\end{align}
\end{lemma}
\begin{proof}
	The above argument is right by observing:
	\begin{align}
	\left| \tilde{h}(y) - \tilde{h}(x)   - \<\nabla \tilde{h}(x), y-x\>   \right|= \left|\int_{x}^y [\nabla\tilde{h}(t)- \nabla \tilde{h}(x)]dt  \right|\leq \int_{x}^y    \tilde{c} \left|t -x  \right|^{\alpha}dt. \notag
	\end{align}
\end{proof} 

\begin{lemma}\label{jcon}
	For all $t\geq1$, for any $t_1\in\RR$, $t_2\in\RR$, $t_3\in \RR$,  $q_1\in (0, 1)$, and $q_2\in (0, 1)$,  we have
	\begin{align}
	| t_1 +t_2  |^t \leq \frac{|t_1|^t}{(1-q_1)^{t-1}}+ \frac{|t_2|^t}{(q_1)^{t-1}}.
	\end{align}
	and
	\begin{align}\label{jcon1}
	| t_1 +t_2 +t_3 |^t \leq \frac{|t_1|^t}{(1-q_2)^{t-1}}+ \frac{|t_2|^t}{(q_2/2)^{t-1}} + \frac{|t_3 |^t}{(q_2/2)^{t-1}}.
	\end{align}
\end{lemma}
\begin{proof}
Both of  two inequalities can be obtained directly by Jensen's inequality. 	We prove \eqref{jcon1} only.  Consider function $f(x) = |x|^t$. Because $t\geq1$, $f(x)$ is convex. Thus for any $x_1$, $x_2$, and $x_3$, we have
	\begin{align}
	\left|  (1-q_2)x_1 +\frac{q_2}{2}x_2 + \frac{q_2}{2}x_3    \right|^t\leq (1-q_2)|x_1 |^t + \frac{q_2}{2}| x_2|^t + \frac{q_2}{2}|x_3 |^t.
	\end{align} 
	Then by setting $x_1 = t_1/(1-q_2)$, $x_2 = 2t_2/q_2$, and $x_3 = 2t_3/q_2$, we can obtain \eqref{jcon1}.

\end{proof} 

\section{Neural Feature Repopulation}\label{app:3}
In this section, we give the proofs of Theorems \ref{thm:nfr}, \ref{thm:nfr-inv} which present the  NFR procedure and its inverse procedure, respectively.
\subsection{Proof of Theorem \ref{thm:nfr}}\label{proof of the3}
\begin{proof} The statement is true when $\ell=0$ since $\Q{0}$ is an identity mapping.
  Assume the statement is true for $\ell-1$. We have
\[
\ff^{(\ell-1)}(\rho,\z{\ell-1};x)  
=\ff^{(\ell-1)}(\rhoo,\tilde{z}^{(\ell-1)};x) .
\]
Recall that the map
$\Q{\ell}(\tilde{\rho},\rho_0;\cdot): \cZ{\ell} \to \cZ{\ell}$ takes the form of:
\begin{equation}
\Q{\ell}(\tilde{\rho},\rho_0;z)\left(\Q{\ell-1}(\tilde{\rho},\rho_0;\zeta)\right) = \z{\ell}(\zeta)
\frac{d\tilde{\rho}^{(\ell-1)}}{d\rhoo^{(\ell-1)}} \left(\Q{\ell-1}(\tilde{\rho},\rho_0;\zeta)\right).
\label{eq:nfr-Q}
\end{equation}
Now, the distribution $\rho^{(\ell)}(\cdot)$ on $\z{\ell}$ induces a
distribution $\tilde{\rho}^{(\ell)}(\cdot)$ which is the pushforward of $\rho^{(\ell)}$ by $\Q{\ell}$.

In the following derivation, we let $\tz{\ell}=\Q{\ell}(\tilde{\rho},\rho_0;\z{\ell})$ and $\tz{\ell-1}=\Q{\ell-1}(\tilde{\rho},\rho_0;\z{\ell-1})$.
Then for the case of $\ell$, it follows that
\begin{align*}
\ff^{(\ell)}(\rho,\z{\ell};x)&
h^{(\ell)}\left( \int
  w(\z{\ell}, \z{\ell-1}) \ff^{(\ell-1)}(\rho,\z{\ell-1};x)
  d{\rho}^{(\ell-1)}(\z{\ell-1})\right)\\
=&
h^{(\ell)}\left( \int
  w(\z{\ell}, \z{\ell-1}) \ff^{(\ell-1)}(\rhoo,\tz{\ell-1};x)
  d{\rho}^{(\ell-1)}(\z{\ell-1})\right)\\
 =&
h^{(\ell)}\left( \int
  w(\z{\ell}, \z{\ell-1}) \ff^{(\ell-1)}(\rhoo,\tz{\ell-1};x)
  d\tilde{\rho}^{(\ell-1)}(\tz{\ell-1})\right)\\ 
  =&
h^{(\ell)}\left( \int
  w(\z{\ell}, \z{\ell-1}) \ff^{(\ell-1)}(\rhoo,\tz{\ell-1};x)
  \frac{d\tilde{\rho}^{(\ell-1)}}{d\rhoo}(\tz{\ell-1})d\rhoo^{(\ell-1)}(\tilde{z}^{(\ell-1)})\right)\\ 
  =&
h^{(\ell)}\left( \int
  w(\tz{\ell}, \tz{\ell-1}) \ff^{(\ell-1)}(\rhoo,\tz{\ell-1};x)
  d\rhoo^{(\ell-1)}(\tilde{z}^{(\ell-1)})\right)\\
  =& \ff^{(\ell)}(\rhoo,\tz{\ell};x) .
\end{align*}
Therefore, the statement holds for all $\ell \in [L]$.

We let $\tu =\tau^{(u)}(\tilde{\rho},\rho_0;u)$. Then from the definition of $\tau^{(u)}(\tilde{\rho},\rho_0;\cdot)$, we have,
\begin{align}
\tu(\Q{L}(\tilde{\rho},\rho_0;\zeta))=
u(\zeta) \frac{d\tilde{\rho}^{(L)}}{d\rhoo^{(L)}}\left(\Q{L}(\tilde{\rho},\rho_0;\zeta)\right),\nonumber
\end{align}
for $\zeta \in \cZ{L}$.

Therefore, for the top layer, 
let $\tz{L}=\Q{L}(\z{L})$,
we have
\begin{align*}
\ff(\rho,\u;x) =&  \int
  \u(\z{L}) \ff^{(L)}(\rho,\z{L};x)
  d\rho^{(L)}(\z{L}) \\
  =&\int
  \u(\z{L}) \ff^{(L)}(\rhoo,\tz{L};x)
  d\tilde{\rho}^{(L)}(\tz{L})\\
  =&\int
  \u(\z{L}) \ff^{(L)}(\rhoo,\tz{L};x)
  \frac{d\tilde{\rho}^{(L)}}{d\rhoo^{(L)}}(\tz{L})d\rhoo^{(L)}(\tz{L})\\
  =&\int
  \tu(\tz{L}) \ff^{(L)}(\rhoo,\tz{L};x)
  d \rhoo^{(L)}(\tz{L})\\
    =& \ff(\rhoo,\tu,x).
\end{align*}

For the regularizer, we have
\begin{align}
R^{(\ell)}(\rho)&=\int r_2\left (\int r_1\Big(w (z^{(\ell)}, z^{(\ell-1)})\Big)d\rho^{(\ell)}(z^{(\ell)})\right ) d\rho^{(\ell-1)}(z^{(\ell-1)})\nonumber\\
&=\int r_2\left (\int r_1\Big(w (z^{(\ell)}, z^{(\ell-1)})\Big)d\tilde{\rho}^{(\ell)}(\tilde{z}^{(\ell)})\right ) d\tilde{\rho}^{(\ell-1)}(\tilde{z}^{(\ell-1)})\nonumber\\
&=\int r_2\left (\int r_1\Big(w (\tilde{z}^{(\ell)}, \tilde{z}^{(\ell-1)})\frac{d\rhoo^{(\ell-1)}}
{d\tilde{\rho}^{(\ell-1)}}(\tz{\ell-1})\Big)d\tilde{\rho}^{(\ell)}(\tilde{z}^{(\ell)})\right ) d\tilde{\rho}^{(\ell-1)}(\tilde{z}^{(\ell-1)})\nonumber
\end{align}

Moreover,
\begin{align}
R^{(\u)}(\rho,\u) &=  \int r^{(u)}(\u(z^{(L)}))
                 d\rho^{(L)}(z^{(L)})\nonumber\\
                 &=\int r^{(u)}(\u(z^{(L)}))
                 d\tilde{\rho}^{(L)}(\tilde{z}^{(L)})\nonumber\\
                 &=\int r^{(u)}\Big(\tu(\tz{L})\frac{d\rhoo^{(L)}}{d\tilde{\rho}^{(L)}}(\tz{L})\Big)
                 d\tilde{\rho}^{(L)}(\tilde{z}^{(L)}).\nonumber
\end{align}
This proves the desired result.

\end{proof}

\subsection{Proof of Theorem \ref{thm:nfr-inv}}
\begin{proof} 

 $\rho^{(\ell)}$ can be recovered by the inverse of the mapping $\Q{\ell}(\tilde{\rho},\rho_0;\cdot)$, i.e., $\tilde{\tau}^{(\ell)}(\tilde{\rho},\rho_0;\cdot)$,  recursively. That is  we can define $\rho^{(\ell)}$ recursively such that the probability distribution $\tilde{\rho}^{(\ell)}$ on $\tz{\ell}$ induces a probability distribution $\rho^{(\ell)}$ on $\z{\ell}=\tQ{\ell}(\tilde{\rho},\rho_0; \tz{\ell})$. $u$ can be recovered by 
  \[
  u = \tilde{\tau}^{(u)}(\tilde{\rho}, \rho_0;\tilde{u}).\nonumber
  \]
  The remaining of the proof is identical to that of Theorem~\ref{thm:nfr}.
\end{proof}

\section{Properties of Continuous DNN with Specific Regularizers}\label{app:4}
Below, we give the proofs of Theorem \ref{thm:convexity}, Corollary \ref{KKT:condition},  and Theorem \ref{theorem:global-optimum}, which present the  results of NFR with specific regularizers. 
\subsection{Proof of Theorem \ref{thm:convexity}}\label{app:41}
We consider a more general result shown below:
\begin{theorem}\label{thm:convexity1}
Assume the loss function $\phi(\cdot, \cdot)$ is convex in the the first argument. Let $r_1(\omega)=|\omega|^{o_1}$,  $r_2(\omega)=|\omega|^{o_2}$ , and $r^{(u)}(u) = \|u\|^{o_3}$ where $\omega\in \R$ and $u\in \R^K$, $(o_1, o_2, o_3)\in\{ (o_1, o_2, o_3): o_2^{-1}\leq o_1\leq 1, o_2\geq 1, o_3\geq 1\}$, 
then $\tilde{Q}(\tilde{\rho}, \tilde{u})$ is a convex function of $(\tilde{\rho}, \tilde{u})$, where $(\tilde{\rho}, \tilde{u})$ is induced by $\{(\rho, u): \rho^{(\ell)}\sim \rho_0^{(\ell)}, \ell \in [L]\}$.
\end{theorem}
We can find $(o_1,o_2, o_3)=(1,2,2)\in \{ (o_1, o_2, o_3): o_2^{-1}\leq o_1\leq 1, o_2\geq 1, o_3\geq 1\}$. Therefore Theorem \ref{thm:convexity} is a special case of Theorem \ref{thm:convexity1}.  We directly prove Theorem \ref{thm:convexity1}.

\begin{proof} of Theorem \ref{thm:convexity1}:

Let $$p^{(\ell)}(z^{(\ell)}) := \frac{d\tilde{\rho}}{d \rho_0}(z^{(\ell)}), \quad \ell \in [L],$$
and $P = \{p^{(1)}, \dots, p^{(L)} \}$.  From the definition of convexity,  we know that it is equivalent to prove that $\breve{Q}(P,\tu):=\tilde{Q}(\trho,\tu)$ is convex on $(P, \tilde{u})$.  From \eqref{eq:nfr-j}, \eqref{eq:nfr-r-layer}, and \eqref{eq:nfr-r-top}, we have 
\begin{align}
    &\breve{Q}(P,\tu)\label{objective}\\
    =&\tilde{Q}(\trho,\tu)\notag\\
    =&  J(\ff(\rho_0,\tu;\cdot)) + \lambda^{(u)} \int  \frac{\left\|\tu (\tilde{z}^{(L)})\right\|^{o_3}}{\left(p^{(L)}(\tilde{z}^{(L)})\right)^{o_3}} p^{(L)}(\tilde{z}^{(L)}) d  \rho_0^{(L)}(\tilde{z}^{(L)}) \notag\\
   & +  \sum_{\ell=1}^{L}\lambda^{(\ell)}  \int \left(\int \frac{\left|w (\tilde{z}^{(\ell)}, \tilde{z}^{(\ell-1)})\right|^{o_1}}
{\left(p^{(\ell-1)}(\tilde{z}^{(\ell-1)})\right)^{o_1}} p^{(\ell)}(\tilde{z}^{(\ell)}) d\rho_0^{(\ell)}(\tilde{z}^{(\ell)})\right)^{o_2} p^{(\ell-1)}(\tilde{z}^{(\ell-1)}) d\rho_0^{(\ell-1)}(\tilde{z}^{(\ell-1)})\notag\\
=& \underbrace{J(\ff(\rho_0,\tu,\cdot))}_{I_1} + \lambda^{(u)}\underbrace{ \int \frac{\|\tu (\tilde{z}^{(L)})\|^{o_3}}{\left(p^{(L)}(\tilde{z}^{(L)})\right)^{o_3-1}}  d  \rho_0^{(L)}(\tilde{z}^{(L)})}_{I_2} \notag\\
   & +  \sum_{\ell=1}^{L}\lambda^{(\ell)}   \int\underbrace{ \frac{\left(\int\left|w (\tilde{z}^{(\ell)}, \tilde{z}^{(\ell-1)}) \right|^{o_1} p^{(\ell)}(\tilde{z}^{(\ell)}) d\rho_0^{(\ell)}(\tilde{z}^{(\ell)})\right)^{o_2}}{\left(p^{(\ell-1)}(\tilde{z}^{(\ell-1)})\right)^{o_1o_2-1}}}_{I_3^{(\ell)}(\tz{\ell-1};P)}  d\rho_0^{(\ell-1)}(\tilde{z}^{(\ell-1)}).
\end{align}
$I_1(\tu)$ is convex  on $\tu$ under our assumption that $\phi$ is convex in the first argument.  $I_2(\tu, p^{(L)})$ is convex on $(\tu, P)$ from Lemma \ref{property of convex} in the end of the section.  In the following, we prove $I_3^{(\ell)}(\tz{\ell-1};P)$   is convex on $P$ for all $\ell \in [L]$, $\tz{\ell-1}\in \mathcal{Z}^{(\ell-1)}$.   Once we obtain it, we have $\int I_3^{(\ell)}(\tz{\ell-1};P) d\rho_0^{(\ell-1)}(\tilde{z}^{(\ell-1)})$ is convex on $P$ for all $\ell\in[L]$, so $\tilde{Q}$ is convex on $(P, \tilde{u})$. 
In fact, define
$$ q^{(\ell-1)}(\tz{\ell-1}) := \int\left|w (\tilde{z}^{(\ell)}, \tilde{z}^{(\ell-1)}) \right|^{o_1} p^{(\ell)}(\tilde{z}^{(\ell)}) d\rho_0^{(\ell)}(\tilde{z}^{(\ell)}), ~~~~\ell \in[L].$$
Consider $\bar{I}_3^{(\ell)}(\tz{\ell-1}) = \frac{\left(q^{(\ell-1)}(\tz{\ell-1})\right)^{o_2}}{\left(p^{(\ell-1)}(\tz{\ell-1})\right)^{o_1o_2-1}}$. So $\bar{I}_3^{(\ell)}(\tz{\ell-1})$ is convex on $\left(q^{(\ell-1)}(\tz{\ell-1}), p^{(\ell-1)}(\tz{\ell-1})\right)$ from Lemma \ref{property of convex} when $o_1o_2\in[1, o_2]$. This means  for any  $\left(q_1^{(\ell-1)}(\tz{\ell-1}), p_1^{(\ell-1)}(\tz{\ell-1})\right)$,\\ $\left(q_2^{(\ell-1)}(\tz{\ell-1}), p_2^{(\ell-1)}(\tz{\ell-1})\right)$,  $a\geq0$, $b\geq0$ with $a+b=1$, we have
$$ \frac{\left[a q_1^{(\ell-1)}(\tz{\ell-1})+b q_2^{(\ell-1)}(\tz{\ell-1})\right]^{o_2}}{\left[ap_1^{(\ell-1)}(\tz{\ell-1})+bp_2^{(\ell-1)}(\tz{\ell-1})\right]^{o_1o_2-1}}   \leq a \frac{\left[q_1^{(\ell-1)}(\tz{\ell-1})\right]^{o_2}}{\left[p_1^{(\ell-1)}(\tz{\ell-1})\right]^{o_1o_2-1}} +b  \frac{\left[q_2^{(\ell-1)}(\tz{\ell-1})\right]^{o_2}}{\left[p_2^{(\ell-1)}(\tz{\ell-1})\right]^{o_1o_2-1}}.  $$
Let $ q_i^{(\ell-1)}(\tz{\ell-1}) = \int\left|w (\tilde{z}^{(\ell)}, \tilde{z}^{(\ell-1)}) \right|^{o_1} p_i^{(\ell)}(\tilde{z}^{(\ell)}) d\rho_0^{(\ell)}(\tilde{z}^{(\ell)}),~~i\in\{1,2\}$, we have
\begin{align}
 &\frac{\left[      \int\left|w (\tilde{z}^{(\ell)}, \tilde{z}^{(\ell-1)}) \right| \left(a p_1^{(\ell)}(\tilde{z}^{(\ell)})+ b p_2^{(\ell)}(\tilde{z}^{(\ell)})\right) d\rho_0^{(\ell)}(\tilde{z}^{(\ell)})  \right]^{o_2}}{\left[ap_1^{(\ell-1)}(\tz{\ell-1})+bp_2^{(\ell-1)}(\tz{\ell-1})\right]^{o_1o_2-1}}\notag\\   
 \leq&   a \frac{\left[       \int\left|w (\tilde{z}^{(\ell)}, \tilde{z}^{(\ell-1)}) \right|  p_1^{(\ell)}(\tilde{z}^{(\ell)}) d\rho_0^{(\ell)}(\tilde{z}^{(\ell)})  \right]^{o_2}}{\left[p_1^{(\ell-1)}(\tz{\ell-1})\right]^{o_1o_2-1}} +b  \frac{\left[
 \int\left|w (\tilde{z}^{(\ell)}, \tilde{z}^{(\ell-1)}) \right|   p_2^{(\ell)}(\tilde{z}^{(\ell)})d\rho_0^{(\ell)}(\tilde{z}^{(\ell)}) 
 \right]^{o_2}}{\left[p_2^{(\ell-1)}(\tz{\ell-1})\right]^{o_1o_2-1}},\notag
\end{align}
which indicates that  $I_3^{(\ell)}(\tz{\ell-1};P)$ is convex on $P$. We obtain our result. 
\end{proof}

\subsection{Proof of Corollary \ref{KKT:condition}}
\begin{proof}

Let $(\rho_*, u_*)$ be a optimal solution of the NN. Then for $(\rho,u)$, consider performing NRF on $\rho_*$, which means we set $\rho_0=\rho_*$.  From \eqref{objective}, we have \begin{align}
    &\breve{Q}(P,\tu)\label{opt:1}\\
=& J(\ff(\rho_*,\tu;\cdot)) + \lambda^{(u)} \int \frac{\left\|\tu (\tilde{z}^{(L)})\right\|^2}{p^{(L)}(\tilde{z}^{(L)})}  d  \rho_*^{(L)}(\tilde{z}^{(L)})\notag\\
   & +  \sum_{\ell=1}^{L}\lambda^{(\ell)}   \int \frac{\left(\int\left|w (\tilde{z}^{(\ell)}, \tilde{z}^{(\ell-1)}) \right| p^{(\ell)}(\tilde{z}^{(\ell)}) d\rho_*^{(\ell)}(\tilde{z}^{(\ell)})\right)^2}{p^{(\ell-1)}(\tilde{z}^{(\ell-1)})} d\rho_*^{(\ell-1)}(\tilde{z}^{(\ell-1)}),\notag
\end{align}
where we denote $p^{(\ell)}(z^{(\ell)}) := \frac{d\tilde{\rho}}{d \rho_*}(z^{(\ell)})$, $\ell \in [L]$, and $P = \{p^{(1)}, \dots, p^{(L)} \}$.  It is  clear that the solution $p^{(\ell)}(z^{(\ell)})=1$ for all $\ell\in[L]$,$z^{(\ell)}\in \mathcal{Z}^{(\ell)}$and   $\tu = u_*$ is a  minimizer  of \eqref{opt:1}  and should satisfy the Karush–Kuhn–Tucker (KKT) conditions. Then by driving the KKT  conditions, there exist  sequences   $\{\breve{\Lambda}^{(\ell)}_1\}_{\ell\in[L]}$ and $\{\breve{\Lambda}^{(\ell)}_2\}_{\ell\in[L]}$, where $\breve{\Lambda}^{(\ell)}_1\in \RR $ and $\breve{\Lambda}^{(\ell)}_2: \mathcal{Z}^{(\ell)}\to \RR$ for all $\ell\in[L]$,   we have for $p^{(\ell)}$, $\ell\in[L-1]$,
\begin{align}
    &\breve{ \Lambda}_1^{(\ell)}+\breve{ \Lambda}_2^{(\ell)}(\tilde{z}^{(\ell)})-\lambda^{(\ell+1)}\frac{\left(\int\left|w (\tilde{z}^{(\ell+1)}, \tilde{z}^{(\ell)}) \right| p^{(\ell+1)}(\tilde{z}^{(\ell+1)}) d\rho_*^{(\ell+1)}(\tilde{z}^{(\ell+1)})\right)^2}{(p^{(\ell)}(\tilde{z}^{(\ell)}))^2}+  2\lambda^{(\ell)}\times\notag\\
    &\int \frac{\int\left|w (z^{(\ell)}, \tilde{z}^{(\ell-1)}) \right| p^{(\ell)}(z^{(\ell)}) d\rho_*^{(\ell)}(z^{(\ell)})}{p^{(\ell-1)}(\tilde{z}^{(\ell-1)})} \left|w (\tilde{z}^{(\ell)}, \tilde{z}^{(\ell-1)})\right|d\rho_*^{(\ell-1)}(\tilde{z}^{(\ell-1)}) =0, ~~\ell\in[L-1],~\tilde{z}^{(\ell)}\in \mathcal{Z}^{(\ell)},\notag
\end{align}
for $p^{(L)}$,
\begin{align}
&\breve{ \Lambda}_1^{(L)}+\breve{\Lambda}_2^{(L)}(\tilde{z}^{(L)}) - \lambda^{(u)} \frac{\left\|\tu (\tilde{z}^{(L)})\right\|^2}{(p^{(L)}(\tilde{z}^{(L)}))^2}+  2\lambda^{(L)}\times \notag \\
&\int \frac{\int\left|w (z^{(L)}, \tilde{z}^{(L-1)}) \right| p^{(L)}(z^{(L)}) d\rho_*^{(L)}(z^{(L)})}{p^{(L-1)}(\tilde{z}^{(L-1)})} \left|w (\tilde{z}^{(L)}, \tilde{z}^{(L-1)})\right|d\rho_*^{(L-1)}(\tilde{z}^{(L-1)}) =0,~~\tilde{z}^{(L)}\in \mathcal{Z}^{(L)},\notag
\end{align}
for $\tu$,
\begin{align}
   \EE_{x,y} \left[J'(f(\rho_*, \tu;x)) f^{(L)}(\rho_*, \tilde{z}^{(L)};x)\right]+ 2\lambda^{(u)} \frac{\tu(\tilde{z}^{(L)})}{p^{(L)}(\tilde{z}^{(L)})}=0,~~\tilde{z}^{(L)}\in \mathcal{Z}^{(L)}\notag
\end{align}
for the constrains,
\begin{align}
&\breve{\Lambda}^{(\ell)}_2(\tilde{z}^{(\ell)})p^{(\ell)}(\tilde{z}^{(\ell)}) = 0, ~ \breve{\Lambda}^{(\ell)}_2(\tilde{z}^{(\ell)})\leq 0, ~ p^{(\ell)}(\tilde{z}^{(\ell)})\geq 0,~~  \ell \in [L], ~~ \tilde{z}^{(\ell)}\in \mathcal{Z}^{(\ell)},\notag
\end{align}
and
\begin{align}
\int p_*^{(\ell)}(\tilde{z}^{(\ell)}) d\rho_*^{(\ell)}(\tilde{z}^{(\ell)}) = 1,~~\ell \in  [L].\notag
\end{align}

Using  $p^{(\ell)}(z^{(\ell)}) = 1$,
we have $\breve{\Lambda}_2^{(\ell)}(z^{(\ell)})=0$.  Set $\Lambda = \breve{\Lambda}_1$.  For all $\ell\in[L-1]$ and $\tilde{z}^{(\ell)} \in  \mathcal{Z}^{(\ell)}$, we have
\begin{align}
  &\lambda^{(\ell+1)}\left(\int\left|w (\tilde{z}^{(\ell+1)}, \tilde{z}^{(\ell)}) \right|  d\rho_*^{(\ell+1)}(\tilde{z}^{(\ell+1)})\right)^2\notag\\
        =&\Lambda^{(\ell)}+  2\lambda^{(\ell)}\int\int\left|w (z^{(\ell)}, \tilde{z}^{(\ell-1)}) \right|  d\rho_*^{(\ell)}(z^{(\ell)}) \left|w (\tilde{z}^{(\ell)}, \tilde{z}^{(\ell-1)})\right|d\rho_*^{(\ell-1)}(\tilde{z}^{(\ell-1)}),\notag
\end{align}
and  for $\tilde{z}^{(L)} \in  \mathcal{Z}^{(L)}$,
\begin{align}
       & \Lambda^{(L)}+  2\lambda^{(L)}\int \int\left|w (z^{(L)}, \tilde{z}^{(L-1)}) \right|  d\rho_*^{(L)}(z^{(L)}) \left|w (\tilde{z}^{(L)}, \tilde{z}^{(L-1)})\right|d\rho_*^{(L-1)}(\tilde{z}^{(L-1)}) =\lambda^{(u)}\left\|\tu (\tilde{z}^{(L)})\right\|^2,\notag\\
       & \EE_{x,y} \left[J'(f(\rho_*, \tu;x)) f^{(L)}(\rho_*, \tilde{z}^{(L)};x)\right] =-2\lambda^{(u)} \tu(\tilde{z}^{(L)}),\notag
\end{align}
which is the desired result.
\end{proof}

\subsection{Proof of Theorem \ref{theorem:global-optimum}}
\begin{proof} 
Let $$p^{(\ell)}(z^{(\ell)}) := \frac{d\tilde{\rho}}{d \rho_0}(z^{(\ell)}), \quad \ell \in [L],$$
 and $P = \{p^{(1)}, \dots, p^{(L)} \}$. For the $\ell_{p,1}$ norm regularizers, we have
 \begin{align}
  \bR^{(\ell)}(P, \tu)&:= R^{(\ell)}(\trho, \tu)\notag\\
  &=\int \int \frac{\left|w (\tilde{z}^{(\ell)}, \tilde{z}^{(\ell-1)})\right|^{q^{(\ell)}}}
{(p^{(\ell-1)}(\tilde{z}^{(\ell-1)}))^{q^{(\ell)}}} p^{(\ell)}(\tilde{z}^{(\ell)}) d\rho_0^{(\ell)}(\tilde{z}^{(\ell)}) p^{(\ell-1)}(\tilde{z}^{(\ell-1)}) d\rho_0^{(\ell-1)}(\tilde{z}^{(\ell-1)})\notag\\
  &=\int \int \frac{\left|w (\tilde{z}^{(\ell)}, \tilde{z}^{(\ell-1)})\right|^{q^{(\ell)}}}
{(p^{(\ell-1)}(\tilde{z}^{(\ell-1)}))^{q^{(\ell)}-1}} p^{(\ell)}(\tilde{z}^{(\ell)}) d\rho_0^{(\ell)}(\tilde{z}^{(\ell)}) d\rho_0^{(\ell-1)}(\tilde{z}^{(\ell-1)})
 \end{align}
 and 
  \begin{align}
  \bR^{(u)}(P, \tu)&:= R^{(u)}(\trho, \tu)\notag\\
  &= \int  \frac{\left\|\tu (\tilde{z}^{(L)})\right\|^{q^{(u)}}}{\left(p^{(L)}(\tilde{z}^{(L)})\right)^{q^{(u)}}} p^{(L)}(\tilde{z}^{(L)}) d  \rho_0^{(L)}(\tilde{z}^{(L)}) \notag\\
&= \int  \frac{\left\|\tu (\tilde{z}^{(L)})\right\|^{q^{(u)}}}{\left(p^{(L)}(\tilde{z}^{(L)})\right)^{q^{(u)}-1}} d  \rho_0^{(L)}(\tilde{z}^{(L)}).
 \end{align}
 Set $p_*^{(\ell)}(z^{(\ell)}) = \frac{d\tilde{\rho}_*}{d \rho_0}(z^{(\ell)})$ with $\ell \in [L]$ and  $P_* = \{p_*^{(1)}, \dots, p_*^{(L)} \}$. 
 Give $\tu$,  because $P_*$ is a local solution. By the KKT conditions,  there exist  sequences   $\{\Lambda^{(\ell)}_1\}_{\ell\in[L]}$ and $\{\Lambda^{(\ell)}_2\}_{\ell\in[L]}$, where $\Lambda^{(\ell)}_1: \mathcal{Z}^{(\ell)}\to \RR$  and $\Lambda^{(\ell)}_2\in \RR $ for all $\ell\in[L]$,   we have
\begin{align}
&\lambda^{(\ell+1)}\nabla_{p^{(\ell)}(\tilde{z}^{(\ell)})}\bR^{(\ell+1)}(P_*, \tu)+ \lambda^{(\ell)}\nabla_{p^{(\ell)}(\tilde{z}^{(\ell)})}\bR^{(\ell)}(P_*, \tu)+\Lambda_1^{(\ell)}(\tilde{z}^{(\ell)})+\Lambda_2^{(\ell)} = 0,~~ \ell \in [L-1],~\tilde{z}^{(\ell)} \in  \mathcal{Z}^{(\ell)},\notag \\
&\lambda^{(L)}\nabla_{p^{(L)}(\tilde{z}^{(L)})}\bR^{(L)}(P_*, \tu)+ \lambda^{(u)}\nabla_{p^{(L)}(\tilde{z}^{(L)})}\bR^{(u)}(P_*, \tu)+\Lambda_1^{(L)}(\tilde{z}^{(L)})+\Lambda_2^{(L)}= 0,~~\tilde{z}^{(L)} \in  \mathcal{Z}^{(L)},  \label{KKT 3}\\
&\Lambda^{(\ell)}_1(\tilde{z}^{(\ell)})p_*^{(\ell)}(\tilde{z}^{(\ell)}) = 0, ~ \Lambda^{(\ell)}_1(\tilde{z}^{(\ell)})\leq 0, ~ p_*^{(\ell)}(\tilde{z}^{(\ell)})\geq 0,~~  \ell \in [L], ~~ \tilde{z}^{(\ell)}\in \mathcal{Z}^{(\ell)},\notag\\
& \int p_*^{(\ell)} d\rho_0^{(\ell)}(\tilde{z}^{(\ell)}) = 1,~~\ell \in  [L].\notag
\end{align}
Let $a_0 = 1$, $b_1 = a_0\geq1$, and $b_{\ell+1} = b_{\ell}(q^{(\ell)}-1)+a_0\geq1$ for all $\ell \in\{2, \dots, L-1\}$.
We can change of variables as $t^{(\ell)} =  (p^{(\ell)})^{\frac{1}{b_\ell}}$ with $\ell\in[L]$.  Let $t =\{t^{(1)}, \dots, t^{(L)}\}$. Then it is equivalent to solve the problem:
\begin{align}\label{t problem}
&\min_{t}~~\sum_{\ell=1}^L \lambda^{(\ell)} \dR^{(\ell)}(t, \tu) +\lambda^{(u)}\dR^{(u)}(t, \tu) \\
&s.t.~~t^{(\ell)}\geq 0  ~\text{and}~ \int  \left(t^{(\ell)}(z^{(\ell)})\right)^{b_{\ell}} d\rho_0^{(\ell)}(z^{(\ell)}) = 1, ~\ell \in \{1, \dots, L\}, \notag
\end{align}
where 
$$  
\dR^{(\ell)}(t,\tu)=\int \int 
|w(\tz{\ell},\tz{\ell-1}|^{q^{(\ell)}}
\frac{t^{(\ell)}(\tz{\ell})^{b_{\ell-1}(q^{(\ell)}-1)+a_0}}
{t^{(\ell-1)}(\tz{\ell-1})^{b_{\ell-1}(q^{(\ell)}-1)}}
  d \rho_0(\tz{\ell}) d\rho_0 (\tz{\ell-1})
$$
and 
$$
 \dR^{(u)}(t,\tu)= \int \frac{\left\|\tu(\tz{L})\right\|^{q^{(u)}}}{t^{(L)}(\tz{L})^{b_L (q^{(u)}-1)}}d \rho_0(\tz{L}) .
$$
Because $f(x, y) = \frac{x^{a}}{y^{b}}$ with $b\geq a-1\geq1$ is joint convex on $(x,y)$ (Refer  Lemma \ref{property of convex}), the objective  of \eqref{t problem} is a convex function. On the other hand,  the weighted $\ell_{b_\ell}$  ($b_\ell\geq1$) norm is also convex.  We can relax  problem \eqref{t problem}  to the following convex optimization problem:
\begin{align}\label{problemt}
&\min_{t}~~\sum_{\ell=1}^L \lambda^{(\ell)} \dR^{\ell}(t, \tu) +\lambda^{(u)}\dR^{(u)}(t, \tu) \\
&s.t.~~t^{(\ell)}\geq 0  ~\text{and}~ \int  \left(t^{(\ell)}(\tilde{z}^{(\ell)})\right)^{b_{\ell}} d \rho_0^{(\ell)}(\tilde{z}^{(\ell)}) \leq 1, ~~\ell \in [L], \notag
\end{align}
  Let $t_*$ be $\{t_*^{(1)},\dots, t_*^{(L)}\}$ with  $t_*^{(\ell)} = (p^{(\ell)}_*)^{\frac{1}{b_\ell}} $ for all $\ell \in [L]$.  We show that  $t_*$ is  the minimum solution of  \eqref{problemt}. To achieve it, we first prove that  $\Lambda_2^{(\ell)}\geq0$ holds for all $\ell\in [L]$.  For all $\ell \in  [L]$, let $z^{(\ell)} = 0$ denote the node whose weights connecting to all $z^{(\ell-1)}$ are  $0$, i.e. $w(0, z^{(\ell-1)})=0$ for all $z^{(\ell-1)}\in \mathcal{Z}^{(\ell-1)}$ . Consider the node $z^{(L)} = 0$,  from \eqref{KKT 3}, we have
$$
- (q^{(u)}-1)\lambda^{(u)}\frac{\left\| \tu(0)\right\|^{q^{(u)}}}{(p_*^{(L)}(0))^{q^{(u)}}} + \Lambda_1^{(L)}(0)+ \Lambda_2^{(L)} =0.
$$
For $\Lambda_1^{(L)}(0)\leq 0$,   we have $\Lambda_2^{(L)}\geq0$.  In a similar way, for all $\ell\in [L]$, we have
$$
-  (q^{(\ell+1)}-1)\lambda^{(\ell+1)} \frac{\int |w(\tilde{z}^{(\ell+1)}, 0)|^{q^{(\ell+1)}} p_*^{(\ell+1)}(\tilde{z}^{\ell+1})d\rho_0(\tilde{z}^{(\ell+1)}) }{(p_*^{(\ell)}(0))^{q^{(\ell+1)}}   } + \Lambda_1^{(\ell)}(0)+ \Lambda_2^{(\ell)} =0,
$$
which indicates that $\Lambda_2^{(\ell)}\geq0$. We denote  
$ \bar{\Lambda}_2^{(\ell)} = \frac{\partial  p_*^{(\ell)}}{\partial t_*^{(\ell)}}   {\Lambda}_2^{(\ell)}=    b_{\ell} (t_*^{(\ell)})^{b_l-1}{\Lambda}_2^{(\ell)}\geq0$ and let $ \bar{\Lambda}_1^{(\ell)} =    b_{\ell} (t_*^{(\ell)})^{b_l-1}{\Lambda}_1^{(\ell)}\leq 0.$ We have 
\begin{align}
&\lambda^{(\ell+1)}\nabla_{t^{(\ell)}(\tilde{z}^{(\ell)})}\dR^{(\ell+1)}(t_*, \tu)+ \lambda^{(\ell)}\nabla_{t^{(\ell)}(\tilde{z}^{(\ell)})}\dR^{(\ell)}(t_*, \tu)+\bar{\Lambda}^{(\ell)}_1(\tilde{z}^{(\ell)})+\bar{\Lambda}^{(\ell)}_2 = 0,~~ \ell \in  [L-1],~\tilde{z}^{(\ell)} \in  \mathcal{Z}^{(\ell)},\notag \\
&\lambda^{(L)}\nabla_{t^{(L)}(\tilde{z}^{(L)})}\dR^{(L)}(t_*, \tu)+ \lambda^{(u)}\nabla_{t^{(L)}(\tilde{z}^{(L)})}\dR^{(u)}(t_*, \tu)+\bar{\Lambda}^{(L)}_1(\tilde{z}^{(L)})+\bar{\Lambda}^{(L)}_2= 0, ~~,~\tilde{z}^{(L)} \in  \mathcal{Z}^{(L)},\notag \\
&\bar{\Lambda}_1^{(\ell)}(\tilde{z}^{(\ell)})t_*^{(\ell)}(\tilde{z}^{(\ell)}) = 0, ~ \bar{\Lambda}_1^{(\ell)}(\tilde{z}^{(\ell)})\leq 0, ~ t_*^{(\ell)}(\tilde{z}^{(\ell)})\geq 0,~~  \ell \in  [L], ~~\tilde{z}^{(\ell)}\in \mathcal{Z}^{(\ell)},\notag\\
&\int \left(t_*^{(\ell)}(\tilde{z}^{(\ell)})\right)^{b_{\ell}} d\rho_0^{(\ell)}(\tilde{z}^{(\ell)}) \leq 1,~~\ell \in  [L], \notag\\
&\bar{\Lambda}^{(\ell)}_2 \left(\int \left(t_*^{(\ell)}(\tilde{z}^{(\ell)})\right)^{b_{\ell}} d \rho_0^{(\ell)}(\tilde{z}^{(\ell)})- 1\right)=0,~\bar{\Lambda}^{(\ell)}_2\geq0,    ~~\ell \in  [L].\notag
\end{align}
So  $\left\{ t_*^{(\ell)},\bar{\Lambda}_1^{(\ell)},\bar{\Lambda}_2^{(\ell)}\right\}_{\ell \in [L]}$ satisfies the KKT conditions for the convex optimization problem \eqref{problemt}. Thus it is a global solution for \eqref{problemt}. This implies that it is also  a global solution  for \eqref{t problem}.  We obtain our result.

\end{proof}

\subsection{Additional Lemma in Appendix \ref{app:4}}
\begin{lemma}\label{property of convex}
	$\frac{x^{a}}{y^b}$ is convex on $(x,y)\in [0, +\infty)\otimes(0, +\infty)  $ when $a-1\geq b  \geq 0$ and is strictly convex on $(x,y)\in (0, +\infty)\otimes (0, +\infty)$ when $a-1> b \geq 0$.
\end{lemma}
\begin{proof}
	When $b=0$, Lemma \ref{property of convex} is right. When $b>0$, let $f(x, y) = \frac{x^a}{y^b},$ we have $  \nabla_x^2 f= \frac{a(a-1)x^{(a-2)}}{y^b}$, $ \nabla_y^2 f= \frac{b(b+1)x^a}{y^{b+2}} $, and $\nabla^2_{xy} f =\frac{ab x^{a-1}}{y^{b+1}}$.  We prove Lemma \ref{property of convex} by showing that for all $x_1 \in \RR$ and $y_1\in \RR$,  
	\begin{align}
	\frac{a(a-1)x^{(a-2)}}{y^b}x_1^2  -  2\frac{ab x^{a-1}}{y^{b+1}} x_1y_1 +  \frac{b(b+1)x^a}{y^{b+2}} y_1^2\geq  \ep  \frac{a(a-1)x^{(a-2)}}{y^b}x_1^2 + \ep  \frac{b(b+1)x^a}{y^{b+2}} y_1^2.
	\end{align}
	When $\ep=0$, it is the not-strictly convex case.  The above argument is equivalent to prove that 
	\begin{align}
	[a(a-1)-\ep]x_1^2 - 2ab\frac{x}{y}x_1y_1+(b(b+1)-\ep)\frac{x^2}{y^2}y_1^2\geq 0.
	\end{align}
	Furthermore, it is sufficient to obtain that 
	\begin{align}\label{con tte}
	[a(a-1)-\ep][ b(b+1)-\ep ]\geq a^2b^2.
	\end{align}
	Set  $\ep_1=a-b-1\geq0$, $\ep = \frac{ab\ep_1}{b^2+b+a^2 -a}$, $\ep_a = \ep/a$, and $\ep_b = \ep/b$. We have
	$$  \ep_1\geq  \ep_a(b+1) +(a-1)\ep_b -\ep_a \ep_b, $$ 
	which leads to \eqref{con tte}.
\end{proof}

\newpage
\section{Discrete  Neural Feature Repopulation Algorithm}\label{sec:alo}
\begin{algorithm}
  \caption{Detailed Discrete Neural Feature Repopulation Algorithm}   
  \label{alo:alo}
  \begin{algorithmic}[1]  
      \STATE {\bfseries Initialization:}  Start with initial $(\hat{w}_0,\hat{u}_0)$
      \FOR{$t=1$ to T}
      \STATE Run multiple steps of SGD to obtain updated $(\tilde{w}_t,\tilde{u}_t)$ from $(\hat{w}_{t-1},\hat{u}_{t-1})$
      \STATE Find $\hat{p}'$ to optimize
      \begin{align*}
      \hat{p}'=\arg\min_{\hat{p}} \hat{R}(\hat{p};\tilde{w}_t,\tilde{u}_t)
      \end{align*}
      with $\hat{R}$ defined in \eqref{eq:reg-im}
      \STATE  Set initial weights as $(\hat{w}_t,\hat{u}_t)$
      \FOR{$\ell=L$ downto $1$}
      \FOR{$j=1$ to $m^{(l)}$}
         \STATE Sample $j'$ from categorical distribution $\text{Cat}(\hat{p}')$
         \STATE $\hat{w}^{(\ell)}_{j,:}=(\tilde{w}_t)^{(\ell)}_{j,:}$
         \IF {$\ell=L$}
            \STATE $\hat{u}_{j}=\frac{1}{\hat{p}'_j}(\tilde{u}_j')_t$
         \ELSE 
            \STATE $\hat{w}^{(\ell+1)}_{:,j}=\frac{1}{\hat{p}'_j} (\tilde{w}_t)^{(\ell+1)}_{:,j}$
         \ENDIF
      \ENDFOR
      \STATE $\hat{w}_t=\hat{w}$, $\hat{u}_t=\hat{u}$
      \ENDFOR
      \ENDFOR
      \STATE {\bfseries Return:} $\hat{w}_{T}$ and $\hat{u}_{T}$
  \end{algorithmic}  
\end{algorithm}  

\end{document}